\documentclass{article}

\PassOptionsToPackage{numbers, compress}{natbib}

\usepackage[preprint]{neurips_2025}

\usepackage[utf8]{inputenc} 
\usepackage[T1]{fontenc}    
\usepackage{hyperref}       
\usepackage{url}            
\usepackage{booktabs}       
\usepackage{amsfonts}       
\usepackage{nicefrac}       
\usepackage{microtype}      
\usepackage{xcolor}         

\usepackage{graphicx} 
\usepackage{subcaption}
\usepackage{caption}
\usepackage{bbm}
\usepackage{breakcites}
\usepackage{algorithm} 
\usepackage{wrapfig}
\usepackage[noend]{algpseudocode}
\usepackage{hyperref}
\usepackage{amsmath}
\usepackage{amssymb}
\usepackage{mathtools}
\usepackage{amsthm}
\usepackage{thm-restate}
\usepackage{placeins}

\theoremstyle{plain}
\newtheorem{theorem}{Theorem}[section]

\newtheorem{fact}[theorem]{Fact}
\newtheorem{claim}[theorem]{Claim}

\theoremstyle{definition}

\theoremstyle{remark}
\newtheorem{remark}[theorem]{Remark}

\newcommand{\floor}[1]{\left\lfloor #1 \right\rfloor}
\newcommand{\ceil}[1]{\left\lceil #1 \right\rceil}

\DeclareMathOperator*{\argmax}{arg\,max}

\DeclareMathOperator{\Points}{\mathcal{C}}
\DeclareMathOperator{\Pointsh}{\mathnormal{\mathcal{C}^h}}

\DeclareMathOperator{\rh}{\mathnormal{r_h}}

\DeclareMathOperator{\ci}{\mathnormal{C_i}}
\DeclareMathOperator{\cih}{\mathnormal{C^h_i}}

\DeclareMathOperator{\OPT}{\mathrm{OPT}}
\DeclareMathOperator{\OPTSF}{\mathrm{OPT_{\text{SF}}}}

\DeclareMathOperator{\OPTR}{\mathrm{OPT_{\text{R}}}}
\DeclareMathOperator{\OPTLPR}{\mathrm{OPT^{\text{LP}}_{\text{R}}}}

\DeclareMathOperator{\OPTU}{\mathrm{OPT_{\text{U}}}}

\DeclareMathOperator{\OPTLPU}{\mathrm{OPT^{\text{LP}}_{\text{U}}}}

\DeclareMathOperator{\constr}{\mathnormal{c_{R}}}
\DeclareMathOperator{\constu}{\mathnormal{c_{U}}}

\DeclareMathOperator{\OPTW}{\mathrm{OPT_{\W}}}

\DeclareMathOperator{\Colors}{\mathcal{H}}

\DeclareMathOperator{\viol}{\mathnormal{\nu}}

\DeclareMathOperator{\disu}{\mathnormal{D}}
\DeclareMathOperator{\totdistcol}{\mathcal{D}}
\DeclareMathOperator{\totviolcol}{\mathcal{V}}

\DeclareMathOperator{\gammap}{\mathnormal{\gamma_p}}
\DeclareMathOperator{\gammapp}{\mathnormal{\gamma_p\prime}}

\newcommand{\adult}{\textbf{Adult}}
\newcommand{\cens}{\textbf{Census1990}}
\newcommand{\credit}{\textbf{CreditCard}}

\newcommand{\bank}{\textbf{Bank}}

\newcommand{\cO}{\mathnormal{O}}
\newcommand{\cH}{\mathcal{H}}

\newcommand{\Sstar}{{S^*}}
\newcommand{\Ssf}{{S_{\SF}}}
\newcommand{\phisf}{{\phi_{\SF}}}

\newcommand{\phibar}{{\bar{\phi}}}

\newcommand{\phir}{\phi}
\newcommand{\sr}{\mathnormal{S_{\SF}}}

\newcommand{\phiu}{\phi}
\newcommand{\su}{\mathnormal{S_{\W}}}

\newcommand{\Sw}{{S_{\W}}}
\newcommand{\phiw}{{\phi_{\W}}}

\newcommand{\istar}{{i^*}}
\newcommand{\phistar}{{\phi^*}}

\newcommand{\SF}{\textbf{SF}}
\newcommand{\W}{\textbf{W}}

\newcommand{\nrst}{\text{nrst}}

\def\x{\mathbf{x}}

\newcommand{\algr}{\mathnormal{\textsc{RawlsianAlg}}}
\newcommand{\algu}{\mathnormal{\textsc{UtilitarianAlg}}}

\newcommand{\rawlsrounding}{\mathnormal{\textsc{Rawlsian-Rounding}}}

\title{Welfare-Centric Clustering}

\author{%
  Claire Jie Zhang \thanks{Equal contribution.} \\
  University of Washington\\
   \And
  Seyed A. Esmaeili \protect\footnotemark[1] \\
  University of Chicago \\
   \AND
  Jamie Morgenstern \\
  University of Washington }

\begin{document}

\maketitle

\begin{abstract}
Fair clustering has traditionally focused on ensuring equitable group representation or equalizing group-specific clustering costs. However, \citet{dickerson2024fair} recently showed that these fairness notions may yield undesirable or unintuitive clustering outcomes and advocated for a welfare-centric clustering approach that models the utilities of the groups. In this work, we model group utilities based on both distances and proportional representation and formalize two optimization objectives based on welfare-centric clustering: the Rawlsian (Egalitarian) objective and the Utilitarian objective. We introduce novel  algorithms for both objectives and prove theoretical guarantees for them. 
Empirical evaluations on multiple real-world datasets demonstrate that our methods significantly outperform existing fair clustering baselines.
\end{abstract}


\section{Introduction}
Algorithmic decision making has proliferated across various domains, from personalized advertising, recommender systems, and  hiring to even more consequential and critical applications such as recidivism prediction, loan approval, and kidney exchange \cite{tejaswini2020accurate, arun2016loan, purohit2019hiring, berk2021fairness, liu2011comparison, berk2013statistical, awasthi2009online, aziz2021optimal, mcelfresh2019scalable, raghavan2020mitigating}. This widespread use has led to many documented examples where intuitive and agreed upon notions of fairness have been violated \cite{kearns2019ethical,o2017weapons}. As a result fairness has become a major consideration in machine learning and algorithm design.

Clustering, a fundamental and classical problem in both machine learning and operations research, has received significant attention as a case study for operationalizing various notions of fairness \cite{chierichetti2017fair, bercea2018cost, bera2019fair, ahmadian2019clustering, ghadiri2021socially, abbasi2021fair, makarychev2021approximation, jung2019center, chen2019proportionally}. These works consider standard clustering objectives (e.g., $k$-means or $k$-median) and solve them subject to a constraint on an operationalization of some fairness notion. For example, the most prominent notion in fair clustering is that of \emph{proportional mixing} where the standard clustering objective is minimized subject to each cluster having equitable representation of each demographic group.\footnote{For example, if the dataset consists of demographic groups $A$, $B$, and $C$ with proportions $50\%$, $30\%$, and $20\%$, respectively, then for each cluster the percentage of points from groups $A$, $B$, and $C$ should be around $50\%$, $30\%$, and $20\%$, respectively.}

While proportional mixing as well as other fairness notions in clustering are well-motivated\footnote{Specifically, for proportional mixing, since each group is proportionally represented in each cluster, this would ensure that any outcome assigned purely by cluster assignment  will affect each group proportionally.}, it has  recently been shown that the application of these notions can lead to clustering outcomes which are unintuitive and result in poor solutions both globally and for each group separately \cite{dickerson2024fair}. Take as an example the clustering instance in Figure \ref{fig:motivation} where the red and blue colors denote two different demographic groups. In this instance, the top two groupings of points are far from each other relative to the other points: clustering them with any points outside those groupings leads to high distance costs for those points. So, any proportionally mixed clustering will lead to large distances for all of the points in those top two groupings.\footnote{In general, shorter distances to centers are more favorable. This holds if the application was a machine learning or an operations research application. In the former, shorter distances correspond to better representation by the center and in the latter they correspond to shorter travel distance to the center.}

As noted in \citet{dickerson2024fair} and shown in Figure \ref{fig:motivation}, a \emph{welfare-centric} clustering which quantifies the utilities of the red points and blue points separately and optimizes for them jointly would overcome these issues and produce clusterings where the upper groupings would be clustered separately and the lower groupings would cluster together. While the instance presented in the figure is simplified, it is representative of many settings where balancing the proportional mixing and the distances should be handled more carefully and where utility-aware clusterings are much more preferable. In this paper, we give a formalization of welfare-centric clustering and introduce algorithms with provable guarantees for it.

\begin{wrapfigure}{r}{0.5\textwidth}
  \vspace{-0.5cm}  
  \begin{center}
    \includegraphics[width=1\linewidth]{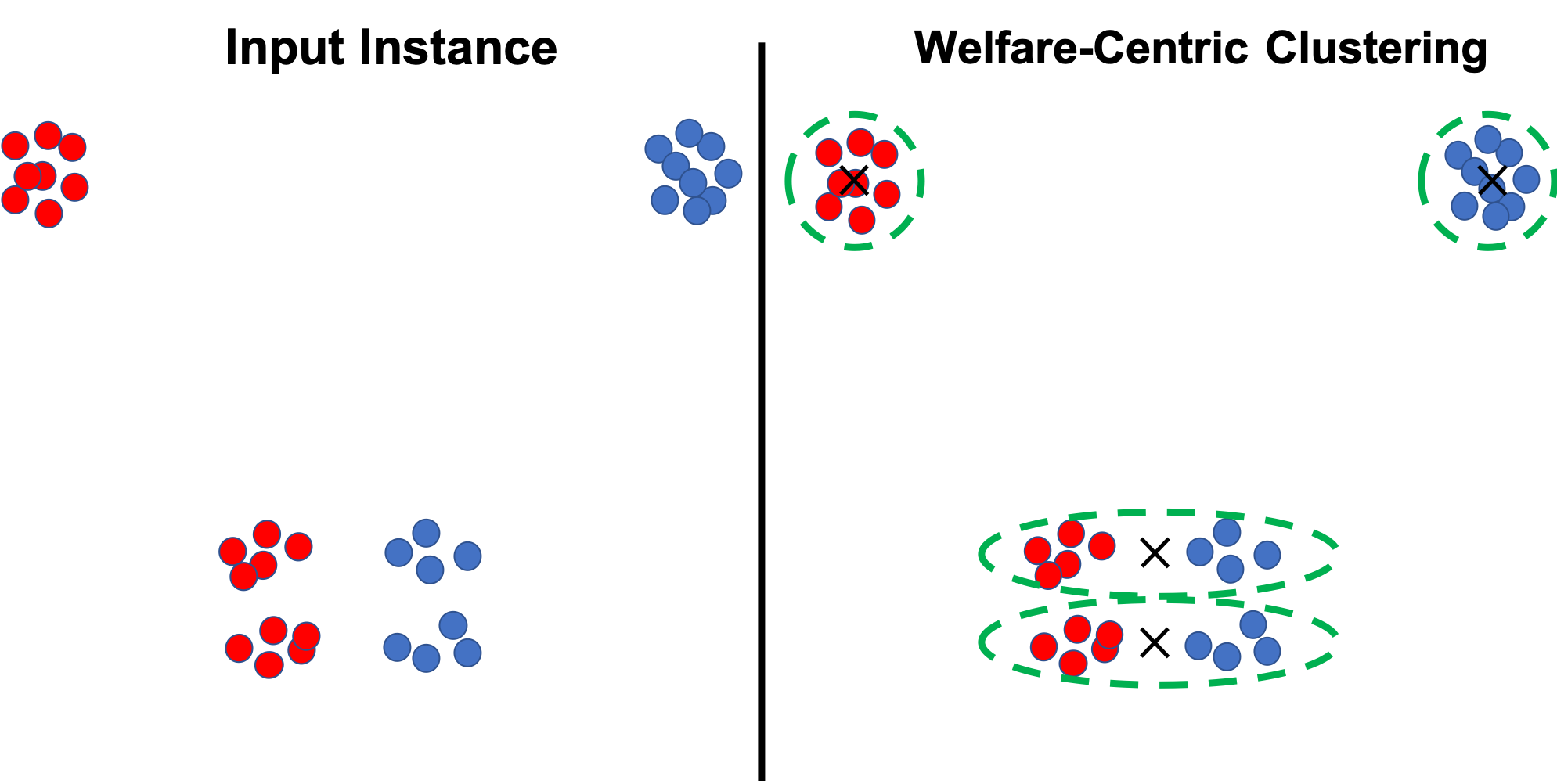}
  \end{center}
    \caption{Instance and its welfare optimal clustering.}
    \label{fig:motivation}
\end{wrapfigure}

\paragraph{Our Contributions and Outline.} In Section \ref{sec:related_work}, we review the related works and demonstrate the need for a clustering notion based on utilities and welfare. In Section \ref{sec:problem_statement}, we formalize the notions for group utility 
and introduce two formal notions of welfare. Specifically, we introduce the Rawlsian (Egalitarian) objective as well as the Utilitarian objective. In Section \ref{sec:wc_algo_and_thm}, we introduce our novel algorithms for both objectives. Interestingly, although our algorithms for both objectives share a similar format, the details of each have to be carefully adjusted to optimize its own objective. In Section \ref{sec:experiments}, we implement our algorithms and show superior performance over a collection of baselines on real world datasets. Due to space limits all of the proofs are pushed to the Appendix.

\section{Related Work}\label{sec:related_work}
As mentioned the most prominent notion in fair clustering is that of proportional color mixing where each cluster is constrained to have proportional representations of the groups. This notion was originally introduced in \cite{chierichetti2017fair} and was then generalized to satisfy arbitrary lower and upper bounds as well as various clustering objectives \cite{bercea2018cost,bera2019fair,kleindessner2019guarantees,knittel2023fair}. \citet{esmaeili2021fair} noted that satisfying proportional color mixing might come at an unbounded increase to the  clustering cost and therefore introduced the \emph{\textbf{F}air \textbf{C}lustering under a \textbf{B}ounded \textbf{C}ost} (FCBC) problem where an upper bound on the clustering cost is given and then proportional color mixing is essentially maximized across the clusters. Although, this notion could mix the bottom clusters in Figure \ref{fig:motivation} it still leads to inequitable solutions where one group may unreasonably have a strictly lower utility than the other. This happens since FCBC is only using the clustering cost instead of formulating a welfare function based on clear utility assignments. See Appendix \ref{app:fcbc_comparison} for a detailed discussion. Additionally, \citet{hakim2024fairness} extends \cite{esmaeili2021fair} by finding the Pareto frontier of the clustering cost vs the proportional violations. However, similar to \cite{esmaeili2021fair} it has the same issues since it only considers the clustering cost instead of the welfare. Further, the algorithm of \cite{hakim2024fairness} has a run-time of $\Omega(n^k)$ where $n$ is the number points and $k$ is the number of clusters, making it highly impractical.

Another important notion is that of \emph{socially fair} clustering \cite{abbasi2021fair,ghadiri2021socially,makarychev2021approximation} where the average clustering cost for each group is found and then the maximum average across all groups is minimized. While this notion is well-motivated, it completely ignores the proportional representation of the groups within the formed clustering. For example, in the instance of Figure \ref{fig:motivation}, the optimal socially fair clustering would not form clusters where the groups are mixed. This is a shortcoming that our framework overcomes since the average clustering cost of a group as well as its proportional representation are optimized for simultaneously by including them in the utility.

Finally, there has been significant work in supervised learning that emphasized the importance of formalizing utility/welfare and having them as the main objective to optimize for instead of imposing simple constraints such as demographic parity \cite{heidari2018fairness,chen2021fairness,chohlas2024learning,mladenov2020optimizing,hu2020fair,cousins2021axiomatic,heidari2019moral,rosenfeld2025machine}. Despite the significant attention welfare-based formulations have received in supervised learning, they have been largely ignored in clustering. The exception being the recent work of \citet{dickerson2024fair}. In particular, \citet{dickerson2024fair} actually introduced an example similar to Figure \ref{fig:motivation} and proved that proportional mixing and socially fair clustering would in fact produce strictly less welfare in comparison to a welfare-centric approach. In this paper, we formulate welfare in a manner largely similar to \citet{dickerson2024fair} with the main difference being that we formulate utility at the level of the group instead of the individual and that we define a specific form to calculate the proportional mixing (or violation of proportional mixing).
Additionally, we note that while \cite{dickerson2024fair} pointed out the importance of considering welfare in clustering, it gave no algorithms or even heuristics for obtaining welfare-centric clusterings.

\vspace{-0.4cm}
\section{Problem Statement}\label{sec:problem_statement}
We are given a set of points  $\Points$ representing individuals to be clustered with $|\Points| = n$. Each point $j \in \Points$ has a color $h \in \cH$ where each color is associated with a specific demographic group. We assume that each point belongs to exactly one group. We therefore,  define the color assignment function $\chi: \Points \to \Colors$ that assigns points to colors, i.e., $\chi(j)$ is the color of point $j$. The subset of individuals with color $h$ is denoted by $\Points^h$, and the total number of such individuals is $n_h = |\Points^h|$. Accordingly, the proportion of color $h$ in the dataset is $r_h = \frac{n_h}{n}$. Additionally, we have a distance function $d(.,.)$ which defines a metric over the space.

A clustering solution is represented by the pair $(S, \phi)$, where $S$ is the set of selected centers and $\phi : \Points \rightarrow S$ is an assignment function that maps each point to a center, thereby forming the clusters. The distance between a point $j$ and its assigned center is $d(j, \phi(j))$. Of particular interest to us are the $k$-median and $k$-means clustering objectives which can be formally written as:
\begin{align}
    \min\limits_{S,\phi} \sum_{j \in \Points} d^p(j,\phi(j))
\end{align}
Where in the above we set $p=1$ and $p=2$ for the $k$-median and $k$-means objectives, respectively. Note that the assignment function $\phi$ is trivial in the above objectives since one would simply assign each point to its closest center to achieve a minimum clustering cost. However, that is not the case when constraints are imposed on the clustering objectives.\footnote{For example, if we constraint each cluster to have a maximum number of points \cite{li2017uniform,chuzhoy2005approximating}, then some points maybe assigned to further away centers if their closest center has too many points assigned to it.} 

Additionally, we use $C_i$ to denote the set of points assigned to center $i \in S$, i.e., $C_i$ is the $i^{\text{th}}$ cluster. Further, $\cih = \ci \cap \Pointsh$, i.e., the subset of those points with color $h$ in $C_i$. With this notation, we may instead write the clustering cost as $\sum_{i \in S} \sum_{j \in C_i} d^p(j,i)$.

Before introducing our problem we formally introduce the following two clustering objectives which will be relevant to us. The first is \emph{socially fair clustering} \citep{abbasi2021fair, ghadiri2021socially, makarychev2021approximation} which is defined as:

\begin{align}
    \textbf{Socially Fair Clustering:} \quad \quad \min\limits_{S,\phi} \max\limits_{h \in \Colors} \frac{1}{|\Points^h|} \sum_{j \in \Points^h} d^p(j,\phi(j))  \label{eq:def_sf}
\end{align}
This objective amounts to minimizing the maximum average clustering cost across the groups. The second clustering objective is \emph{weighted clustering} \citep{har2004coresets,charikar1999constant} where each point $j \in \Points$ has a specific weight $w_j\in \mathbb{R}_{>0}$ associated with it, formally we have:
\begin{align}
    \textbf{Weighted Clustering:} \quad \quad \min\limits_{S,\phi}  \sum_{j \in \Points} w_j \ d^p(j,\phi(j))  \label{eq:def_wighted}
\end{align}
Note that unlike socially fair clustering, weighted clustering is well-defined even if the points have no color assignments.

Our approach is focused on optimizing welfare (optimizing an aggregate of the average group utilities). First, clearly in a clustering solution, the distance between a point and its assigned center will have a great effect on the disutility\footnote{Since smaller distance leads to better utility, we find it more natural to use disutilities instead of utilities. Therefore we will minimize an aggregate of the disutilities instead of maximizing it.}, with shorter distances leading to smaller disutilities. Therefore, for a given solution $(S,\phi)$ and a given group $h \in \Colors$, we define $\totdistcol_h$ as follows:
\begin{align}
    \totdistcol_h(S,\phi) = \sum_{j \in \Points^h} d^p(j,\phi(j)) 
\end{align}
Additionally, since the proportional representation of a group in a given clustering has a significant impact on the disutility as well\footnote{As noted earlier, this is the main motivation in the notion of proportional mixing.}, we define $\totviolcol_h(S,\phi)$ as follows:
\begin{align}\label{eq:totviolcol}
    \totviolcol_h(S,\phi) = \sum_{i \in S} |C_i| \cdot \viol(h,i)
\end{align}
Where $\viol(h,i)$ is the amount of proportional violation color $h$ experiences in cluster $i$. Formally, $\viol: \Colors \times S \to \mathbb{R}_{\ge 0}$ and is defined as follows: 
\begin{align}\label{eq:def_piece_wise_linear_asym_fairness_in_each_point}
    \viol(h,i)= \max 	\left\{ \frac{|\cih|}{|\ci|} - (r_h+\alpha_h) , \ (r_h-\beta_h) - \frac{|\cih|}{|\ci|} , \ 0  \right\}
\end{align}
where $\alpha_h$ and $\beta_h$ are upper and lower proportional violations respectively.
To understand \eqref{eq:def_piece_wise_linear_asym_fairness_in_each_point}, note that if color $h$ is proportionally represented, i.e., $r_h+\alpha_h \ge \frac{|\cih|}{|\ci|} \ge r_h-\beta_h$, then $\viol(h,i)=0$. On the other hand, if the $h$ is over-represented, i.e.,  $\frac{|\cih|}{|\ci|} > r_h+\alpha_h $, then would have $\viol(h,i)= \frac{|\cih|}{|\ci|} - (r_h+\alpha_h)$. Similarly, $\viol(h,i) = (r_h-\beta_h) - \frac{|\cih|}{|\ci|}$ when there is an under-representation. Note that $\alpha_h$ and $\beta_h$ are there to relax the violation values for the upper and lower proportions desired for color $h$, respectively. For example, if we set $\alpha_h\ge 1-r_h$ and $\beta_h=r_h$, then we only obtain positive violations if color $h$ is under-represented and never when it is over-represented. 

It is also important to note that in \eqref{eq:totviolcol} the amount of violation a group (color) experiences in cluster $i$ is weighted by the cluster size $|C_i|$. This choice is well-motivated, since one would expect that a group would not be affected significantly by an disproportional representation in a cluster if the cluster size is not significant.

Based on the above, we define the average disutility of a group $\disu_h$ as:
\begin{align}\label{eq:disu}
    \disu_h(S,\phi) =  \frac{\lambda \cdot \totdistcol_h(S,\phi) + (1-\lambda) \cdot \totviolcol_h(S,\phi)}{|\Points^h|}
\end{align}
where we have $\lambda \in [0,1]$. Equation \eqref{eq:disu} is a very natural assignment for a group's disutility, simply stating that distance and proportional violation should be minimized simultaneously. The role of $\lambda$ is to decide the weighting between the distance and the proportional violation. 

All we have left is to set an objective that minimizes an aggregate of the disutilities. The two most standard objectives from welfare economics \cite{brandt2016handbook} are the \emph{Rawlsian} (egalitarian) objective $R = \max\limits_{h \in \Colors} \disu_h(S,\phi)$ and the \emph{Utilitarian} objective $U(S,\phi) =  \sum_{h \in \Colors} \disu_h(S,\phi)$, therefore we will focus on minimizing one of the following two objectives:
\begin{align}
    \textbf{Rawlsian:} & \quad \quad \min\limits_{S,\phi} R(S,\phi) = \min\limits_{S,\phi} \max\limits_{h \in \Colors} \disu_h(S,\phi) \label{defn:Rawlsian} \\ 
    \textbf{Utilitarian:} & \quad \quad \min\limits_{S,\phi} U(S,\phi)  = \min\limits_{S,\phi} \sum_{h \in \Colors} \disu_h(S,\phi) \label{defn:Utilitarian}  
\end{align}

Clearly, for the case where $\lambda=1$ the Rawlsian objective $R$ is equal to the socially fair objective \eqref{eq:def_sf}. Further, for $\lambda=1$ it is not difficult to see that the the Utilitarian objective $U$ would be a special instance of weighted clustering \eqref{eq:def_wighted} where for each point $j \in \Points$ we would have $w_j=\frac{1}{|C^{\chi(j)}|}$ where $\chi(j)$ is the color of point $j$. 

As a final notional note, we will use $\OPTR$ and $\OPTU$ for the optimal values of the Rawlsian and Utilitarian objectives, respectively.

\section{Algorithms and Theoretical Guarantees}\label{sec:wc_algo_and_thm}
It is not difficult to see that ordinary clustering is a special case of both the Rawlsian and the Utilitarian objectives. Therefore, since both the $k$-median and $k$-means problems are NP-hard \cite{bhattacharya2020hardness,charikar1999constant,awasthi2015hardness}, it follows immediately that both the Rawlsian and Utilitarian objectives are NP-hard as well. We therefore design approximation algorithms for them that would be guaranteed to run in polynomial instead of exponential time. Our approach is modular and is similar for both objectives. Specifically, for both objectives we start by running an algorithm to find the set of centers and then having fixed the set of centers we find an assignment from points to centers. 
Although such a ``two stage'' approach of choosing centers and then solving for the assignment was followed before in fair clustering \cite{bercea2018cost,bera2019fair,esmaeili2020probabilistic,esmaeili2021fair}, the center selection is done using a \emph{vanilla} clustering algorithm and the rounding can be done mostly through a standard structure of a min-cost max-flow network. In our case, this approach has to be significantly adjusted. Specifically, as will be shown in the next two subsections, center selection has to be done using socially fair and weighted clustering algorithms for the Rawlsian and Utilitarian objectives, respectively. Further, the min-cost max-flow rounding has to be adjusted significantly for the Rawlsian  objective.

For a more clear presentation, we will set $\gammap=2^{p-1}(2^{p-1}+1)$ and $\gammapp=2^{2(p-1)}$.

\subsection{Algorithm for the Rawlsian Objective: $\algr$}\label{subsec:rawlsian}

\begin{algorithm}[h!] 
\caption{$\algr$}\label{alg:algr}
    \begin{algorithmic}[1]
    \State \textbf{Input:} Set of points $\Points$, number of centers $k$ .
    \State $\sr = \textsc{SociallyFairAlg}(\Points,k)$. \label{alg_line:cs_r}
    \State Solve LP \eqref{eq:assignment_lp} to obtain the solution $\x^{\text{frac}}$. \label{alg_line:lp_r}
    \State Set $\x^{\text{integ}} = \textsc{Rawlsian-Rounding}(\x^{\text{frac}})$ and form the assignment $\phi$ using $\x^{\text{integ}}$.\label{alg_line:round_r}
    \State \textbf{Return} $(\sr,\phir)$. 
    \end{algorithmic}
\end{algorithm}

Algorithm \ref{alg:algr} shows the pseudo-code of $\algr$ which is our approximation algorithm for the Rawlsian objective. The set of centers $\sr$ is set to the output of a $\textsc{SociallyFairAlg}$ which can be any $\alpha$-approximation algorithm for the socially fair objective. The intuition behind this choice lies in the fact that the  Rawlsian objective is $\min_{h \in \Colors} \disu_h = \frac{\lambda \totdistcol_h+ (1-\lambda) \totviolcol_h}{|\Points^h|}$, which ignoring the term $(1-\lambda) \totviolcol_h$ is identical to a $\lambda$ multiple of the socially fair objective \eqref{eq:def_sf}. In fact a formal connection can be proved. Specifically, if $\sr$ is the set of centers chosen by an $\alpha$-approximation algorithm for the socially fair objective then there exists an assignment $\phi'$ that leads to a solution with a bounded increase to the approximation factor $\alpha$ for the Rawlsian objective. The lemma below gives the formal statement. 
\begin{restatable}{lemma}{existanceofapproxsol}\label{lemma:existance_of_approx_solution}
Using the set of centers $\sr$ chosen by an  $\alpha$-approximation algorithm for the socially fair objective, there exists an assignment $\phi'$ such that $(\sr,\phi')$ leads to a $(\gamma_p+\gamma_p' \alpha)$-approximate solution for the Rawlsian objective. 
\end{restatable} 
Note that the above lemma leads to an approximation factor of $(2+\alpha)$ and $(6+4\alpha)$ for $p=1$ and $p=2$, respectively. 

It is important to note that the above lemma is non-constructive, i.e., although it establishes the existence of such an assignment $\phi'$ it does not give a procedure for finding it. Nevertheless, the lemma will be essential to constructing the assignment algorithm and proving guarantees for it.

It is actually still NP-hard to optimize the Rawlsian objective even with the fixed set of centers  $\sr$ (see Appendix \ref{app:hardness_of_assignment} for details). Our approach to finding the assignment function $\phir$ starts by solving the following linear program (LP):
\begin{subequations} \label{eq:assignment_lp}
\begin{align}
    \min \quad & \quad z\\
    \text{subject to} & \quad \forall h  \in \cH : \frac{1}{|\Points^{h}|} \left( \lambda \sum_{j \in \Points^{h}} \sum_{i \in \sr}  d^p(i, j)x_{ij} + (1-\lambda)  \sum_{i \in \sr} t_{i, h}\right) \leq z \label{constraint:b} \\
    & \quad \forall i \in \sr, \forall h \in \Colors:  (r_h - \beta_h) \cdot \sum_{j \in \Points}x_{ij} - \sum_{j \in \Pointsh}x_{ij} = u_{i, h} \label{constraint:c}\\
    &\quad \forall i \in \sr, \forall h \in \Colors:  \sum_{j \in \Pointsh}x_{ij} -(r_h+\alpha_h) \cdot \sum_{j \in \Points}x_{ij} = o_{i, h} \label{constraint:d}\\
    &\quad \forall i \in \sr, \forall h \in \Colors: u_{i, h} \leq t_{i, h}, \ \ o_{i, h} \leq t_{i, h}, \ \ 0 \leq t_{i, h}\label{constraint:e}\\
    &\quad \forall j \in \Points: \sum_{i \in \sr}x_{ij} = 1, \ \ x_{ij} \in [0, 1] \label{constraint:f} 
\end{align}
\end{subequations}
The LP is more easily interpreted by considering integral values (in $\{0,1\}$) of the variables $x_{ij}$ instead of the fractional values (in $[0,1]$). In particular, $x_{ij}$ decides if point $j \in \Points$ is assigned to center $i \in \sr$ and since each point must be assigned to exactly one center, the constraints of \eqref{constraint:f} must hold. Further, when there is an under-representation of color $h$ in cluster $i$, the variable $u_{i,h}$ in \eqref{constraint:c} would equal the amount of under-representation (i.e., $|C_i| \viol(h,i)$). On the other hand, if there is no under-representation of color $h$ in cluster $i$ then we would have $u_{i,h}\leq 0$. A similar comment can be made about $o_{i,h}$ \eqref{constraint:d} for the over-representation of color $h$ in cluster $i$. Therefore, the variable $t_{i,h}$ is introduced and in constraint \eqref{constraint:e} we have $t_{i,h}$ lower bounded by $0$, $u_{i,h}$ and $o_{i,h}$ ensuring that it is equal to the value of $|C_i| \viol(h,i)$ under all possible scenarios. Moreover, the left hand side of constraint \eqref{constraint:b} equals the disutility $\disu_h$ of group $h$ and we have the variable $z$ lower bounded by the disutilities of all groups. This ensures that the objective $\min z$ is minimizing the maximum disutility, i.e., the Rawlsian objective. In fact, the LP is exactly a relaxation of of the integer program that would find the optimal Rawlsian assignment given the centers $\sr$. Based on the above the following lemma must hold. 
\begin{restatable}{lemma}{lpboundopt}\label{lemma:lpboundopt}
Denoting the optimal value of LP $\eqref{eq:assignment_lp}$ by $\OPTLPR$ then $\OPTLPR \leq (\gamma_p+\gamma_p' \alpha) \OPTR$. 
\end{restatable} 

The essential point behind the above lemma is that the LP includes all possible assignments to the centers $\sr$ and would thus include the assignment $\phi'$ from Lemma \ref{lemma:existance_of_approx_solution} leading to a cost of at most $(\gamma_p+\gamma_p' \alpha) \OPTR$.

Although we can obtain the optimal LP solution (assignment) $\x^{\text{frac}}$, the values in the entries $x^{\text{frac}}_{ij}$ would in general be fractional instead of integral and therefore cannot be used directly. One of the standard rounding methods used in fair clustering is based on constructing a min-cost max-flow network \cite{bercea2018cost,esmaeili2021fair,dickerson2024doubly}. Essentially, in this constructed network lower and upper bounds on arcs are set to ensure that the total amount of points from each group in each cluster do not vary from the fractional LP solution. Further, the costs on the arcs are set to the distance values between the points and the centers, ensuring that an assignment that minimizes the total distance aggregated across points from different colors is found. However, since the Rawlsian objective minimizes the maximum disutility across the colors instead of aggregating them as a sum, we modify the min-cost max-flow scheme. Figure \ref{fig:color_specific_rounding} shows a schematic of our rounding which we call $\rawlsrounding$. In particular, the rounding is done $|\Colors|$ many times, each time over the points of one color in isolation from the rest of the points in the dataset. The full details of the network and $\rawlsrounding$ are given in Appendix \ref{app:rawlsrounding}. Denoting the quantities associated with the LP solution with the superscript $^\text{frac}$ and those of the integral solution with $^\text{integ}$, the following lemma can be established for $\rawlsrounding$. 
\begin{restatable}{lemma}{rawlsrounding}\label{lemma:rawlsrounding}
Let $\x^{\text{integ}}=\textsc{Rawlsian-Rounding}(\x^{\text{frac}})$ then the following holds.
\begin{align}
    \forall h \in \Colors: & \totdistcol^{\text{integ}}_h \leq \totdistcol^{\text{frac}}_h \label{eq:rawlsrounding_g1}\\
    \forall i \in \sr, \forall h \in \Colors: & \floor{\sum_{j \in \Pointsh} x_{ij}^{\text{frac}}} \leq \sum_{j \in \Pointsh}  x_{ij}^{\text{integ}} \leq \ceil{\sum_{j \in \Pointsh} x_{ij}^{\text{frac}}} \label{eq:rawlsrounding_g2}
\end{align}
\end{restatable}

\begin{figure}[h!]
  \centering
  \includegraphics[scale=0.4]{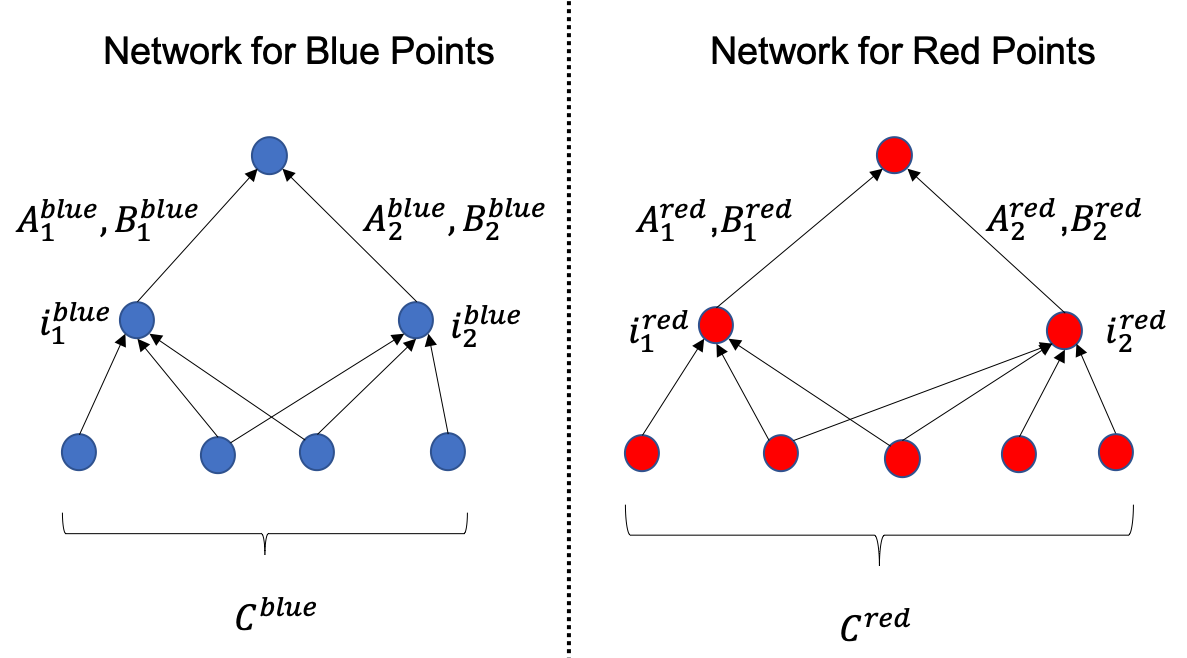}
     \caption{\textsc{Rawlsian-Rounding} for the case with two centers $S=\{i_1,i_2\}$. For any color $h$ and center $i$ we set flow lower and upper bound on the arcs between the centers and the sink to $A_{i}^{h} = \floor{\sum_{j \in \Points^{h}} x_{ij}^{\text{frac}}}$, $B_{i}^{h} = \ceil{\sum_{j \in \Points^{h}} x_{ij}^{\text{frac}}}$, respectively.  Other arcs, have a flow upper bound of $1$ and a lower bound of $0$. See Appendix \ref{app:rawlsrounding} for full details. }
   \label{fig:color_specific_rounding}
\end{figure}

In the above, Inequality \eqref{eq:rawlsrounding_g1} ensures that the value of $\totdistcol_h$ does not increase in the rounded integral solution for any color. Further, Inequality \eqref{eq:rawlsrounding_g2} ensures that the change in the total amount of color assigned to each cluster cannot vary by more than 1 (since it is between the floor and the ceiling). Based on the above, the following lemma can be established.
\begin{restatable}{lemma}{roundedfromLPrawls}\label{lemma:roundedfromLPrawls}
The rounded Rawlsian solution has a cost of at most $\OPTLPR+\constr$ where $\constr=\Big(\frac{|\Colors|+1}{\min\limits_{h\in \Colors} r_h} \cdot \frac{k}{n}\Big)$. 
\end{restatable} 
From Lemmas \ref{lemma:lpboundopt} and \ref{lemma:roundedfromLPrawls}, we can establish the the main guarantee for $\algr$.
\begin{restatable}{theorem}{rawlsmaing}\label{tm:rawls_main_g}
$\algr$ (Algorithm \ref{alg:algr}) returns a solution of Rawlsian cost at most $(\gammap+\gammapp \alpha) \OPTR +\constr$. 
\end{restatable} 

\begin{remark}
Since the number of colors $|\Colors|$ and the ratio $r_h$ of each color $h$ is also a constant. Our final guarantee is $(\gammap+\gammapp \alpha) \OPTR +\cO(\frac{k}{n})$.
\end{remark}

Although our final solution incurs an additional additive approximation of $\cO(\frac{k}{n})$ we note that it is not significant since in most applications we would have $n \gg k$, thus $\frac{k}{n}$ would be very small. Empirically, we do not observe any significant violations because of this additive factor.

\subsection{Algorithm for the Utilitarian Objective: $\algu$}\label{subsec:rawlsian}
Our algorithm for the Utilitarian Objective is $\algu$ whose pseudo-code is shown in Algorithm \ref{alg:algu}. Interestingly, although the details of $\algu$ are significantly different from $\algr$, the format remains very similar. In particular, $\algu$ begins by assigning a weight of $w_j=\frac{1}{|\Points^{\chi(j)}|}$ for each point $j \in \Points$ and then it runs \textsc{WightedClusteringAlg} which is an $\alpha$ approximation for the weighted clustering cost to find a set of centers $\su$. The rationale behind doing so is similar to the Rawlsian objective. Specifically, the Utilitarian objective is $\sum_{h \in \Colors} \disu_h = \sum_{h \in \Colors} \frac{\lambda \totdistcol_h + (1-\lambda)\totviolcol_h}{|\Points^h|}$, which ignoring the $(1-\lambda)\totviolcol_h$ term would equal a $\lambda$ weighted multiple of the weighted clustering cost of \eqref{eq:def_wighted} with the weight of each point $j$ set to $\frac{1}{|\Points^{\chi(j)}|}$. Similar, to Lemma \ref{lemma:existance_of_approx_solution} we can establish the following lemma. 
\begin{restatable}{lemma}{weightedapproxutil}\label{lemma:weighted_approx_util}
Using the set of centers $\Sw$ chosen by an  $\alpha$-approximation algorithm for the weighted clustering objective  with $w_j=\frac{1}{|\Points^{\chi(j)}|} \ \forall j \in \Points$, there exists an assignment $\phi'$ such that $(\Sw,\phi')$ leads to a $(\gammap +\gammapp\alpha)$-approximate solution for the Utilitarian objective. 
\end{restatable} 
Lemma \ref{lemma:weighted_approx_util} shares significant similarities with Lemma \ref{lemma:existance_of_approx_solution}
despite the fact that they address different objectives and are built using a different set of centers (socially fair vs weighted). In particular, the proof techniques are similar and both lemmas are non-constructive, i.e., the existence of assignment $\phi'$ is proved but without an algorithm for finding it. Furthermore, the increase in the approximation factor is identical and equal to $(2+\alpha)$ and $(6+4\alpha)$ for $p=1$ and $p=2$, respectively.  

Having found the set of centers $\su$ we proceed to find the assignment. We follow a similar approach to that in the Rawlsian objective. In particular, we setup LP \eqref{eq:assignment_lp_util} below. The interpretation of the variables in LP \eqref{eq:assignment_lp_util} are similar to those in LP \eqref{eq:assignment_lp}. Note however, we do not have a variable $z$ and directly minimizes the sum of the disutilities in the objective. 

\begin{subequations} \label{eq:assignment_lp_util}
\begin{align}
    \min \quad & \quad \quad \sum_{h \in \Colors} \frac{1}{|\Points^{h}|} \left( \lambda \sum_{j \in \Points^{h}} \sum_{i \in \su}  d^p(i, j)x_{ij} + (1-\lambda)  \sum_{i \in \su} t_{i, h}\right)  \label{obj_util:a} \\
    \text{subject to} & \quad \forall i \in \su, \forall h \in \Colors:  (r_h - \beta_h) \cdot \sum_{j \in \Points}x_{ij} - \sum_{j \in \Pointsh}x_{ij} = u_{i, h} \label{constraint_util:b}\\
    &\quad \forall i \in \su, \forall h\in \Colors:  \sum_{j \in \Pointsh}x_{ij} -(r_h+\alpha_h) \cdot \sum_{j \in \Points}x_{ij} = o_{i, h} \label{constraint_util:c}\\
    &\quad \forall i \in \su, \forall h \in \Colors: u_{i, h} \leq t_{i, h}, \ \ o_{i, h} \leq t_{i, h}, \ \ 0 \leq t_{i, h}\label{constraint_util:d}\\
    &\quad \forall j \in \Points: \sum_{i \in \su}x_{ij} = 1, \ \ x_{ij} \in [0, 1] \label{constraint_util:e} 
\end{align}
\end{subequations}

Based on similar reasoning we can establish a similar guarantee for the optimal value of LP \eqref{eq:assignment_lp_util} to Lemma \ref{lemma:lpboundopt} as shown below.

\begin{restatable}{lemma}{lpboundoptutil}\label{lemma:lpboundopt_util}
Denoting the optimal value of LP $\eqref{eq:assignment_lp_util}$ by $\OPTLPU$ then $\OPTLPU \leq (\gamma_p+\gamma_p' \alpha) \OPTU$. 
\end{restatable} 
Having found the optimal LP solution $\x^{\text{frac}}$ we round it using the \textsc{Utilitarian-Rounding} min-cost max-flow network. Unlike \textsc{Rawlsian-Rounding}, the \textsc{Utilitarian-Rounding} is very similar to the standard min-cost max-flow rounding done in \cite{bercea2018cost,esmaeili2021fair,dickerson2024doubly} with the difference being that we set the cost of an arc between the vertex for a point $j$ and a vertex for a center $i$ to $\frac{d^p(j,i)}{|\Points^{\chi(j)}|}$ instead of just $d^p(j,i)$ as usually done. See Appendix \ref{app:utilrounding} for the full details. Based on the construction of the the network, the following lemma can be proved. 
\begin{restatable}{lemma}{utilrounding}\label{lemma:utilrounding}
Let $\x^{\text{integ}}=\textsc{Utilitarian-Rounding}(\x^{\text{frac}})$ then the following holds.
\begin{align}
    & \sum_{h \in \Colors} \frac{1}{|\Points^h|}\totdistcol^{\text{integ}}_h \leq \sum_{h \in \Colors}  \frac{1}{|\Points^h|} \totdistcol^{\text{frac}}_h \label{eq:utilsrounding_g1}\\
    \forall i \in \su, \forall h \in \Colors: & \floor{\sum_{j \in \Pointsh} x_{ij}^{\text{frac}}} \leq \sum_{j \in \Pointsh}  x_{ij}^{\text{integ}} \leq \ceil{\sum_{j \in \Pointsh} x_{ij}^{\text{frac}}} \label{eq:utilsrounding_g2}\\
    \forall i \in \su: & \floor{\sum_{j \in \Points} x_{ij}^{\text{frac}}} \leq \sum_{j \in \Points}  x_{ij}^{\text{integ}} \leq \ceil{\sum_{j \in \Points} x_{ij}^{\text{frac}}} \label{eq:utilsrounding_g3}
\end{align}
\end{restatable} 
Inequality \eqref{eq:utilsrounding_g1} guarantees that $\sum_{h\in \Colors} \frac{1}{|\Points^h|} \totdistcol_h$ does not increase in the rounding. Further, Inequality \eqref{eq:utilsrounding_g2} is identical to Inequality \eqref{eq:rawlsrounding_g2} from \textsc{Rawlsian-Rounding} and guarantees that the amount of points from any color assigned to each cluster in the integral solution does not vary from the fractional amount by more than $1$. Interestingly, \textsc{Utilitarian-Rounding} has an additional inequality which is Inequality \eqref{eq:utilsrounding_g3} that also guarantees that the integral cluster size (total number of points in the cluster) also does not vary by more than 1 from  the fractional cluster size. We can now establish the following lemma.

\begin{restatable}{lemma}{roundedfromLPutil}\label{lemma:roundedfromLPutil}
The rounded Utilitarian solution has a cost of at most $\OPTLPU+\constu$ where $\constu=\Big( \frac{2k}{n} \cdot \sum_{h \in \Colors} \frac{1}{r_h} \Big)$. 
\end{restatable} 
Based on Lemma \ref{lemma:lpboundopt_util} and Lemma \ref{lemma:roundedfromLPutil}, the main guarantee for $\algu$ follows. 

\begin{restatable}{theorem}{utilmaing}\label{tm:util_main_g}
$\algu$ (Algorithm \ref{alg:algu}) returns a solution of Utilitarian cost at most $(\gammap+\gammapp \alpha) \OPTU +\constu$. 
\end{restatable} 

\begin{remark}
Similar to the guarantee of Theorem \ref{tm:rawls_main_g} the number of colors $|\Colors|$ and the ratio $r_h$ of each color $h$ is also a constant. Therefore, $\constu$ is a constant and our final guarantee is $(\gammap+\gammapp \alpha) \OPTR +\cO(\frac{k}{n})$.
\end{remark}

It is worthwhile to note that \textsc{Rawlsian-Rounding} could have been applied instead of \textsc{Utilitarian-Rounding}. However, since \textsc{Rawlsian-Rounding} lacks the guarantee of Inequality \eqref{eq:utilsrounding_g3} the constant $\constu$ would have increased. Specifically, it would instead be $\constu=\Big( (|\Colors|+1)\frac{k}{n} \cdot \sum_{h \in \Colors} \frac{1}{r_h} \Big)$ where $2$ has been replaced by $(|\Colors|+1)$. Finally, we also see empirically that the final $\cO(\frac{k}{n})$ term does not have a significant effect.

\begin{algorithm}[h!] 
\caption{$\algu$}\label{alg:algu}
    \begin{algorithmic}[1]
    \State \textbf{Input:} Set of points $\Points$, number of centers $k$ .
    \State Set $w_j=\frac{1}{|\Points^{\chi(j)}|} \ \forall j \in \Points$.
    \State $\su = \textsc{WightedClusteringAlg}(\Points,k,\{w_j\}_{j\in \Points})$. \label{alg_line:cs_u}
    \State Solve LP \eqref{eq:assignment_lp_util} to obtain the solution $\x^{\text{frac}}$. \label{alg_line:lp_u}
    \State Set $\x^{\text{integ}} = \textsc{Utilitarian-Rounding}(\x^{\text{frac}})$ and form the assignment $\phi$ using $\x^{\text{integ}}$.  \label{alg_line:round_u}
    \State \textbf{Return} $(\su,\phiu)$. 
    \end{algorithmic}
\end{algorithm}

\section{Experiments}\label{sec:experiments}
We evaluate performance of $\algr$ and $\algu$ across multiple datasets. 
Empirically we observe that $\algr$ and $\algu$ find solutions that dominate (are lower in the overall objective) the baselines in Rawlsian and the Utilitarian objectives. In the main paper, we include comparisons on three datasets for the two group setting. 
Additional experimental details and results including those for multiple groups ($|\Colors|>2$) are in Appendix~\ref{app:exp}.

\paragraph{Algorithms.} We compare $\algr$ with three other baselines: (1) \textsc{Vanilla} $k$-means; (2) \textsc{Socially Fair}: for the socially fair $k$-means objective; (3) \textsc{FCBC-Rawl}: the egalitarian variant in \cite{esmaeili2021fair} which minimizes maximum group fairness violation under bounded clustering distance. For $\algu$, we compare with: (1) \textsc{Vanilla} $k$-means; (2) \textsc{Weight Clustering}: weighted version of $k$-means with $w_j=\frac{1}{|\Points^{\chi(j)}|} \ \forall j \in \Points$; (3) \textsc{FCBC-Util}: the utilitarian variant in \cite{esmaeili2021fair} which minimizes sum of group fairness violation under bounded clustering distance. The distance measure is always the Euclidean distance.  We set $\alpha_h=\beta_h =\delta r_h$, where $\delta \in \{0.01, 0.1\}$. This ensures a fair comparison with FCBC algorithms which take parameter $\delta$. See more details of this translation in Appendix~\ref{app:exp}. We set the upper bound in FCBC to be $1.5$ times of vanilla $k$-means clustering cost. In Appendix~\ref{app:exp} , we show other upper bounds lead FCBC to perform worse.

\paragraph{Software and Packages.} All experiments are run on Python 3.10. We use \texttt{KMeans} and its weighted option from \texttt{Scikit-learn} \citep{pedregosa2011scikit} to find initial centers.

Linear programs are solved using the CPLEX package \citep{nickel2022ibm}. Network-flow rounding is done via \texttt{NetworkX} \citep{hagberg2013networkx}.  More details about hardware and implementation details can be found in Appendix~\ref{app:exp}.

\paragraph{Datasets.} We conduct experiments on four UCI datasets \citep{frank2010uci}: \adult{}, \credit{}, \cens{}

and \bank{}. We form feature vectors by choosing a set of columns as features in each dataset. We also choose a specific column as group membership. For \adult{}, \credit{}, and \cens{} the features and group memberships are selected following \cite{bera2019fair,esmaeili2020probabilistic} whereas for \bank{} we follow \cite{backurs2019scalable}. More details about datasets and normalization are in Appendix~\ref{app:exp}.

\begin{figure}[h!]
    \centering
    \setlength{\tabcolsep}{6pt}

    \begin{tabular}{@{}c c c c@{}}
        & \textbf{Adult} & \textbf{Credit Card} & \textbf{Census} \\

        \raisebox{4 em}{\rotatebox[origin=c]{90}{\normalsize\textbf{Rawlsian}}} &
        \includegraphics[width=0.28\textwidth]{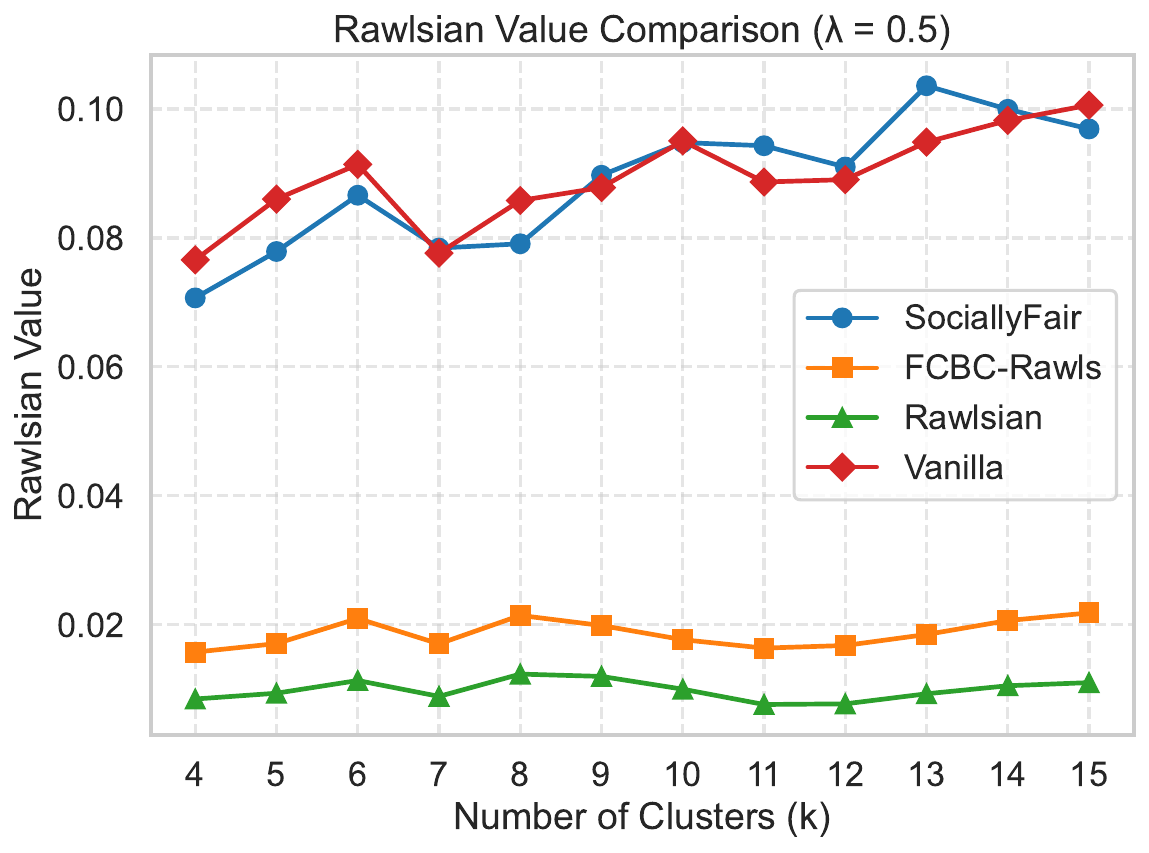} &
        \includegraphics[width=0.28\textwidth]{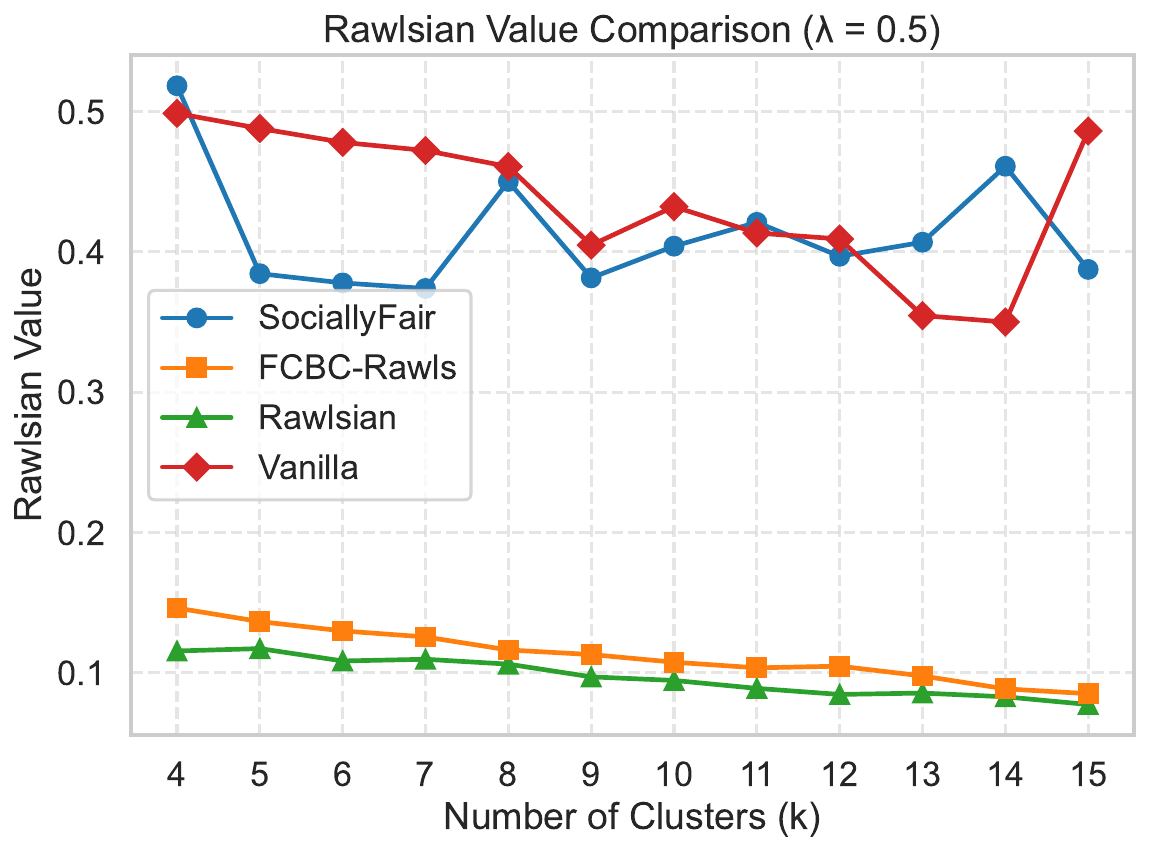} &
        \includegraphics[width=0.28\textwidth]{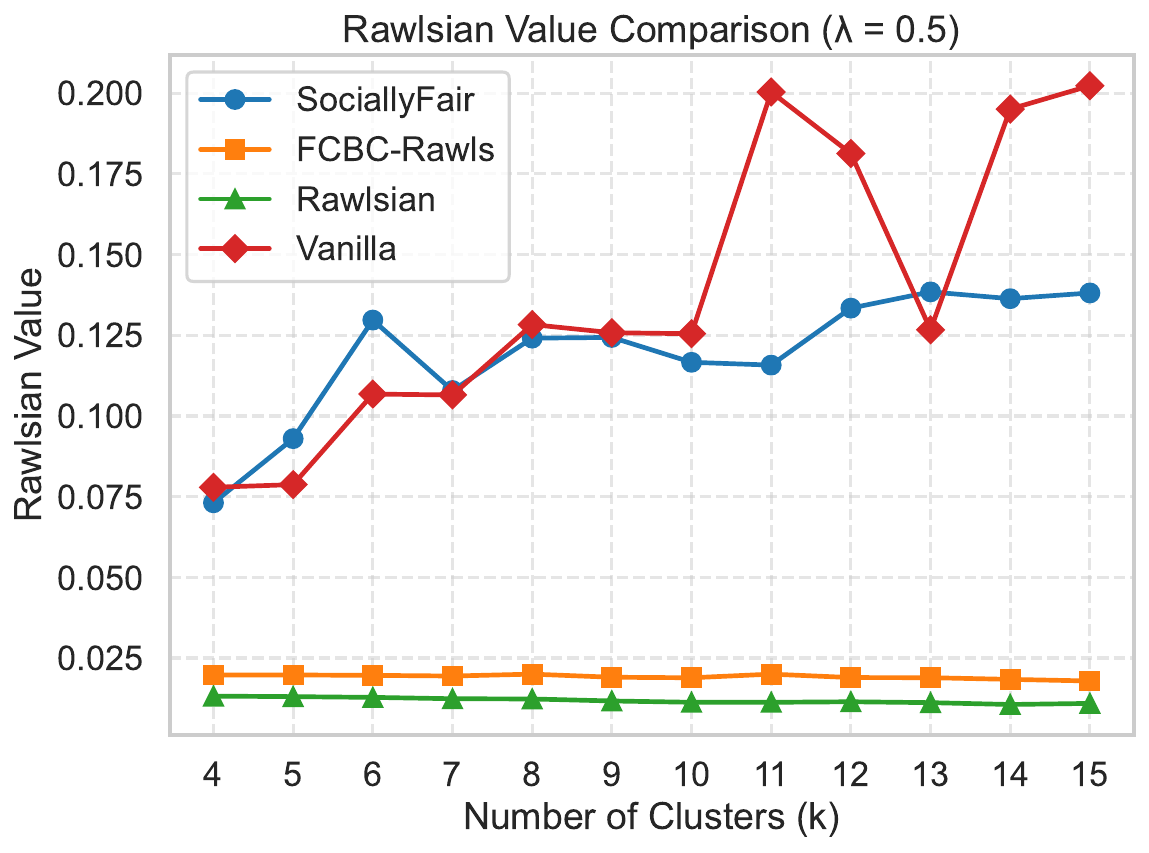} \\

        \raisebox{4 em}{\rotatebox[origin=c]{90}{\normalsize\textbf{Utilitarian}}} &
        \includegraphics[width=0.28\textwidth]{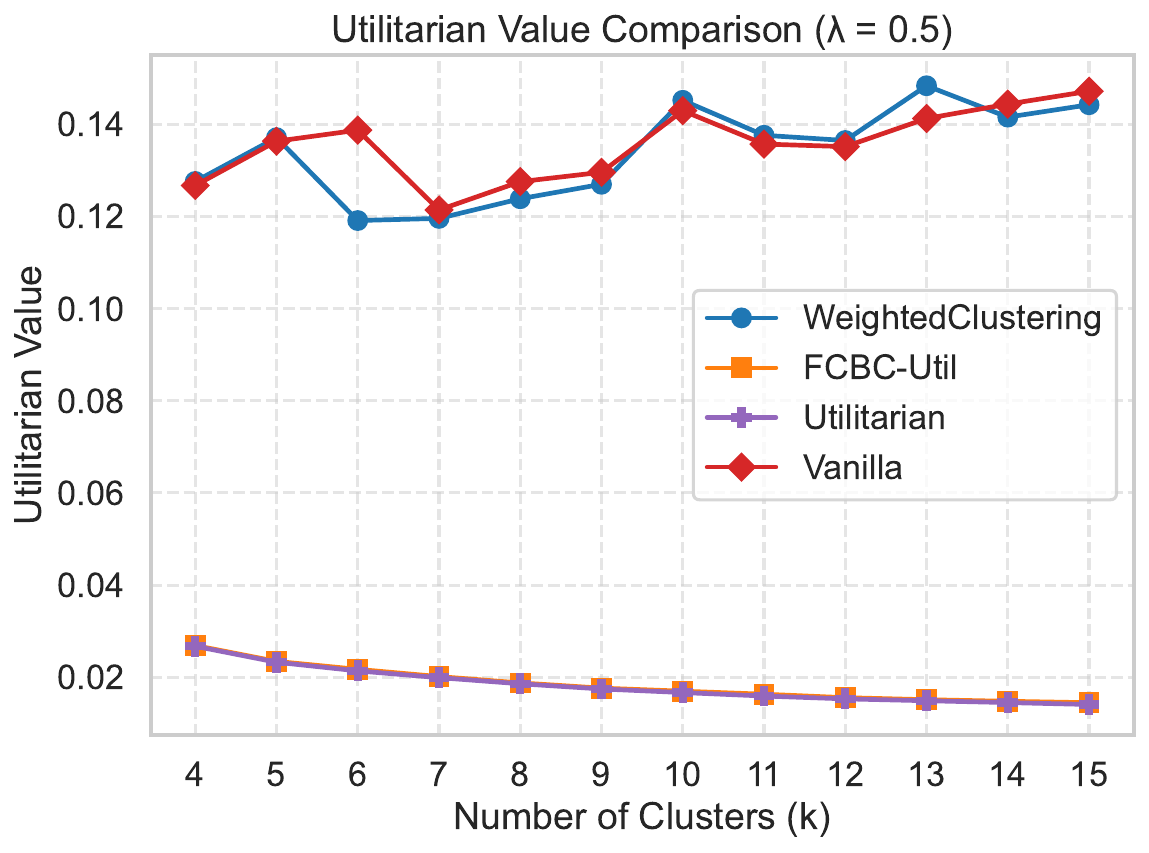} &
        \includegraphics[width=0.28\textwidth]{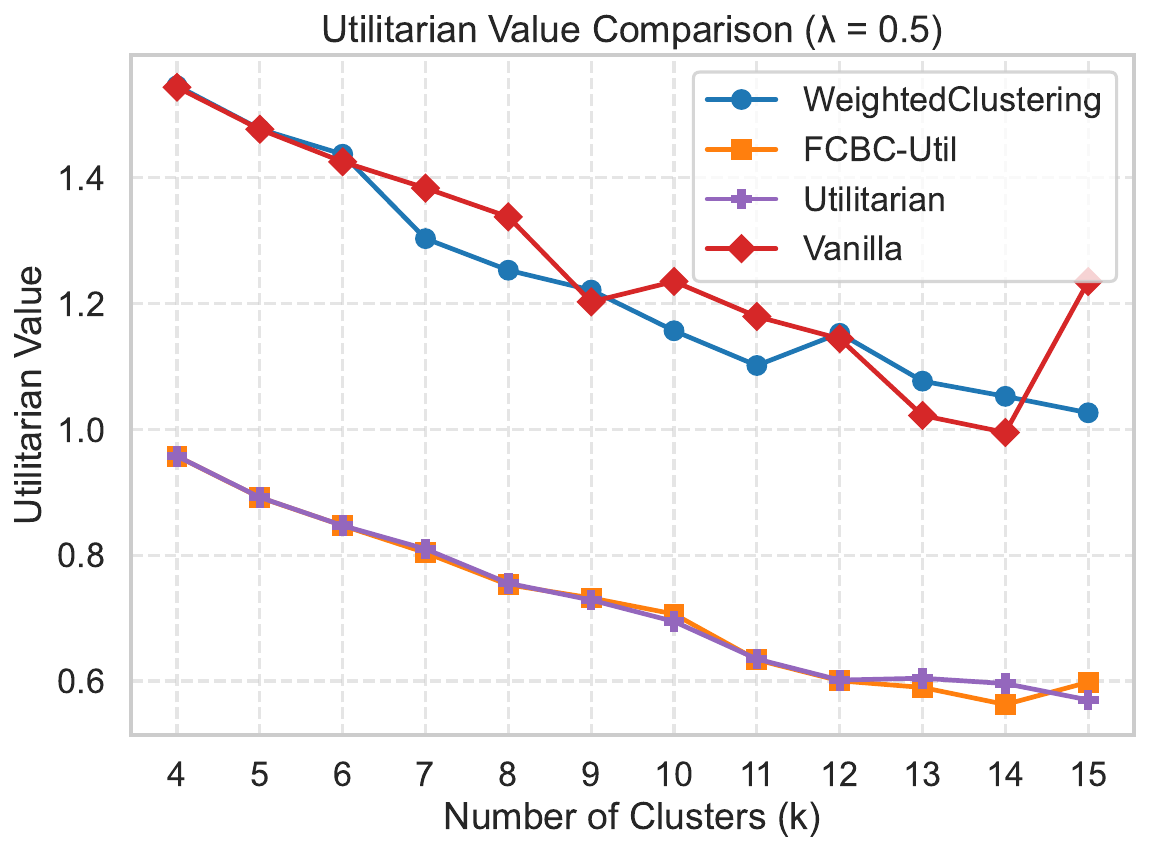} &
        \includegraphics[width=0.28\textwidth]{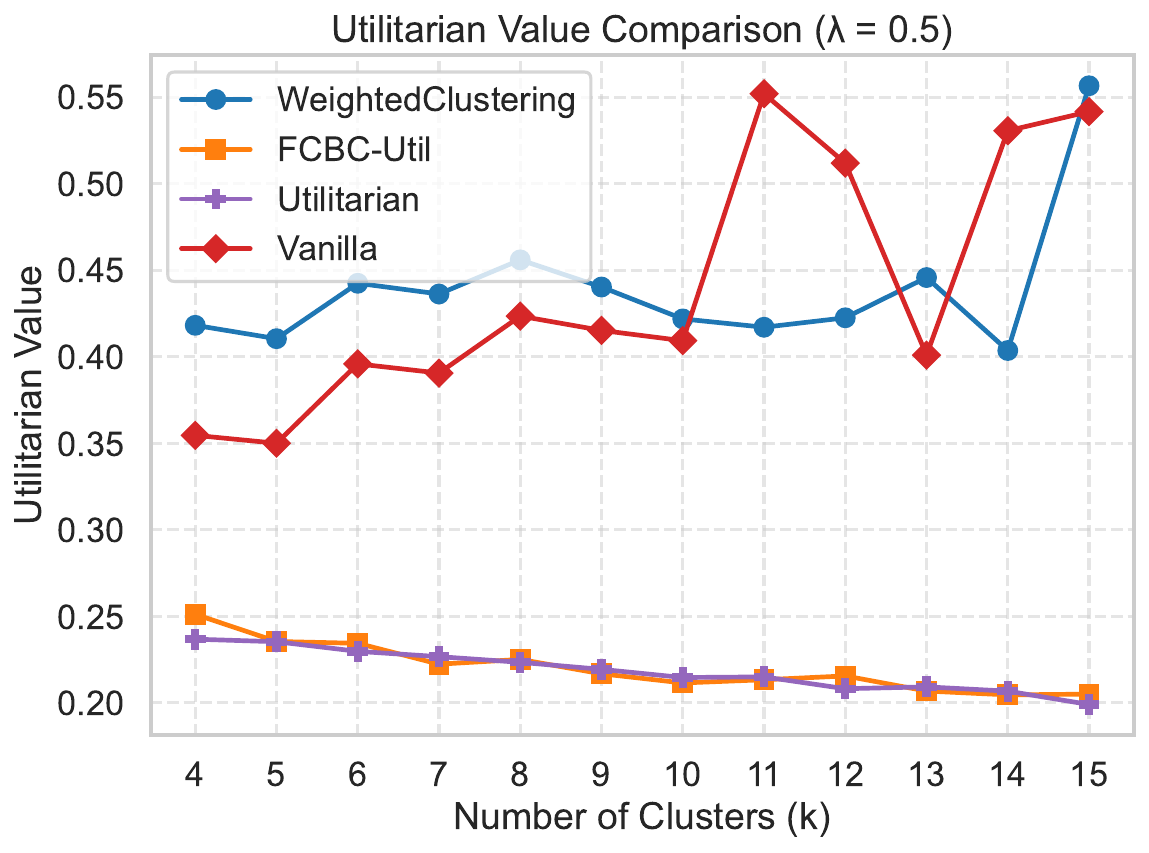} \\
    \end{tabular}

    \caption{Comparison of Rawlsian (top row) and Utilitarian (bottom row) objective values across three datasets. Each column corresponds to a dataset.}
    \label{fig:3_data_lambda_0_5}
\end{figure}

In Figure \ref{fig:3_data_lambda_0_5}, we compare the Rawlsian and Utilitarian objective values. In each plot, the number of centers $k$ is varied from $4$ to $15$. We present results for $\lambda = 0.5$. We observe that $\algr$ dominates the other baselines, and $\algu$ is equally competitive with \textsc{FCBC-Util} for $\lambda = 0.5$ on the three datasets. In Appendix \ref{app:exp}, we show plots for the full range of $\lambda$ where our algorithms find solutions that either are equally competitive or dominate existing methods by a larger amount including \textsc{FCBC-Util}.

\begin{ack}
Use unnumbered first level headings for the acknowledgments. All acknowledgments
go at the end of the paper before the list of references. Moreover, you are required to declare
funding (financial activities supporting the submitted work) and competing interests (related financial activities outside the submitted work).
More information about this disclosure can be found at: \url{https://neurips.cc/Conferences/2025/PaperInformation/FundingDisclosure}.

Do {\bf not} include this section in the anonymized submission, only in the final paper. You can use the \texttt{ack} environment provided in the style file to automatically hide this section in the anonymized submission.
\end{ack}


\bibliography{biblio}
\bibliographystyle{plainnat}

\newpage
\appendix
\section{Useful Facts}
\begin{fact}\label{claim:generalized_triangle_ineq}
For $p \in \{1,2\}$ and any three arbitrary points $a,b,$ and $c$ with a distance measure satisfying the triangle inequality  we have 
\begin{align}
    d^p(a,b) \leq 2^{p-1} \cdot \big(d^p(a,c)+d^p(c,b)\big) 
\end{align}
\end{fact}
\begin{proof}
For $p=1$ we have $2^{p-1}=1$ and the proof follows by the triangle inequality. For $p=2$ we have by squaring both sides of the triangle inequality:
\begin{align*}
    d^2(a,b) & \leq d^2(a,c) + d^2(c,b) + 2 d(a,c) d(b,c) \\
             & \leq 2 (d^2(a,c) + d^2(c,b)) 
\end{align*}
where in the above we used the fact that $d(a,c) d(b,c) \leq \frac{d^2(a,c) + d^2(c,b)}{2}$ by the AM-GM inequality.
\end{proof}

\section{Omitted Proofs} \label{sec:missing_proofs}
We restate our first lemma and give its proof. 
\existanceofapproxsol*
\begin{proof}
We denote by $(\Sstar, \phistar)$ an optimal solution for the Rawlsian objective.  $(\Ssf,\phisf)$ denote the solution of the $\alpha$-approximation solution for the socially fair objective. 
We construct the assignment $\phi'$ as follows: for every $\istar \in \Sstar$, define $\nrst(\istar)=\arg\min_{i \in \Ssf}d(i, \istar)$. For every $j \in \Points$, define $\phi'(j) = \nrst(\phistar(j))$. 

Before we delve into the details of the proof, we will denote by $\totdistcol^{*}_h$, $\totdistcol^{\SF}_h$, and $\totdistcol'_h$ the values of $\totdistcol_h(\Sstar,\phistar)$,  $\totdistcol_h(\Ssf,\phisf)$, and $\totdistcol_h(\Ssf,\phi')$, respectively. Similarly, for $\totviolcol^{*}_h$, $\totviolcol^{\SF}_h$, and $\totviolcol'_h$. Moreover, as noted earlier $\OPTR$ is the optimal value of the Rawlsian objective. Further, we let $\OPTSF$ be the optimal value of the socially fair objective when evaluated for the socially fair objective.  

We start with the following claim.
\begin{claim}\label{cl:intermediate}
\begin{align}
      \max\limits_{h \in \Colors}  \frac{\totdistcol_{h}^{\SF}}{|\Points^{h}|} \quad  \leq \quad \alpha \cdot \frac{1}{\lambda} \OPTR 
\end{align}
\end{claim}
\begin{proof}
Let $(\Ssf^*,\phisf^*)$ denote an optimal \SF solution and let $\totdistcol^{\SF^*}_h=\totdistcol_h(\Ssf^*,\phisf^*)$. We start by showing that $\OPTSF \leq \frac{1}{\lambda} \OPTR$. 

Suppose that was not true, i.e., $\frac{1}{\lambda}\OPTR <\OPTSF$. This implies that we have 
\begin{align*}
    \max_{h} \left(\frac{1}{|\Points^h|} \totdistcol^*_h + \frac{1-\lambda}{\lambda} \frac{1}{|\Points^h|}  \totviolcol^*_h \right) < \max_{h'}\frac{1}{|\Points^{h'}|}  \totdistcol^{\SF^*}_{h'}
\end{align*} 
Now note that $\frac{1-\lambda}{\lambda}  \frac{1}{|\Points^h|} \totviolcol^*_h \geq 0$. This means that
\begin{align*}
\max_{h} \frac{1}{|\Points^h|} \totdistcol^*_h < \max_{h'}\frac{1}{|\Points^{h'}|}  \totdistcol^{\SF^*}_{h'} = \OPTSF
\end{align*}
But this implies that by using $(\Sstar,\phistar)$ we would achieve a strictly more optimal solution for the socially fair objective which is a contradiction since $(\Ssf^*,\phisf^*)$ is an optimal solution for the $\SF$ objective. Therefore, we must have $\OPTSF \leq \frac{1}{\lambda} \OPTR$.

Now, using the above and since $(\Ssf,\phisf)$ results from an $\alpha$ approximation algorithm for the socially fair objective we must have $\Big( \max\limits_{h \in \Colors}  \frac{\totdistcol_{h}^{\SF}}{|\Points^{h}|} \Big)  \leq \alpha \OPTSF \leq \alpha \cdot \frac{1}{\lambda} \OPTR$. 
\end{proof}

We will now prove the following claim. 

\begin{claim}\label{cl:sf_bound_rawl}
$\forall h \in \Colors$ we have
\begin{align}
        \totdistcol'_h \leq \gammap \totdistcol^*_h \ + \ \gammapp \totdistcol_h^{\SF}
\end{align}

\end{claim}
\begin{proof}
We focus on a point $j$ and let $i^*=\phi^*(j)$, the following must hold:
\begin{align*}
    d^p(j,\phi'(j)) &=d^p(j, \nrst(\istar))\\
    &\leq 2^{p-1}(d^p(j, \istar)+d^p(\istar, \nrst(\istar))) \quad \text{(By Fact \ref{claim:generalized_triangle_ineq})}\\
    &\leq  2^{p-1}( d^p(j, \istar)+d^p(\istar, \phisf(j)) )\quad \text{(Since $\nrst(\istar)$ is the closest center to $i$ in $\Ssf$)} \\
    & \leq 2^{p-1}(d^p(j, \istar)+2^{p-1}(d^p(j, \phisf(j))+ d^p(j, \istar))) \quad \text{(By Fact \ref{claim:generalized_triangle_ineq})}\\
    & =2^{p-1}((2^{p-1}+1)d^p(j, \istar) + 2^{p-1} d^p(j, \phisf(j))) \\ 
    & = 2^{p-1}((2^{p-1}+1) d^p(j, \phi^*(j)) + 2^{p-1} d^p(j, \phisf(j))) \\ 
    & = \gammap d^p(j, \phi^*(j)) + \gammapp d^p(j, \phisf(j)))
\end{align*}
From the above, we get 
\begin{align*}
        \sum_{j \in \Pointsh}d^p(j, \phi'(j)) \leq \gammap \sum_{j \in \Pointsh}d^p(j, \phistar(j)) + \gammapp \sum_{j \in \Pointsh}d^p(j, \phisf(j)) 
\end{align*}
Which is equivalent to 
\begin{align}
    \totdistcol'_h \leq \gammap \totdistcol^*_h + \gammapp \totdistcol_h^{\SF}
\end{align} 
\end{proof}

Now we introduce the second claim for the violation. 
\begin{claim}\label{cl:viol_bound_rawl}
$\forall h \in \Colors$ we have
\begin{align}
    \totviolcol'_h \leq \totviolcol^*_h
\end{align}
\end{claim}
\begin{proof}
First, note that $\totviolcol^*_h= \sum_{i^* \in \Sstar} |C_{i^*}| \viol(h,i^*)$ and $\totviolcol'_h = \sum_{i\in \Ssf} |C'_i| \viol(h,i)$ where the cluster $C'_i$ is formed using the new assignment functio $\phi'$. If the assignment assignment function $\phi'$ routes to center $i \in \Ssf$ one cluster from a center $i^* \in \Sstar$, then it follows immediately that $|C'_i| \viol(h,i) = |C_{i^*}| \viol(h,i^*)$.

We will therefore, consider the case where center $i \in \Ssf$ recieves a collection of clusters from $\Sstar$. In particular, let $M$ denote the set of clusters routed to $i$ by $\phi'$. Let $O$ denote the subsets of clusters in $M$ in which color $h$ is over-represented (strictly greater than $\alpha_h + r_h$), $U$ denotes the subsets of clusters in $M$ where $h$ is under-represented (strictly less than $\alpha_h + r_h$), and $E$ denotes the subsets of clusters in $M$ that have no proportional violations. If the resulting proportion in cluster $i$ is between  $\rh-\beta$ and $r_h+\alpha_h$ then it follows that $\viol(h,i)=0$ and therefore $|C'_i| \viol(h,i) \leq \sum_{i^* \in M} |C_{i^*}| \viol(h,i^*)$. 

Therefore, we consider the case where the resulting cluster has an over-representation. The following must be true. 
\begin{align*}
|C'_i| \viol(h,i) &= \left(\sum_{i^* \in M} |C_{i^*}| \right)\left (\frac{\sum_{{i^*} \in M} |C_{i^*}^h| }{\sum_{{i^*} \in M } |C_{i^*}|} - (r_h + \alpha_h )\right )\\
&=\left ( \sum_{{i^*} \in M} |C_{i^*}^h| -(r_h + \alpha_h)\sum_{{i^*} \in M} |C_{i^*}| \right )\\
&=\sum_{{i^*} \in M} ( |C_{i^*}^h|-  (r_h + \alpha_h)|C_{i^*}|)\\
&= \sum_{{i^*} \in O}( |C_{i^*}^h|-  (r_h + \alpha_h)|C_{i^*}|) + \sum_{{i^*} \in U}( |C_{i^*}^h|-  (r_h + \alpha_h)|C_{i^*}|) + \sum_{{i^*} \in E}( |C_{i^*}^h|-  (r_h + \alpha_h)|C_{i^*}|)\\
&\leq \sum_{i^* \in O}| |C_{i^*}^h|-  (r_h + \alpha_h)|C_i^*||\\
&\leq \sum_{i^* \in M} |C_i^*| \viol(h, i^*)
\end{align*}
Now if we assume that cluster $i$ has under-representation, then the following is true. 
\begin{align*}
|C'_i| \viol(h,i)
&=\left (  (r_h - \beta_h) \sum_{{i^*}\in M} |C_{i^*}| -\sum_{{i^*} \in M} |C_{i^*}^h| \right )\\
&=\sum_{{i^*} \in M} ( (r_h - \beta_h)|C_{i^*}| - |C_{i^*}^h|  )\\
&= \sum_{{i^*} \in U}( (r_h - \beta_h)|C_{i^*}| - |C_{i^*}^h|) + \sum_{{i^*} \in O}( (r_h - \beta_h)|C_{i^*}| - |C_{i^*}^h| ) + \sum_{{i^*} \in E}| (r_h - \beta_h)|C_{i^*}| - |C_{i^*}^h||\\
&\leq \sum_{{i^*} \in U}| (r_h - \beta_h)|C_{i^*}| - |C_{i^*}^h||\\
&\leq \sum_{i^* \in M} |C_{i^*}| \viol(h, {i^*})
\end{align*}

Since each center in $\Sstar$ maps to one center in $\Ssf$, it follows that $\sum_{i\in \Ssf} |C'_i| \viol(h,i) \leq \sum_{i^* \in \Sstar} |C_{i^*}| \viol(h,i^*)$. Therefore, it follows that $\totviolcol'_h \leq \totviolcol^*_h$.
\end{proof}
The following claim must then hold.
\begin{claim}\label{cl:bound_term_r}
$\forall h \in \Colors$ we have
\begin{align}
    \disu'_h \leq \gammap \disu^*_h + \gammapp \lambda \frac{\totdistcol^{\SF}_h}{|\Points^h|}
\end{align}
\end{claim}
\begin{proof}
\begin{align*}
\lambda \totdistcol'_h & \leq   \gammap  \lambda   \totdistcol^*_h + 
\gammapp \lambda  \totdistcol_h^{\SF} \quad \quad \text{(By Claim \ref{cl:sf_bound_rawl})} \\ \implies     \lambda \totdistcol'_h + (1-\lambda) \totviolcol'_h & \leq  \gammap  \lambda   \totdistcol^*_h + 
\gammapp \lambda  \totdistcol_h^{\SF} + (1-\lambda) \totviolcol^*_h \quad \quad \text{(By Claim \ref{cl:viol_bound_rawl})}  \\ 
& \leq \gammap    \big(\lambda \totdistcol^*_h + (1-\lambda) \totviolcol^*_h \big) + \gammapp \lambda  \totdistcol_h^{\SF}  
\end{align*}
by multiplying both sides with $\frac{1}{|\Points^h|}$ we get the inequality. 
\end{proof}
Now we are ready to prove the lemma. Let $\hbar \in \argmax\limits_{h \in \Colors} \disu'_h$, the following is true

\begin{align*}
    \max\limits_{h \in \Colors} \disu'_h & = \disu'_{\hbar}  \\
    & \leq \gammap \disu^*_{\hbar} + \gammapp \lambda \frac{\totdistcol^{\SF}_{\hbar}}{|\Points^{\hbar}|}  \quad \quad \text{(By Claim \ref{cl:bound_term_r})}\\  
    & \leq \gammap \max\limits_{h' \in \Colors} \left(\disu^*_{h'} \right)+ \gammapp \cdot \lambda \cdot \max\limits_{h'' \in \Colors} \left(\frac{ \totdistcol^{\SF}_{h''} }{|\Points^{h''}|} \right) \\
    & \leq \gammap \OPTR + \gammapp \cdot \lambda \cdot \left( \alpha \frac{1}{\lambda} \OPTR  \right) \quad \quad \text{(By Claim \ref{cl:intermediate})} \\ 
    & = (\gammap + \gammapp \alpha) \OPTR  
\end{align*}
\end{proof}

We restate the next lemma and give its proof. 
\lpboundopt*
\begin{proof}
Consider some assignment $\phibar: \Points \to \Ssf$. For a point $j \in \Points$ set $x_{ij}=1$ if $\phibar(j)=i$ and zero otherwise. It follows that such an assignment satisfies constraint \eqref{constraint:f}. Further, it is not difficult to see that $t_{i,h}=|C_i| \viol(h,i)$ for assignment $\phibar$. Moreover, the left hand side of constraint \eqref{constraint:b} equals $\disu_h(\Ssf,\phibar)$. Therefore, since constraint \eqref{constraint:b} has $z\ge \disu_h(\Ssf,\phibar) \ \forall h \in \Colors$, it follows that $\min z = \min\limits_{h \in \Colors} \disu_h(\Ssf,\phibar)$. This proves that any integral assignment is feasible in LP \eqref{eq:assignment_lp} and that it evaluates to the same value. 

Since Lemma \ref{lemma:existance_of_approx_solution} establishes the existence of an assignment whose value is at most $(\gammap+\gammapp \alpha) \OPTR$. It follows that $\OPTLPR \leq (\gammap+\gammapp \alpha) \OPTR$.  
\end{proof}

We restate the next lemma and give its proof. 
\rawlsrounding*
\begin{proof}
We denote by $\x^{\text{frac}}_h$ the  the fractional solution restricted to the entries of points of color $h$, i.e., $j \in \Points^h$. It follows by construction of the network flow from Appendix \ref{app:rawlsrounding}, that for the network associated with each color $h$ that $\x^{\text{frac}}_h$ is a feasible assignment. 

Since we minimize the min-cost max-flow for each color separately and since all of the capacities and demands are set to integral values, by standard network flow guarantees we can find an integral assignment that is optimal. Therefore, Inequality \eqref{eq:rawlsrounding_g1} follows.

Further, by the demand values of $\floor{\sum_{j \in \Points^h} x^{\text{frac}}_{ij}}$ set at each vertex $v^h_i$ and the capacity of 1 on arc $(v^h_i,t^h)$ Inequality~\eqref{eq:rawlsrounding_g2} follows. 
\end{proof}

We restate the next lemma and give its proof. 
\roundedfromLPrawls* 
\begin{proof}

We begin with the following claim which can be proved based on Inequality~\eqref{eq:rawlsrounding_g2}. 

\begin{claim}\label{cl:cluster_bounds_rawls}
\begin{align}
  \forall i \in S: \sum_{j \in \Points} x_{ij}^{\text{frac}} -|\cH|  \leq \sum_{j \in \Points} x_{ij}^{\text{integ}} \leq \sum_{j \in \Points} x_{ij}^{\text{frac}} +|\cH| 
\end{align}
\end{claim}
\begin{proof}
First, by  Inequality \eqref{eq:rawlsrounding_g2} we have
\begin{align*}
\forall i \in S, \forall h \in \cH: \lfloor \sum_{j \in \Pointsh} x_{ij}^{\text{frac}}\rfloor \leq \sum_{j \in \Pointsh}  x_{ij}^{\text{integ}} \leq \lceil \sum_{j \in \Pointsh} x_{ij}^{\text{frac}} \rceil 
\end{align*}
Since $\floor{z}\ge z-1$ and $\ceil{z}\leq z+1$ we have 
\begin{align*}
  \forall i \in S, \forall h \in \cH:  \sum_{j \in \Pointsh} x_{ij}^{\text{frac}} -1  \leq \sum_{j \in \Pointsh}  x_{ij}^{\text{integ}} \leq \sum_{j \in \Pointsh} x_{ij}^{\text{frac}} +1 \\
\end{align*}
By summing over the colors we get
\begin{align*}
  \forall i \in S: \sum_{h \in \cH} \sum_{j \in \Pointsh} x_{ij}^{\text{frac}} -|\cH|  \leq \sum_{h \in \cH} \sum_{j \in \Pointsh}  x_{ij}^{\text{integ}} \leq \sum_{h \in \cH} \sum_{j \in \Pointsh} x_{ij}^{\text{frac}} +|\cH| \\
\end{align*}
Which is 
\begin{align}\label{color_specific_rounding:i_property}
  \forall i \in S: \sum_{j \in \Points} x_{ij}^{\text{frac}} -|\cH|  \leq \sum_{j \in \Points}  x_{ij}^{\text{integ}} \leq \sum_{j \in \Points} x_{ij}^{\text{frac}} +|\cH| 
\end{align}
\end{proof}
The following claim must then be true.
\begin{claim}
\begin{align}
\forall i \in \Ssf, \forall h \in \Colors: & \nonumber \\
& \sum_{j \in \Pointsh}  x_{ij}^{\text{integ}} \leq (r_h+\alpha_h)\sum_{j \in \Points}  x_{ij}^{\text{integ}} + (r_h+\alpha_h) |\cH| + t_{i,h}+ 1 \label{eq:upper} \\
& \sum_{j \in \Pointsh} x_{ij}^{\text{integ}} \ge (r_h - \beta_h)\sum_{j \in \Points}  x_{ij}^{\text{integ}} - (r_h - \beta_h) |\cH| -t_{i,h}-  1  \label{eq:lower}
\end{align}
\end{claim}
\begin{proof}
We start with the first Inequality \eqref{eq:upper}
\begin{align*}
    \sum_{j \in \Pointsh}   x_{ij}^{\text{integ}} &\leq \lceil \sum_{j \in \Pointsh} x_{ij}^{\text{frac}} \rceil \quad \quad \text{(By Inequality~\eqref{eq:rawlsrounding_g2})}\\
    &\leq  \sum_{j \in \Pointsh} x_{ij}^{\text{frac}} + 1\\
    &\leq (r_h+\alpha_h) \sum_{j \in \Points} x_{ij}^{\text{frac}} + t_{i,h}+ 1 \quad \text{(By LP constraints \eqref{constraint:d} and \eqref{constraint:e})}\\
    &\leq (r_h+\alpha_h) \left(\sum_{j \in \Points}  x_{ij}^{\text{integ}}  + |\cH|\right)+ t_{i,h}+ 1 \quad \text{(By Claim \eqref{cl:cluster_bounds_rawls})} \\
    &= (r_h+\alpha_h)\sum_{j \in \Points}  x_{ij}^{\text{integ}} + (r_h+\alpha_h) |\cH| + t_{i,h}+ 1 
\end{align*}

Similarly, for the second Inequality~\eqref{eq:lower} we have
\begin{align*}
 \sum_{j \in \Pointsh}  x_{ij}^{\text{integ}} &\geq \lfloor \sum_{j \in \Pointsh} x_{ij}^{\text{frac}}\rfloor \quad \quad \text{(By Inequality \eqref{eq:rawlsrounding_g2})}\\ 
 &\geq  \sum_{j \in \Pointsh} x_{ij}^{\text{frac}} - 1\\
 &\geq (r_h - \beta_h) \sum_{j \in \Points} x_{ij}^{\text{frac}} - t_{i,h}- 1 \quad \text{(By LP constraints \eqref{constraint:c} and \eqref{constraint:e})}\\
 &\geq (r_h - \beta_h)\left(\sum_{j \in \Points}  x_{ij}^{\text{integ}}  - |\cH|\right)- t_{i,h}- 1 \quad \text{(By Claim \ref{cl:cluster_bounds_rawls})} \\
 &=(r_h - \beta_h)\sum_{j \in \Points}  x_{ij}^{\text{integ}} - (r_h - \beta_h) |\cH| -t_{i,h}-  1 
\end{align*}
\end{proof}

We denote by $|C^{\text{integ}}_i|$ and $|C^{\text{frac}}_i|$ the size of the $i^{\text{th}}$ cluster according to the integral and fractional solutions,  respectively. The following claim can be derived. 
\begin{claim}\label{cl:bound_viol_term}
$\forall i \in \Ssf, \forall h \in \Colors$ we have $|C^{\text{integ}}_i| \viol^{\text{integ}}(h,i) \leq |C^{\text{frac}}_i| \viol^{\text{frac}}(h,i) + (|\Colors|+1)$.
\end{claim}
\begin{proof}
First, note that by LP constraints \eqref{constraint:c}, \eqref{constraint:d}, and \eqref{constraint:e} it follows that $|C^{\text{frac}}_i| \viol^{\text{frac}}(h,i)=t_{i,h}$. 

Now suppose that color $h$ in cluster $i$ is over-represented, then it follows that
\begin{align*}
    |C^{\text{integ}}_i| \viol^{\text{integ}}(h,i) & \leq \sum_{j \in \Pointsh}   x_{ij}^{\text{integ}} - (r_h+\alpha_h)\sum_{j \in \Points}   x_{ij}^{\text{integ}} \\
    & \leq (r_h+\alpha_h)\sum_{j \in \Points}  x_{ij}^{\text{integ}} + (r_h+\alpha_h) |\cH| + t_{i,h}+ 1 - (r_h+\alpha_h)\sum_{j \in \Points}   x_{ij}^{\text{integ}} \quad \quad \text{(By Inequality \eqref{eq:upper})} \\
    & = t_{i,h} + (r_h+\alpha_h) |\cH| + 1  \\ 
    & = |C^{\text{frac}}_i| \viol^{\text{frac}}(h,i) + (\left (r_h+\alpha_h) |\cH| + 1 \right) \\ 
    & \leq |C^{\text{frac}}_i| \viol^{\text{frac}}(h,i) + (\left |\cH| + 1 \right) \quad \quad \text{(Since $r_h+\alpha_h \leq 1$)}
\end{align*}
Similarly, if color $h$ in cluster $i$ is instead under-represented, then it follows that
\begin{align*}
    |C^{\text{integ}}_i| \viol^{\text{integ}}(h,i) & \leq (r_h-\beta_h)\sum_{j \in \Points}   x_{ij}^{\text{integ}} -  \sum_{j \in \Pointsh}   x_{ij}^{\text{integ}}   \\
    & \leq  (r_h-\beta_h)\sum_{j \in \Points}   x_{ij}^{\text{integ}} - 
    (r_h - \beta_h)\sum_{j \in \Points}  x_{ij}^{\text{integ}} + (r_h - \beta_h) |\cH| + t_{i,h} +  1 \quad \quad \text{(By Inequality \eqref{eq:lower})} \\
    & = t_{i,h} + (r_h - \beta_h) |\cH| + 1  \\ 
    & = |C^{\text{frac}}_i| \viol^{\text{frac}}(h,i) + (\left (r_h-\beta_h) |\cH| + 1 \right) \\ 
    & \leq |C^{\text{frac}}_i| \viol^{\text{frac}}(h,i) + (\left |\cH| + 1 \right) \quad \quad \text{(Since $r_h-\beta_h \leq 1$)}
\end{align*}
\end{proof}
Now we can easily prove the following claim.
\begin{claim}\label{cl:bound_avg_viol_rounded_rawls}
$\forall h \in \Colors$ we have:
\begin{align}
    \frac{\totviolcol^{\text{integ}}_h}{|\Points^h|} \leq \frac{\totviolcol^{\text{frac}}_h}{|\Points^h|} + \Big( \frac{|\Colors|+1}{r_h} \Big) \cdot \frac{k}{n}.
\end{align}
\end{claim}
\begin{proof}
\begin{align*}
    \frac{\totviolcol^{\text{integ}}_h}{|\Points^h|}  & = \frac{\sum_{i \in \Ssf}|C^{\text{integ}}_i| \viol^{\text{integ}}(h,i) }{|\Points^h|} \\
    & \leq  \frac{\sum_{i \in \Ssf} \Big( |C^{\text{integ}}_i| \viol^{\text{farc}}(h,i) + \big(|\Colors|+1\big)\Big)}{|\Points^h|} \quad \quad \text{(By Claim \ref{cl:bound_viol_term})} \\
    & = \frac{\totviolcol^{\text{frac}}_h}{|\Points^h|} + \Big( \frac{|\Colors|+1}{|\Points^h|} \Big) \cdot k \\ 
    & = \frac{\totviolcol^{\text{frac}}_h}{|\Points^h|} + \Big( \frac{|\Colors|+1}{r_h} \Big) \cdot \frac{k}{n} \quad \quad \text{(Since $|\Points^h|=r_h n$)} 
\end{align*}
\end{proof}
Now we ready to prove the lemma. Specifically, for any color $h \in \Colors$ by Inequality \ref{eq:rawlsrounding_g1} we have $\totdistcol^{\text{integ}}_h \leq \totdistcol^{\text{frac}}_h$. Further, by Claim \ref{cl:bound_avg_viol_rounded_rawls} we have $\frac{\totviolcol^{\text{integ}}_h}{|\Points^h|} \leq \frac{\totviolcol^{\text{frac}}_h}{|\Points^h|} + \Big( \frac{|\Colors|+1}{r_h} \Big) \cdot \frac{k}{n}$. Therefore we have 
\begin{align}
    \disu^{\text{integ}}_h = \frac{\lambda \totdistcol^{\text{integ}}_h+ (1-\lambda) \totviolcol^{\text{integ}}_h }{|\Points^h|} & \leq \frac{\lambda \totdistcol^{\text{frac}}_h+ (1-\lambda) \totviolcol^{\text{frac}}_h }{|\Points^h|} + (1-\lambda) \cdot \Big( \frac{|\Colors|+1}{r_h} \Big) \cdot \frac{k}{n} \\
    & = \disu^{\text{farc}}_h + (1-\lambda) \cdot \Big( \frac{|\Colors|+1}{r_h} \Big) \cdot \frac{k}{n} 
\end{align}
Since the above follows for each color it immediately follows that the cost is at most $\OPT^{\text{LP}}_R+\constr$ where $\constr=\Big( \frac{|\Colors|+1}{\min\limits_{h \in \Colors} r_h} \cdot \frac{k}{n}\Big)$.
\end{proof}

We restate the next theorem and give its proof.
\rawlsmaing*
\begin{proof}
Lemma \ref{lemma:existance_of_approx_solution} guarantees that we can find a fractional solution whose cost is at most $(\gammap+\gammapp\alpha)\OPTR$. Further, once this solution is rounded the cost would at most be $(\gammap+\gammapp\alpha)\OPTR+\constr$ by Lemma \ref{lemma:roundedfromLPrawls}. 
\end{proof}

We restate the next lemma and give its proof. 
\weightedapproxutil*
\begin{proof}
The proof is similar to the proof of Lemma \ref{lemma:existance_of_approx_solution}. We will denote by $(\Sstar, \phistar)$ an optimal solution for the Utilitarian objective.  $(\Sw,\phiw)$ denote the solution of the $\alpha$-approximation solution for the weighted clustering objective. 
We construct the assignment $\phi'$ as follows: for every $\istar \in \Sstar$, define $\nrst(\istar)=\arg\min_{i \in \Sw}d(i, \istar)$. For every $j \in \Points$, define $\phi'(j) = \nrst(\phistar(j))$. 

Similar to the proof of Lemma \ref{lemma:existance_of_approx_solution} we define the values  $\totdistcol^{*}_h$, $\totdistcol^{\W}_h$, and $\totdistcol'_h$ as the values of $\totdistcol_h(\Sstar,\phistar)$,  $\totdistcol_h(\Sw,\phiw)$, and $\totdistcol_h(\Sw,\phi')$, respectively. To avoid confusion, note that $\totdistcol_h(\Sw,\phi')$ evaluates the distances without any weighting, i.e., $\totdistcol_h(\Sw,\phi') = \sum_{j \in \Points^h} d(j,\phiw(j))$. Similar notion follow for $\totviolcol^{*}_h$, $\totviolcol^{\W}_h$, and $\totviolcol'_h$. Moreover, as noted earlier $\OPTU$ is the optimal value of the Utilitarian objective. Further, we let $\OPTW$ be the optimal value of the weighted clustering objective when evaluated on the weighted clustering objective.  

We start with the following claim which is similar to Claim \ref{cl:intermediate}.
\begin{claim}\label{cl:intermediate_util}
\begin{align}
     \sum_{h \in \Colors} \frac{1}{|\Points^{h}|} \totdistcol_h(\Sw,\phiw) \leq \alpha \cdot \frac{1}{\lambda} \OPTU 
\end{align}
\end{claim}
\begin{proof}
Let $(\Sw^*,\phiw^*)$ denote an optimal solution for weighted clustering and let $\totdistcol^{\W^*}_h=\totdistcol_h(\Sw^*,\phiw^*)$. We start by showing that $\OPTW \leq \frac{1}{\lambda} \OPTU$. 

Suppose that was not true, i.e., $\frac{1}{\lambda}\OPTU <\OPTW$. This implies that we have 
\begin{align*}
     \sum_{h \in \Colors} \left(\frac{1}{|\Points^h|} \totdistcol^*_h + \frac{1-\lambda}{\lambda} \frac{1}{|\Points^h|}  \totviolcol^*_h \right) <  \sum_{h \in \Colors} \frac{1}{|\Points^{h}|}  \totdistcol^{\W^*}_{h}
\end{align*} 
Since $ \sum_{h \in \Colors} \frac{1-\lambda}{\lambda}  \frac{1}{|\Points^h|} \totviolcol^*_h \geq 0$, this means that
\begin{align*}
 \sum_{h \in \Colors} \frac{1}{|\Points^h|} \totdistcol^*_h <  \sum_{h \in \Colors} \frac{1}{|\Points^{h}|}  \totdistcol^{\W^*}_{h} = \OPTW
\end{align*}
But this implies that by using $(\Sstar,\phistar)$ we would achieve a strictly more optimal solution for the weighted clustering objective which is a contradiction by optimality of $(\Sw^*,\phiw^*)$ for weighted clustering. Thus, $\OPTW \leq \frac{1}{\lambda} \OPTU$.

Moreover, since $(\Sw,\phiw)$ is an $\alpha$-approximation for weighted clustering, it follows that $\sum_{h \in \Colors} \frac{1}{|\Points^{h}|} \totdistcol_h(\Sw,\phiw) \leq \alpha \sum_{h \in \Colors} \frac{1}{|\Points^{h}|}  \totdistcol^{\W^*}_{h} \leq \alpha \cdot \frac{1}{\lambda} \OPTU$. 
\end{proof}

Following proofs similar to the proofs of Claims \ref{cl:sf_bound_rawl}, \ref{cl:viol_bound_rawl}, and \ref{cl:bound_term_r} with the socially fair solution replaced with the weighted clustering solution. The following three claims can be proved. 

\begin{claim}\label{cl:sf_bound_util}
$\forall h \in \Colors$ we have
\begin{align}
        \totdistcol'_h \leq \gammap \totdistcol^*_h \ + \ \gammapp \totdistcol_h^{\W}
\end{align}
\end{claim}

\begin{claim}\label{cl:viol_bound_util}
$\forall h \in \Colors$ we have
\begin{align}
    \totviolcol'_h \leq \totviolcol^*_h
\end{align}
\end{claim}

\begin{claim}\label{cl:bound_term_util}
$\forall h \in \Colors$ we have
\begin{align}
     \disu'_h  \leq \gammap  \disu^*_h  + \gammapp \lambda  \frac{\totdistcol_h^{\W}}{|\Points^h|} 
\end{align}

\end{claim}
Now we are ready to prove the lemma. 
\begin{align*}
    \sum_{h \in \Colors} \disu'_h & \leq \gammap \big(\sum_{h \in \Colors} \disu^*_h \big) + \gammapp \lambda \Big( \sum_{h \in \Colors} \frac{\totdistcol_h^{\W}}{|\Points^h|} \Big) \quad \quad \text{(By Claim \ref{cl:bound_term_util})} \\
    & = \gammap \OPTU + \gammapp \lambda \Big( \sum_{h \in \Colors} \frac{\totdistcol_h^{\W}}{|\Points^h|} \Big) \\
    & \leq \gammap \OPTU + \gammapp \lambda \big( \alpha \frac{1}{\lambda} \OPTU \big) \quad \quad \text{(By Claim \ref{cl:intermediate_util})} \\ 
    & = (\gammap + \gammapp \alpha) \OPTU 
\end{align*}
\end{proof}

We restate the next lemma and give its proof. 
\lpboundoptutil* 
\begin{proof}
Similar to the proof of Lemma \ref{lemma:lpboundopt} we can consider an assignment $\phibar: \Points \to \su$ and set $x_{ij}=1$ only if $\phibar(j)=i$ and zero otherwise. It follows that this assignment would be feasible and that $t_{i,h}$ equals the violation $|C_i| \viol(h,i)$ resulting from $\phibar$. Further, the objective in \eqref{obj_util:a} equals $\sum_{h \in \Colors} \disu_h(\su,\phibar)$. This proves that any integral assignment in feasible in LP \eqref{eq:assignment_lp_util} and that the LP objective value equals the Utilitarian objective of that assignment. 

Now, since Lemma \ref{lemma:weighted_approx_util} establishes the existence of an assignment of cost at most $(\gammap+\gammapp \alpha) \OPTU$ it follows that $\OPTLPU \leq (\gammap+\gammapp \alpha) \OPTU$

\end{proof}

We restate the next lemma and give its proof. 
\utilrounding*
\begin{proof}
Similar to the proof of Lemma \ref{lemma:rawlsrounding}, one can see that in the network constructed in Appendix \ref{app:utilrounding} the fractional solution $\x^{\text{frac}}$ is feasible. 

Since the capacities and demand values in the network are integral it follows that there exists an optimal integral solution. Hence, Inequality \eqref{eq:utilsrounding_g1} follows.

Furthermore, by the demand and capacity values set 
inequalities \eqref{eq:utilsrounding_g2} and  \eqref{eq:utilsrounding_g3} immediately follow. 
\end{proof}

\roundedfromLPutil*

\begin{proof}

We begin with the following claim holds immediately by Inequality \eqref{eq:utilsrounding_g3} since for any number $z$ we have $\floor{z} \ge z-1$ and  $\ceil{z} \leq z+1$. 
\begin{claim}\label{cl:cluster_bounds_util}
\begin{align}
  \forall i \in S: \sum_{j \in \Points} x_{ij}^{\text{frac}} -1  \leq \sum_{j \in \Points} x_{ij}^{\text{integ}} \leq \sum_{j \in \Points} x_{ij}^{\text{frac}} +1
\end{align}
\end{claim}

The following claim can be proved.
\begin{claim}
\begin{align}
\forall i \in \su, \forall h \in \Colors: & \nonumber \\
& \sum_{j \in \Pointsh}  x_{ij}^{\text{integ}} \leq (r_h+\alpha_h)\sum_{j \in \Points}  x_{ij}^{\text{integ}} +  t_{i,h} + 2 \label{eq:upper_util} \\
& \sum_{j \in \Pointsh} x_{ij}^{\text{integ}} \ge (r_h - \beta_h)\sum_{j \in \Points}  x_{ij}^{\text{integ}} -  t_{i,h} -  2  \label{eq:lower_util}
\end{align}
\end{claim}
\begin{proof}
We start with the first Inequality \eqref{eq:upper}
\begin{align*}
    \sum_{j \in \Pointsh}   x_{ij}^{\text{integ}} &\leq \lceil \sum_{j \in \Pointsh} x_{ij}^{\text{frac}} \rceil \quad \quad \text{(By Inequality \eqref{eq:utilsrounding_g2})}\\
    &\leq  \sum_{j \in \Pointsh} x_{ij}^{\text{frac}} + 1\\
    &\leq (r_h+\alpha_h) \sum_{j \in \Points} x_{ij}^{\text{frac}} + t_{i,h}+ 1 \quad \text{(By LP constraints \eqref{constraint_util:c} and \eqref{constraint_util:d})}\\
    &\leq (r_h+\alpha_h) \left(\sum_{j \in \Points}  x_{ij}^{\text{integ}}  + 1\right)+ t_{i,h}+ 1 \quad \text{(By Claim \eqref{cl:cluster_bounds_util})} \\
    &= (r_h+\alpha_h)\sum_{j \in \Points}  x_{ij}^{\text{integ}} + t_{i,h}+ 2 \quad \quad \text{(Since $r_h+\alpha_h \leq 1$)}
\end{align*}

Similarly, for the second Inequality \eqref{eq:lower} we have
\begin{align*}
 \sum_{j \in \Pointsh}  x_{ij}^{\text{integ}} &\geq \lfloor \sum_{j \in \Pointsh} x_{ij}^{\text{frac}}\rfloor \quad \quad \text{(By Inequality \eqref{eq:utilsrounding_g2})}\\ 
 &\geq  \sum_{j \in \Pointsh} x_{ij}^{\text{frac}} - 1\\
 &\geq (r_h - \beta_h) \sum_{j \in \Points} x_{ij}^{\text{frac}} - t_{i,h}- 1 \quad \text{(By LP constraints \eqref{constraint_util:b} and \eqref{constraint_util:d})}\\
 &\geq (r_h - \beta_h)\left(\sum_{j \in \Points}  x_{ij}^{\text{integ}}  - 1\right)- t_{i,h}- 1 \quad \text{(By Claim \ref{cl:cluster_bounds_rawls})} \\
 &=(r_h - \beta_h)\sum_{j \in \Points}  x_{ij}^{\text{integ}} - (r_h - \beta_h)  -t_{i,h}-  1  \\
  &=(r_h - \beta_h)\sum_{j \in \Points}  x_{ij}^{\text{integ}} - t_{i,h}-  2 \quad \quad \text{(Since $r_h+\alpha_h \leq 1$)}
\end{align*}
\end{proof}

We denote by $|C^{\text{integ}}_i|$ and $|C^{\text{frac}}_i|$ the size of the $i^{\text{th}}$ cluster according to the integral and fractional solutions,  respectively. The following claim can be derived. 
\begin{claim}\label{cl:bound_viol_term_util}
$\forall i \in \su, \forall h \in \Colors$ we have $|C^{\text{integ}}_i| \viol^{\text{integ}}(h,i) \leq |C^{\text{frac}}_i| \viol^{\text{frac}}(h,i) + 2$.
\end{claim}
\begin{proof}
First, note that by LP constraints \eqref{constraint_util:b}, \eqref{constraint_util:c}, and \eqref{constraint_util:d} it follows that $|C^{\text{frac}}_i| \viol^{\text{frac}}(h,i)=t_{i,h}$. 

Now suppose that color $h$ in cluster $i$ is over-represented, then it follows that
\begin{align*}
    |C^{\text{integ}}_i| \viol^{\text{integ}}(h,i) & \leq \sum_{j \in \Pointsh}   x_{ij}^{\text{integ}} - (r_h+\alpha_h)\sum_{j \in \Points}   x_{ij}^{\text{integ}} \\
    & \leq (r_h+\alpha_h)\sum_{j \in \Points}  x_{ij}^{\text{integ}} +  t_{i,h}+ 2 - (r_h+\alpha_h)\sum_{j \in \Points}   x_{ij}^{\text{integ}} \quad \quad \text{(By Inequality \eqref{eq:upper_util})} \\
    & = t_{i,h} + 2  \\ 
    & = |C^{\text{frac}}_i| \viol^{\text{frac}}(h,i) + 2 
\end{align*}
Similarly, if color $h$ in cluster $i$ is instead under-represented, then it follows that
\begin{align*}
    |C^{\text{integ}}_i| \viol^{\text{integ}}(h,i) & \leq (r_h-\beta_h)\sum_{j \in \Points}   x_{ij}^{\text{integ}} -  \sum_{j \in \Pointsh}   x_{ij}^{\text{integ}}   \\
    & \leq  (r_h-\beta_h)\sum_{j \in \Points}   x_{ij}^{\text{integ}} - 
    (r_h - \beta_h)\sum_{j \in \Points}  x_{ij}^{\text{integ}} + t_{i,h} +  2 \quad \quad \text{(By Inequality \eqref{eq:lower_util})} \\
    & = t_{i,h} + + 2  \\ 
    & = |C^{\text{frac}}_i| \viol^{\text{frac}}(h,i) + 2  
\end{align*}
\end{proof}
Now we can easily prove the following claim.
\begin{claim}\label{cl:bound_avg_viol_rounded_util}
\begin{align}
    \sum_{h \in \Colors} \frac{\totviolcol^{\text{integ}}_h}{|\Points^h|} \leq \sum_{h \in \Colors} \frac{\totviolcol^{\text{frac}}_h}{|\Points^h|} + \Big( \frac{2k}{n}  \cdot  
    \sum_{h \in \Colors} \frac{1}{r_h} \Big) 
\end{align}
\end{claim}
\begin{proof}
\begin{align}
    \sum_{h \in \Colors} \frac{\totviolcol^{\text{integ}}_h}{|\Points^h|} & =  \sum_{h \in \Colors} \frac{\sum_{i\in S} |C_i^{\text{integ}}|\viol^{\text{integ}}(h,i)}{|\Points^h|} \\
    & \leq \sum_{h \in \Colors} \frac{\sum_{i\in S} (|C_i^{\text{frac}}|\viol^{\text{frac}}(h,i) + 2 )}{|\Points^h|} \quad \quad \text{(By Claim \ref{cl:bound_viol_term_util})} \\
    & = \sum_{h \in \Colors} \frac{\totviolcol^{\text{frac}}_h}{|\Points^h|} + \Big( 2k \cdot \sum_{h \in \Colors} \frac{1}{|\Points^h|} \Big)  \\
    & = \sum_{h \in \Colors} \frac{\totviolcol^{\text{frac}}_h}{|\Points^h|} + \Big( \frac{2k}{n} \cdot \sum_{h \in \Colors} \frac{1}{r_h} \Big)  \quad \quad \text{(Since $|\Points^h|=r_h n$)}
\end{align}
\end{proof}
Now we ready to prove the lemma. 
\begin{align*}
    \sum_{h \in \Colors} \disu^{\text{integ}}_h & = \sum_{h \in \Colors} \frac{\lambda \totdistcol^{\text{integ}}_h + (1-\lambda) \totviolcol^{\text{integ}}_h}{|\Points^h|} \\
    & \leq \sum_{h \in \Colors} \frac{\lambda \totdistcol^{\text{farc}}_h + (1-\lambda)\totviolcol^{\text{frac}}_h}{|\Points^h|} + (1-\lambda) \Big( \frac{2k}{n} \cdot \sum_{h \in \Colors} \frac{1}{r_h} \Big) \quad \quad \text{(By Inequality \eqref{eq:utilsrounding_g1} and Claim \ref{cl:bound_avg_viol_rounded_util})} \\ 
    & = \sum_{h \in \Colors} \disu^{\text{frac}}_h +  (1-\lambda) \Big( \frac{2k}{n} \cdot \sum_{h \in \Colors} \frac{1}{r_h} \Big) \\
    & \leq \sum_{h \in \Colors} \disu^{\text{frac}}_h + \constu 
\end{align*}
Where $\constu = \Big( \frac{2k}{n} \cdot \sum_{h \in \Colors} \frac{1}{r_h} \Big)$. 

\end{proof}

We restate the next theorem and give its proof. 
\utilmaing* 
\begin{proof}
The LP will return a fractional solution of cost at most $(\gammap+\gammapp\alpha)\OPTU$ by Lemma \ref{lemma:lpboundopt_util}. Further, Lemma \ref{lemma:roundedfromLPutil} implies that the rounded solution from the LP will have a cost of at most $(\gammap+\gammapp\alpha)\OPTU+\constu$.
\end{proof}
\section{NP-Hardness of the Assignment Problem}\label{app:hardness_of_assignment}
Here we show that finding the optimal Rawlsian or Utilitarian assignment $\phi$ given a set of centers $S$ is NP-hard. 

Following the same reduction from Exact Cover by $3$-Sets as that of \cite{esmaeili2021fair}. See also Theorem 5.1 and Figure 4 in \cite{esmaeili2021fair}. Our instance will be identical. In particular, we would have $r_{\text{red}}=\frac{1}{4}$ and $r_{\text{blue}}=\frac{3}{4}$. Using our notation we would set $\alpha_h=\beta_h=0 \ , \forall h  \in \{\text{red},\text{blue}\}$. We use the same set of vertices $\mathcal{F}$ for the set of given centers. Is it immediate that the minimum disutility for the red color is $\disu_{\text{red}}\ge \frac{\lambda (0\cdot \frac{|\Points^{\text{red}}|}{2}+1\cdot \frac{|\Points^{\text{red}}|}{2}) +(1-\lambda)\cdot 0}{|\Points^{\text{red}}|} = \frac{1}{2} \lambda$. Similarly, for the blue color, we have  $\disu_{\text{blue}} \ge \frac{\lambda (1\cdot |\Points^{\text{blue}}|) +(1-\lambda)\cdot 0}{|\Points^{\text{blue}}|} = \lambda$. We assume $\lambda>0$. It follows straightforwardly that we have a Rawlsian objective of $\lambda$ and a Utilitarian objective of $\frac{3}{2}\lambda$ if and only if Exact Cover by $3$-Sets has a solution. This is proves that the assignment problem is NP-hard for both the Rawlsian and Utilitarian objectives.

\section{The Network Flow Rounding Scheme}

Our rounding schemes build upon standard network flow rounding procedure introduced in \citet{bercea2018cost}. We first introduce the standard version as appeared in \citet{bercea2018cost} and will then introduce our variants in their subsections.

The input to network rounding scheme is fractional solution to Linear Program, $\x^{\text{frac}}$ and center set $S$ and distances $d(i, j)$, $\forall i \in S, j \in \Points$. In standard rounding, one network is constructed.  Denote the network as $G = (V, E)$ where $V$ is the set of vertices and $E$ the set of arcs. Components of graph $G$ are listed below.
\begin{enumerate}
    \item $V =  \{v_j| j \in \Points \} \cup \{v_i^h | h \in \Colors, i \in S \}\cup \{v_i|i \in S\} \cup \{t\} $
    \item $E =  \{(v_j, v_i^h )| x_{ij}^{\text{frac}}>0\} \cup \{(v_i^h, v_i)| h \in \Colors, i \in S \}\cup\{(v_i,t)| i \in S \}$ 
\end{enumerate}

Specifically, we create one vertex $v_j$ for each point $j \in \Points$ with demand $-1$.   For each center $i \in S $ we define a set of vertices $v_i^h$ for each color $h \in \cH$, and a separate vertex $v_i$. Let each $v_i^h$ has demand $\lfloor \sum_{j \in \Pointsh} x_{ij}\rfloor$, and $v_i$ has a demand $\left(\lfloor \sum_{j \in \Points} x_{ij}\rfloor - \sum_{h \in \Colors}\lfloor \sum_{j \in \Pointsh} x_{ij}\rfloor\right)$.
We create a set of arcs $(v_i^h, v_i)$ for each $i$ and $h$. Let these arcs have cost $0$.  For each $x_{ij}^{\text{frac}}>0$, we create an arc $(v_j, v_{i}^h)$ for the $h$ such that $j \in \Pointsh$. For this arc $(v_j, v_{i}^h)$ we set its cost to be $d^p(i,j)$.

We define a sink vertex $t$ with demand $(|\Points| - \sum_{i \in S} \lfloor \sum_{j \in \Points} x_{ij} \rfloor)$ and connect $t$ with $v_i$ for each $i\in S$ with an arc. Arcs $\{(v_i,t)| i \in S \}$ have costs 0. We let all arcs have unit capacity.

Because all capacities and demands in this network are integral, an integral solution can be found in polynomial time. 

It is not difficult to see that $\x^{\text{frac}}$ is feasible and therefore the cost of the integral solution cannot exceed that of the fractional solution.

By construction, we have the two inequalities shown in Lemma~\ref{lemma:utilrounding}, 
\eqref{eq:utilsrounding_g2} and  \eqref{eq:utilsrounding_g3}.

\subsection{Details of the Rawlsian Rounding}\label{app:rawlsrounding}

We construct the Rawlsian Rounding min-cost flow instance on fractional solution $\x^{\text{frac}}$ from LP \eqref{eq:assignment_lp}. The network instance contains unconnected independent networks, one for each color $h \in \Colors$ as illustrated in Figure \ref{fig:color_specific_rounding}. We denote $G^h$ as network of color $h$. $G^h = (V^h, E^h).$ Components of $G^h$ are listed below.
\begin{enumerate}
    \item $V^h = \{v_j|j \in \Pointsh\} \cup\{v_i^h|i \in \sr\} \cup t^h$
    \item $E^h = \{(v_j, v_i^h)| j \in \Pointsh, i \in \sr\}\cup \{(v_i^h, t^h) | i \in \sr \}$
\end{enumerate}
  For each fixed color $h \in \cH$, for each center $i \in \sr$ we define a set of vertices $v_i^h$. Each of $v_i^h$ has demand $\lfloor \sum_{j \in \Pointsh} x_{ij}^{\text{frac}}\rfloor$ for the corresponding center $i$. Note $x_{ij}^{\text{frac}} $ are elements of solution $\x^{\text{frac}}$. For each point $j \in \Pointsh$, we create a vertex $v_j$, with demand $-1$. Between $v_j$ and $v_i^h$, we create arc $(v_j, v_i^h)$ if $x_{ij}>0$. We let arcs $(v_j, v_i^h)$ each has capacity 1. We include a sink vertex $t^h$ which has demands $(|\Pointsh|-\sum_{i \in \sr} \lfloor \sum_{j \in \Pointsh} x_{ij}^{\text{frac}}\rfloor)$. We connect each of $v_i^h$ for $i \in \sr$  with  the sink $t^h$ with an arc $(v_i^h, t^h)$. These arcs have capacity 1 as well. On the set of arcs $(v_j, v_i^h)$ we assign cost $\frac{1}{|\Pointsh|}d^p(i, j)$. For other arcs in this network, we assign cost 0.

On each color specific network constructed as described above, given the capacity, demands are integers, there exist an integral optimal min-cost flow solution. Note the fractional assignment from LP solution constrained to only  entries of color $h$,  $ \x_{h}^{\text{frac}} = \{x_{ij}^{\text{frac}} \}_{j\in \Pointsh}$   constitutes a feasible fractional flow solution to the min-cost flow problem.

\subsection{Details of the Utilitarian Rounding}\label{app:utilrounding}

For Utilitarian Rounding, we take fractional solution of Linear Program \eqref{eq:assignment_lp_util}, centers $\su$, and distances $d(i, j)$, $\forall i \in S, j \in \Points$ as input. We set up one network whose the components $G=(V, E)$ are the same as described as standard network setup at beginning of this section. The only difference is we set costs of arcs $(v_j, v_{i}^h)$ to $\frac{d^p(i,j)}{|\Points^{\chi(j)}|}$. Network structure remain the same. Costs of other arcs, demands of vertices, capacities of all arcs remain the same.

\section{Necessity of Solving the Assignment Problem}
Here we start by showing that the value of the average violation $\totviolcol_h$ for any color $h \in \Colors$ cannot exceed $2$, (see Fact \ref{fact:maximum_fair_cost} below). Further, Lemmas \ref{lemma:existance_of_approx_solution} and \ref{lemma:weighted_approx_util} establish the existence of an assignment with a bounded approximation ratio for the Rawlsian and Utilitarian objectives, respectively. Based on the above, one  might wonder if optimizing for the assignment as done in the LP and rounding steps of $\algr$ and $\algu$ is even necessary, perhaps due to the boundedness of the average violation we would never incur a large degradation in the approximation. We show in Theorem \ref{thm:withoutphi_unbounded_approx_ratio} that optimizing for the assignment is in fact necessary as otherwise the approximation ratio would be unbounded while the additive approximation would be a constant unlike our additive approximation which is on the order of $\cO(\frac{k}{n})$ which converges to zero as $n \to \infty$.

\begin{fact}\label{fact:maximum_fair_cost}
For any color $h$ the normalized violation is upper bounded by $2$. More specifically, $\frac{\totviolcol_h}{|\Pointsh|}  \leq 2(1-r_h) \leq 2$.
\end{fact}

\begin{proof}
We set $\alpha_h=\beta_h=0 \ , \forall h \in \Colors$. Note that by definition of $\viol(h,i)$ in Eq.~\eqref{eq:def_piece_wise_linear_asym_fairness_in_each_point} setting   $\alpha_h=\beta_h=0$ would only lead to a higher value for $\totviolcol_h$ in comparison to any other values for $\alpha_h$ and $\beta_h$. 

Now, consider a fixed color $h$. For a center $i \in S$ set $p_i= \frac{|\cih|}{|\Pointsh|}$. That is, $p_i$ is the proportion of color $h$ points assigned to cluster $i$ in the given clustering. By definition $\sum_{i\in S} p_i=1$. Denote number of points in cluster $\ci$ having colors in $\cH$ other than color $h$ by $C_i^{\bar h}$. Accordingly, we have 
\begin{align}
     \frac{\totviolcol_h}{|\Pointsh|} & = \frac{1}{|\Pointsh|}\sum_{i \in S} |\ci| \cdot \viol(h,i) \nonumber \\ 
     &= \frac{1}{|\Pointsh|} \sum_{i \in S}  \left \vert |\ci| r_h - |\cih|\right \vert \nonumber \\
     &= \sum_{i \in S} \left \vert \frac{|\ci|}{|\Pointsh|} r_h - p_i \right \vert \nonumber \\
     &= \sum_{i \in S} \left \vert \frac{|\ci|}{|\Points|}  - p_i \right \vert \nonumber \\
     &= \sum_{i \in S} \left \vert   \frac{|C_i^{\bar h} | + p_i |\Pointsh|}{|\Points|}  - p_i  \right \vert \nonumber \\
     &=\sum_{i \in S} \left \vert \frac{|C_i^{\bar h} | }{|\Points|} + \Big(\frac{|\Pointsh|}{|\Points|}-1\Big) p_i  \right \vert \nonumber \\
     &=\sum_{i \in S} \left \vert \frac{|C_i^{\bar h} | }{|\Points|} + (r_h-1) p_i \right \vert \nonumber \\
     &= \sum_{i \in S} \left \vert \frac{|C_i^{\bar h} | }{|\Points|} - (1-r_h) p_i \right \vert \nonumber\\
     &= (1-r_h) \sum_{i \in S} \left \vert \frac{|C_i^{\bar h} | }{|\Points|(1-r_h) } -  p_i \right \vert \label{eq:total_cost_h_over_clusters}
\end{align}

When maximizing over assignment $\phi$, we essentially maximize over $p_i$ and $|\Points_i^{\bar h} |$ over $i \in S$ because other quantities in the objective are fixed. That is 
\begin{align*}
\max_{\phi}\frac{1}{|\Pointsh|}\sum_{i \in S} |\ci| \cdot \viol(h,i) = \max_{p_i,|C_i^{\bar h} |}(1-r_h) \sum_{i \in S} \left \vert \frac{|C_i^{\bar h} | }{|\Points|(1-r_h) } -  p_i \right \vert 
\end{align*}

Note consider the assignment of points of colors in $\cH$ other than $h$, i.e., $|C_i^{\bar h} |$ over $i \in S$. Define $q_i=\frac{|C_i^{\bar h} | }{|\Points|(1-r_h) }$. By definition $\sum_{i=1}^k q_i = 1$. With this notation, 
\eqref{eq:total_cost_h_over_clusters} can be written as
\begin{align*}
\max_{p_i,|C_i^{\bar h} |}(1-r_h) \sum_{i \in S} \left \vert \frac{|C_i^{\bar h} | }{|\Points|(1-r_h) } -  p_i \right \vert & = (1-r_h) \max_{p_i,q_i} \sum_{i \in S}  |q_i - p_i| \\
& = (1-r_h) \max_{\mathbf{p} , \mathbf{q}} \|\mathbf{q} - \mathbf{p}\|_1  \\ 
& =  (1-r_h) \cdot 2 \quad \quad \text{(Since $ \max_{\mathbf{p} , \mathbf{q}} \|\mathbf{q} - \mathbf{p}\|_1 = 2$ as $\mathbf{p}$ and $\mathbf{q}$ are probability vectors)}\\
& \leq 2 \quad \quad \text{(Since $(1-r_h) \leq 1$)}
\end{align*}
\end{proof}

Now we present that theorem which implies that optimizing the assignment is necessary.

\begin{restatable}{theorem}{sfunboundedapproxratio}\label{thm:withoutphi_unbounded_approx_ratio}
Denote by $(\Ssf^*,\phisf^*)$ the optimal socially fair solution and by $(\Sw^*,\phiw^*)$ the optimal weighted clustering solution with $w_j=\frac{1}{|\Points^h|} \ , \forall j \in \Points$. Then there exists an instance where $(\Ssf^*,\phisf^*)$  has a Rawlsian objective that is an unbounded multiple of the optimal Rawlsian objective plus an  additive constant. Similarly, $(\Sw^*,\phiw^*)$ has a Utilitarian objective that is an unbounded multiple of the optimal Utilitarian objective plus an additive constant. Specifically, we have 
\begin{align}
    R(\Ssf^*,\phisf^*) & \ge \rho_1  \OPTR+ \rho_2 \\
    U(\Sw^*,\phiw^*) & \ge \rho_1  \OPTU+ \rho_2
\end{align}
where $\rho_1$ is an arbitrarily large positive value and $\rho_2$ is a positive constant. 
\end{restatable}

\begin{proof}
\begin{figure}[h]
    \centering
    \includegraphics[width=0.4\linewidth]{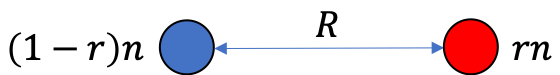}
    \caption{Instance on which socially fair solution has unbounded Rawlsian and Utilitarian objective value.}
    \label{fig:sf_unbounded_approx_ratio}
\end{figure}
We start with the proof for the socially fair solution and the Rawlsian objective. 

Consider the instance in Figure \ref{fig:sf_unbounded_approx_ratio}. In this instance, there are $(1-r)n$ blue points and $rn$ red points where $r \in (0,1)$. All blue points coincide with one another and all red points coincide with one another with both colors being separated by a distance of $R$. Further, we set $\alpha_h=\beta_h=0 \ , \forall h \in \Colors$. For this instance, with $k=2$, the optimal socially fair solution will have two centers $\Ssf^*$ with one center at the blue points and another center at the red points. Clearly, $\phisf^*$ will assign each point to its closest center. The disutilities for the blue and red groups would be:
\begin{align*}
    & \disu_{\text{blue}}(\Ssf^*,\phisf^*) = \lambda \cdot 0 + (1-\lambda) \cdot \frac{|(1-r)(1-r)n-(1-r)n| + |(1-r)rn-0|}{(1-r) n} = 2 (1-\lambda) r \\
    & \disu_{\text{red}}(\Ssf^*,\phisf^*) = \lambda \cdot 0 + (1-\lambda) \cdot \frac{|r(1-r)n-0| + |r\cdot rn-rn|}{r n} = 2 (1-\lambda) (1-r) 
\end{align*}
For simplicity, we set $r=\frac{1}{2}$, thus $R(\Ssf^*,\phisf^*)=\disu_{\text{blue}}(\Ssf^*,\phisf^*)=\disu_{\text{red}}(\Ssf^*,\phisf^*)=(1-\lambda)$. 

Consider a solution $(\Ssf^*,\phi')$, i.e., using the same centers as the optimal socially fair solution but with a different assignment $\phi'$. Specifically, $\phi'$ assigns half the blue points to the right center and half the red points to the left center. It is clear that assignment $\phi'$ would lead to no proportional violations. The disutilities would thus be
\begin{align*}
    & \disu_{\text{blue}}(\Ssf^*,\phi') = \lambda \cdot \frac{\frac{(1-r)n}{2} R^p}{(1-r)n} + (1-\lambda) \cdot 0 = \frac{\lambda}{2} R^p \\
    & \disu_{\text{red}}(\Ssf^*,\phi') = \lambda \cdot \frac{\frac{rn}{2} R^p}{rn} + (1-\lambda) \cdot 0 = \frac{\lambda}{2} R^p 
\end{align*}
Thus for this solution we would have $R(\Ssf^*,\phi')=\disu_{\text{blue}}(\Ssf^*,\phi')=\disu_{\text{red}}(\Ssf^*,\phi')=\frac{\lambda}{2} R^p$. 

Now we set $R^p=2\cdot \frac{1-\lambda}{\lambda}\cdot \theta$ where $\theta \in (0,1)$. This leads to $R(\Ssf^*,\phi')\leq \theta (1-\lambda) < (1-\lambda) =R(\Ssf^*,\phisf^*)$. 

Based on the above, using $\gamma \in [0,1]$ we can write $R(\Ssf^*,\phisf^*)$ in terms of $R(\Ssf^*,\phi')$ as follows:
\begin{align*}
    R(\Ssf^*,\phisf^*) = \gamma \cdot \frac{R(\Ssf^*,\phi')}{\theta}+ (1-\gamma) \cdot (1-\lambda)
\end{align*}
Which is equivalent to 
\begin{align}
    R(\Ssf^*,\phisf^*) = \rho_1 \cdot R(\Ssf^*,\phi') + \rho_2 \label{eq:obj_bound}
\end{align}
Where $\rho_1=\frac{\gamma}{\theta}$ and $\rho_2=(1-\gamma) \cdot (1-\lambda)$. It is clear that for $\gamma \in (0,1)$, $\rho_2$ is a positive constant and $\rho_1$ is a constant that can be made arbitrarily large by choosing an instance with sufficiently small $\theta$. 

Since $(\Ssf^*,\phi')$ is a feasible solution for Rawlsian clustering it follows that $R(\Ssf^*,\phi') \ge \OPTR$. Based on \eqref{eq:obj_bound} this leads to
\begin{align*}
    R(\Ssf^*,\phisf^*) \ge \rho_1 \cdot \OPTR + \rho_2 
\end{align*}

For the weighted clustering solution and the Utilitarian objective. Simply note that $(\Sw^*,\phiw^*)=(\Ssf^*,\phisf^*)$, i.e., the weighted clustering solution is equal to the socially fair solution. We would have $U(\Sw^*,\phiw^*)=2(1-\lambda)$. Following a the same assignment $\phi'$ we would have $U(\Sw^*,\phi')=\lambda R^p$. Further, we still have $R^p=2\frac{1-\lambda}{\lambda}\theta$ where $\theta \in (0,1)$ we have $U(\Sw^*,\phi') < U(\Sw^*,\phiw^*)$. Therefore, we can similarly establish the following: 
\begin{align*}
    U(\Sw^*,\phiw^*)  = \frac{\gamma}{\theta} \cdot U(\Sw^*,\phi') + 2 (1-\gamma) (1-\lambda)
\end{align*}
Since $\OPTU\leq U(\Sw^*,\phi')$ then we have the final guarantee.
\begin{align*}
    U(\Sw^*,\phiw^*) \ge \rho_1 \cdot \OPTU + \rho_2 
\end{align*}
With $\rho_1=\frac{\gamma}{\theta}$ and $\rho_2=2 (1-\gamma) (1-\lambda)$ .
\end{proof}

\section{Comparing Welfare-Centric Clustering to FCBC}\label{app:fcbc_comparison}
The FCBC formulation unlike ours is framed in terms of clustering cost not disutilities and welfare. This causes FCBC to find solutions that are acceptable in terms of the clustering cost but clearly when one considers the disutilities and welfare. To more easily demonstrate how clustering cost, i.e., $\sum_{j\in \Points} d^p(j,\phi(j))$ would different from the Rawlsian and Utilitarian objectives, we will assume that $\lambda=1$, so that $(1-\lambda)=0$ and proportional violations would not incur a cost. 

Consider the simple example shown in Figure \ref{fig:app_comap_to_fcbc}. In that example we have $r_{\text{blue}}=\frac{2}{3}$ and $r_{\text{red}}=r_{\text{green}}=\frac{1}{6}$. There are $4$ special locations in the space $\{a,b,c,o\}$. All red, blue, and green points are at locations $a$,$b$, and $c$, respectively. For ease of discussion we consider the $k$-median problem and the Hamming distance. This implies that $d(a,o)=d(b,o)=d(c,o)=1$ and that $d(a,b)=d(a,c)=d(b,c)=2$. 

Consider the solution $(S,\phi)$ where $S=\{o\}$ and $\phi$ assigns all points to center $o$. Then we would have $\disu_{\text{red}}=\disu_{\text{green}}=\disu_{\text{blue}}=1$. This also leads to $R(S,\phi)=1$ and $U(S,\phi)=3$. The clustering cost would be $\sum_{j\in \Points} d^p(j,\phi(j))=1\cdot n=n$.

Now, consider a less equitable solution $(S',\phi')$ where $S'=\{b\}$ and all points are assigned to center $b$. This would lead to $\disu_{\text{red}}=\disu_{\text{green}}=2$ while $\disu_{\text{blue}}=0$. The Rawlsian and Utilitarian objectives would become clearly higher. Specifically, we would have $R(S',\phi')=2$ and $U(S',\phi')=2\cdot 2=4$. On the other hand, the clustering cost would drop since $\sum_{j\in \Points} d^p(j,\phi'(j))=2 \cdot \frac{1}{3} \cdot n=\frac{2}{3}n < n = \sum_{j\in \Points} d^p(j,\phi(j))$. 

Since FCBC uses the clustering costs as an upper bound it would favor such solutions which are clearly inequitable. This highlights the importance of formulating the problem in terms of welfare instead of naive view of simple clustering costs.  

\begin{figure}[h!]
    \centering
    \includegraphics[width=0.2\linewidth]{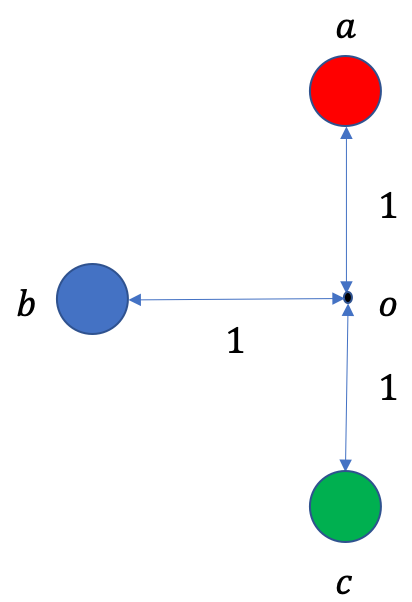}
    \caption{Clustering instance.}
    \label{fig:app_comap_to_fcbc}
\end{figure}

\newpage 
\section{Additional Experiments and Results} 
\label{app:exp}
In this section, we give details of the experiments and show additional experiment results.

\subsection{Details for Datasets Used in Experiments}
\begin{table}[h!]
\centering
\renewcommand{\arraystretch}{1.2}
\begin{tabular}{|c|c|c|c|c|}
\hline
\textbf{Dataset} & \textbf{Dimension} & \textbf{Number of Points} & \textbf{Group Label}&\textbf{$\alpha_h$,$\beta_h$} \\
\hline\hline
Adult & 5 & 32561 & gender &$0.01r_h$\\
\hline
Credit card &14 & 29623  & marital-status &$0.1r_h$ \\
\hline

Census II (subsampled) & 67 & 20,000 & gender &$0.01r_h$\\
\hline
\end{tabular}
\caption{Description of the four datasets used in two group empirical evaluations.}
\label{tab:two_group_dataset_description}
\end{table}

We select the same group labels and features for \adult{}, \credit{}, and \cens{} as done in \cite{bera2019fair,esmaeili2020probabilistic}. 
On \adult{}, attributes "age", "final-weight", "education-num", "capital-gain", "hours-per-week" are used as features. 
On \credit{} and \cens{}, all attributes except group label are used as features. 
 Gender is used as the group label in \adult{} and \cens{}. On \credit{}, we use marital status. 

\begin{table}[H]
\centering
\renewcommand{\arraystretch}{1.2}
\begin{tabular}{|c|c|c|c|c|}
\hline
\textbf{Dataset} & \textbf{Dimension} & \textbf{Number of points} & \textbf{Group Label}&\textbf{$\alpha_h$,$\beta_h$} \\
\hline\hline
Bank & 3 & 4521 & marital-status ($|\Colors|=3$) &$ 0.01r_h$\\
\hline

\end{tabular}
\caption{Description of dataset used in our multi-group empirical evaluation. }
\label{tab:multi_group_dataset_description}
\end{table}

On \bank{}, attributes chosen to represent point in space are "age", "balance" and "duration-of-account", the same as in \citet{backurs2019scalable}.

\subsection{Implementation Details}
\paragraph{Existing Implementation of Socially Fair and FCBC} We use existing implementation of socially fair $k$-means from \citet{ghadiri2021socially} \footnote{https://github.com/samirasamadi/SociallyFairKMeans} which accommodates groups with two colors but not multiple colors. Note existing code for socially fair is in written Matlab, we translate to equivalent Python implementation. We use existing code for FCBC implementation \footnote{https://github.com/Seyed2357/Fair-Clustering-Under-Bounded-Cost}.

\paragraph{Randomness in Center Selection} 
Both our algorithms and FCBC variants uses some form of $k$-means to select  initial centers (socially fair can be setup to runs $k$-means to select initial centers instead of randomly select). Thus randomness in the experimental pipeline lies in the initial centers selected by each algorithm. To minimizes the effects of selecting bad centers, we let all algorithms being compared each run $k$-means++ to initialize centers and select the best out of $10$ sets of centers.

\paragraph{Translation of Parameters $\alpha$, $\beta$ between Our Algorithms and FCBC}
As mentioned in Section \ref{sec:experiments}, we set $\alpha_h=\beta_h =\delta r_h$, where $\delta \in \{0.01, 0.1\}$. Note the same symbol $\alpha_h$ and $\beta_h$ refer to different concepts in paper \citet{esmaeili2021fair} where FCBC algorithms were originally introduced. Denote the upper and lower bounds in \citet{esmaeili2021fair} as $\bar \alpha_h$, $\bar \beta_h$. To fairly compare with FCBC algorithms, we set $\bar \alpha_h = r_h + \alpha_h$, $\bar \beta_h = r_h + \beta_h$, where $\alpha_h$, $\beta_h$ are parameters for our formulation.

\subsection{Data Normalization}

We normalize the data so that the average clustering distance under vanilla $k$-means is on the same order as its proportional violation. This ensures that the distance and violation terms are comparable in magnitude and therefore avoids extreme values for the trade-off parameter $\lambda$. We now introduce details of normalization factor calculation. Denote a normalization factor for each $k$ as $F_{\text{Rawl}}^k$ and $F_{\text{Util}}^k$.  We calculate $F_{\text{Rawl}}^k$ as 
\begin{align*}
    F_{\text{Rawl}}^k =  \frac{\frac{1}{|\Points|}\sum_{j \in \Points} d^p(j,\phi(j))}{\sum_{h \in \Colors }\frac{\totviolcol_h(S,\phi)}{|\Points^h|}}
\end{align*}
where $(S,\phi)$ is vanilla $k$-means' solution with $p=2$, $|S|=k$. To run algorithm for intended $k$ range $[K]$, we normalize once with $F_{\text{Rawl}} = \text{mean}(F_{\text{Rawl}}^k)$, where mean is taken over the range $[K]$. Denote each data point $j$'s feature vector from as $x_j$. Data is then normalized as
\begin{align*}
    \mathbf{x}_j = \mathbf{x}_j / \sqrt{F_{\text{Rawl}}}.
\end{align*}

We calculate $F_{\text{Util}}^k$ as
\begin{align*}
    F_{\text{Util}}^k =  \frac{ \sum_{h \in \Colors} \frac{1}{|\Points^h|}\sum_{j \in \Points^h} d^p(j,\phi(j))}{\sum_{h \in \Colors }\frac{\totviolcol_h(S,\phi)}{|\Points^h|}}
\end{align*}
where $(S,\phi)$ is $k$-means' solution with $p=2$, $|S|=k$. $F_{\text{Util}}$ is similarly calculated as a mean when $k$ is selected to be a range. Data for Utilitarian comparison are normalized as
\begin{align*}
    \mathbf{x}_j = \mathbf{x}_j / \sqrt{F_{\text{Util}}}.
\end{align*}

\subsection{Experimental Results in Setting with Multiple Groups}
In this section, we present comparison of Utilitarian algorithm $\algu$ with three other algorithms on settings with multiple groups. Below are plots for \bank{} dataset, where marital status is chosen as group label forming three groups. See Table \ref{tab:multi_group_dataset_description} for more details about dataset size. For \textsc{FCBC-Util} algorithm, we set its upper bound on clustering distance cost as $1.5$ times $k$-means vanilla distance cost. Parameter $\delta$ is set as $0.01$. 

We do not include Rawlsian comparisons in the multiple-group setting, as the existing implementation of the $k$-means socially fair algorithm—which we use as the first step in $\algr$—is designed for two groups and only provides theoretical guarantees in that setting; its extension to multiple groups relies on a heuristic without formal guarantees.

\begin{figure}[H]
    \centering

    \begin{subfigure}[b]{0.3\textwidth}
        \includegraphics[width=\textwidth]{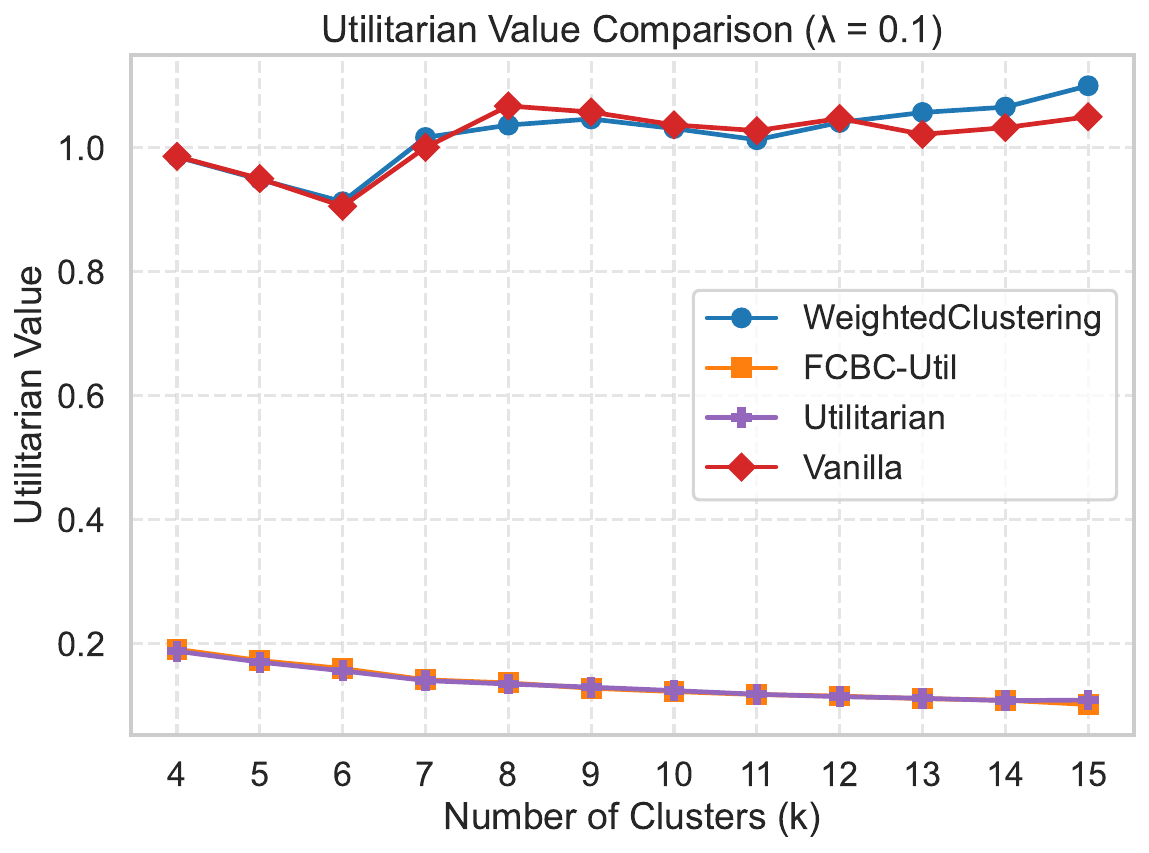}
        \caption{Utilitarian $\lambda = 0.1 $}
    \end{subfigure}
    \hfill
    \begin{subfigure}[b]{0.3\textwidth}
        \includegraphics[width=\textwidth]{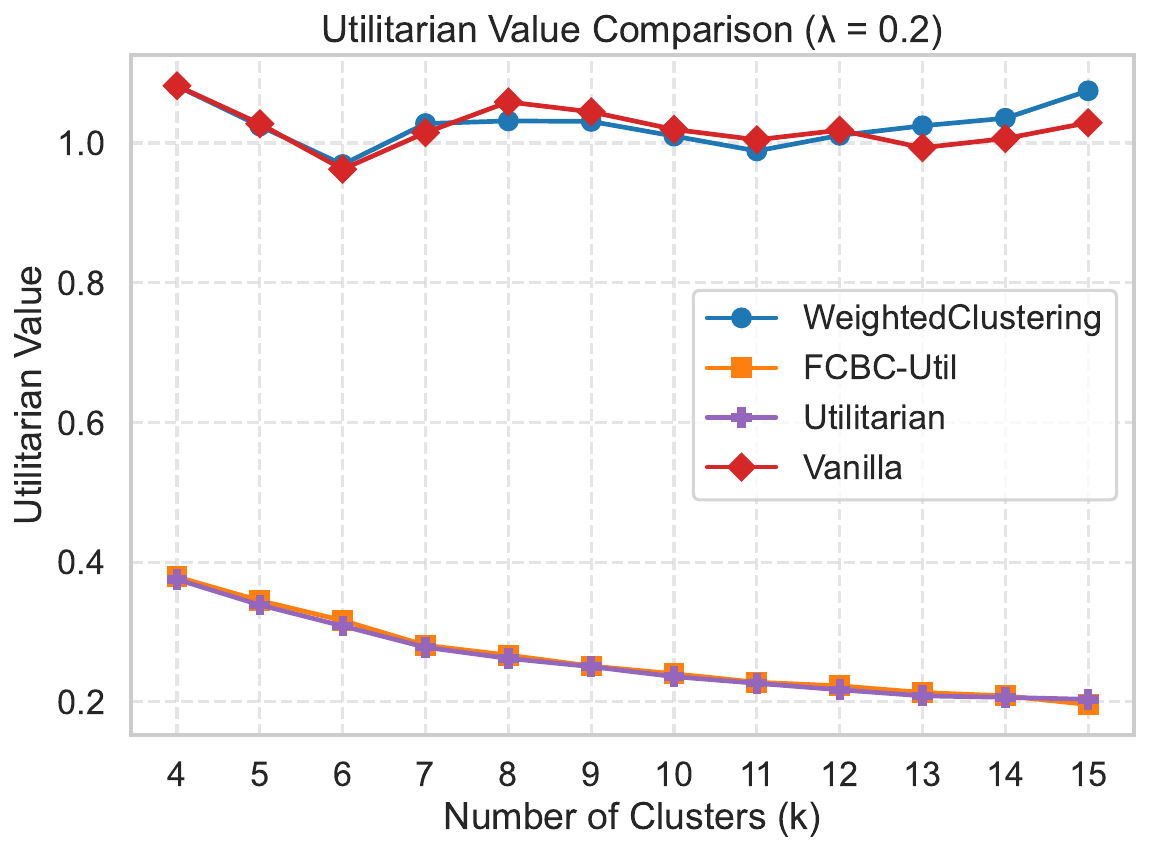}
        \caption{Utilitarian $\lambda = 0.2 $}
    \end{subfigure}
 \hfill
    
    \begin{subfigure}[b]{0.3\textwidth}
        \includegraphics[width=\textwidth]{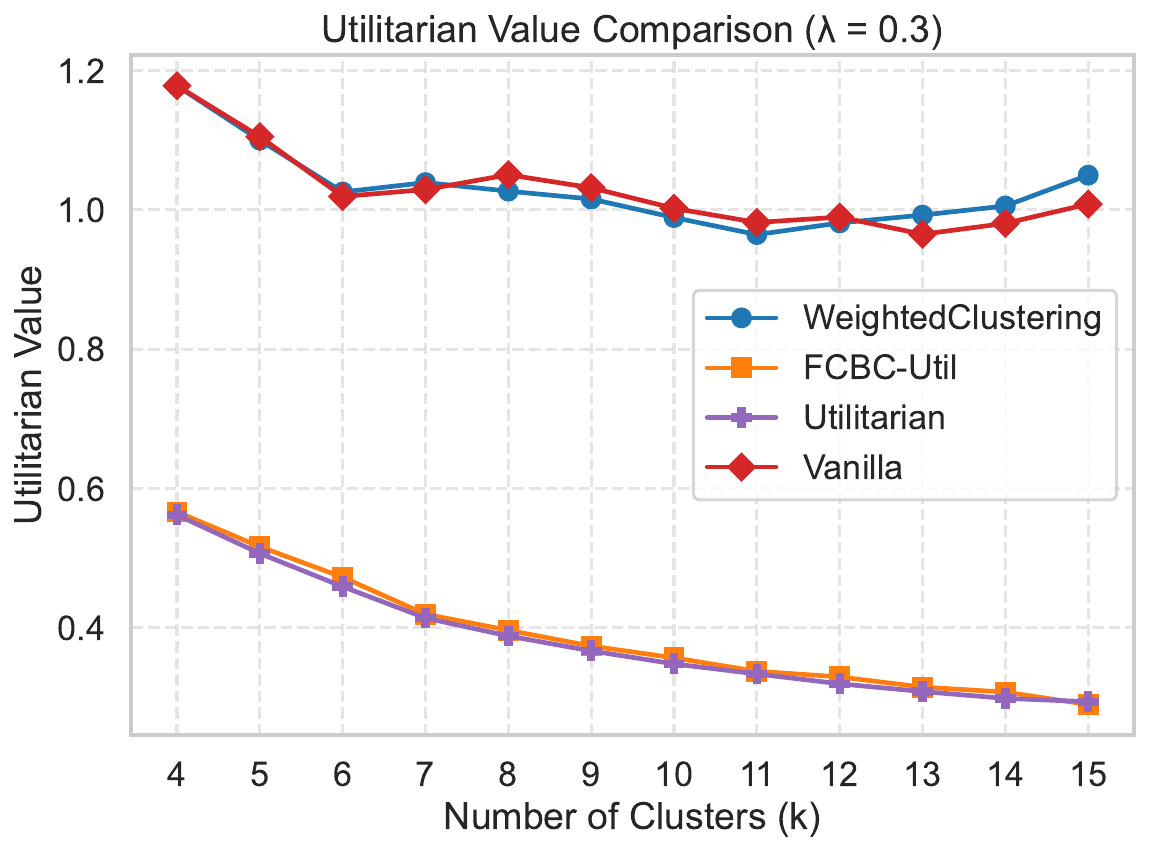}
        \caption{Utilitarian $\lambda = 0.3 $}
    \end{subfigure}

    \begin{subfigure}[b]{0.3\textwidth}
        \includegraphics[width=\textwidth]{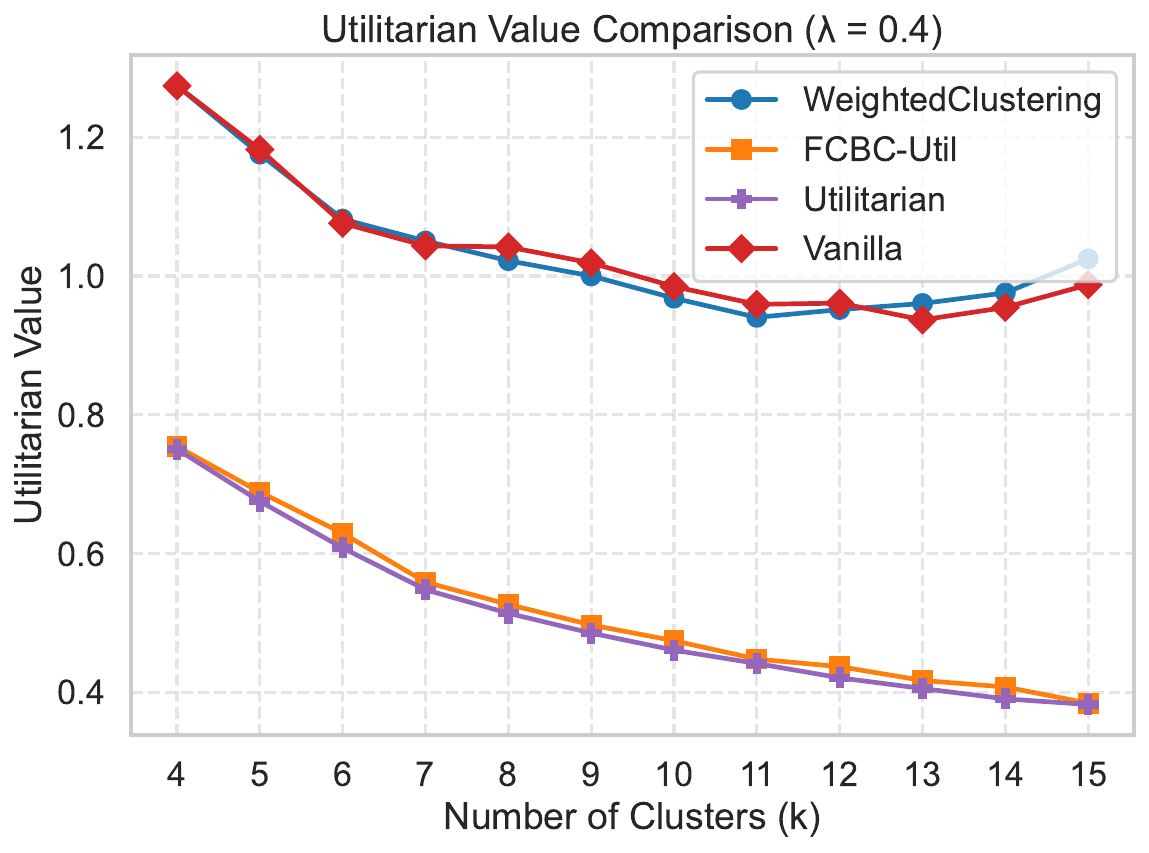}
        \caption{Utilitarian $\lambda = 0.4 $}
    \end{subfigure}
    \hfill
    \begin{subfigure}[b]{0.3\textwidth}
        \includegraphics[width=\textwidth]{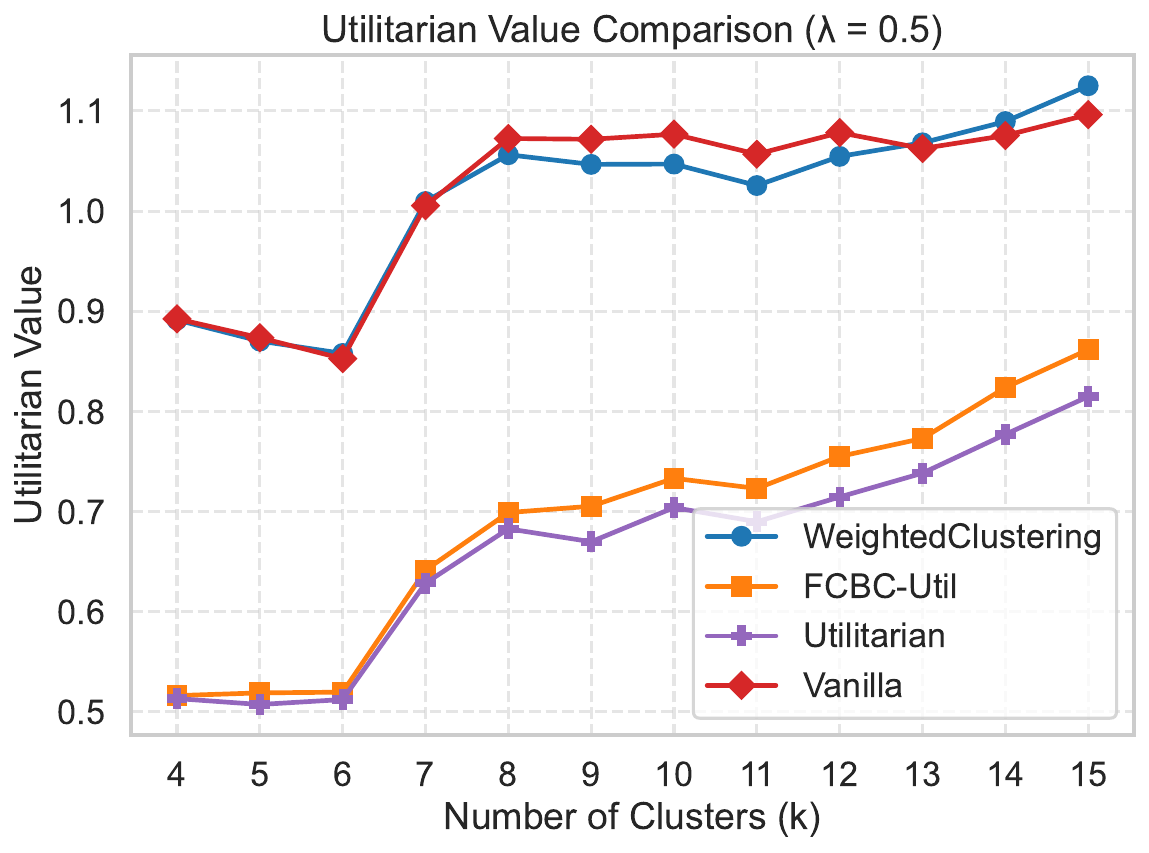}
        \caption{Utilitarian $\lambda = 0.5 $}
    \end{subfigure}
    \hfill
    \begin{subfigure}[b]{0.3\textwidth}
        \includegraphics[width=\textwidth]{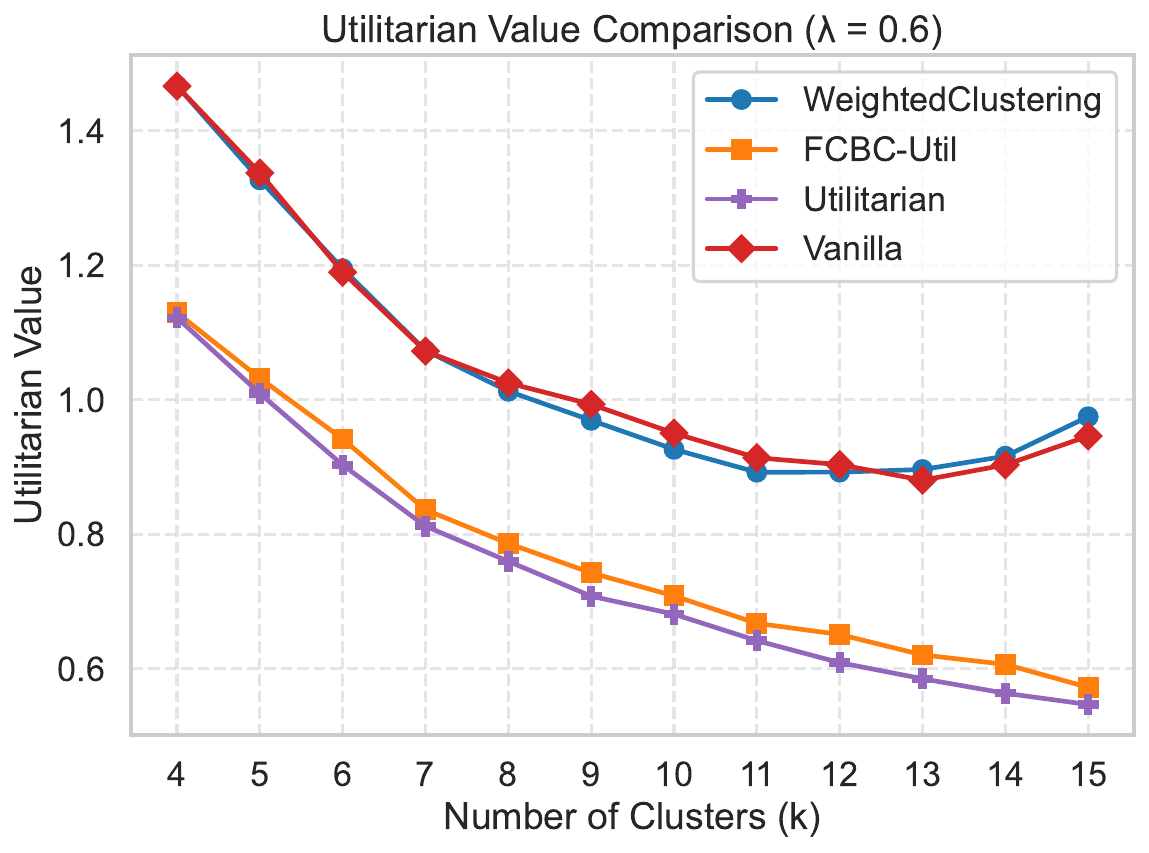}
        \caption{Utilitarian $\lambda = 0.6 $}
    \end{subfigure}
   
    \begin{subfigure}[b]{0.3\textwidth}
        \includegraphics[width=\textwidth]{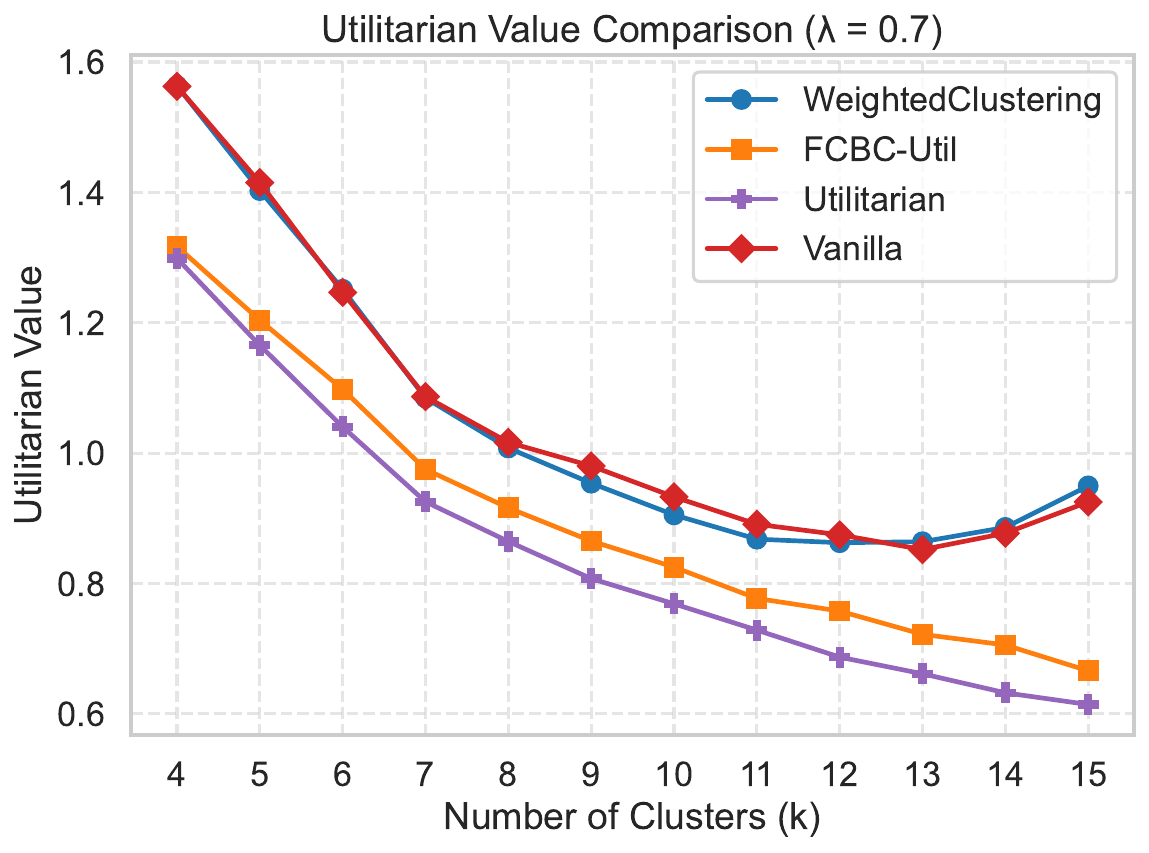}
        \caption{Utilitarian $\lambda = 0.7 $}
    \end{subfigure}
    \hfill
    \begin{subfigure}[b]{0.3\textwidth}
        \includegraphics[width=\textwidth]{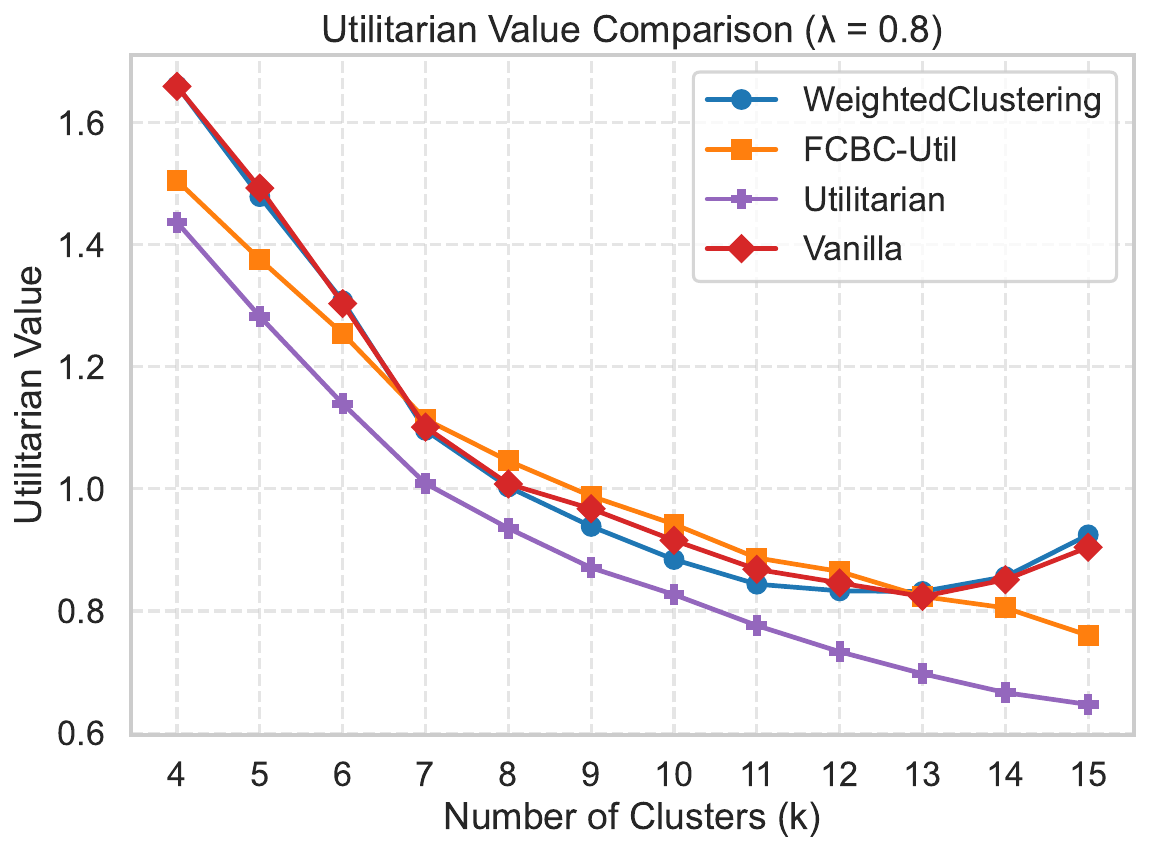}
        \caption{Utilitarian $\lambda = 0.8 $}
    \end{subfigure}
\hfill
    
    \begin{subfigure}[b]{0.3\textwidth}
        \includegraphics[width=\textwidth]{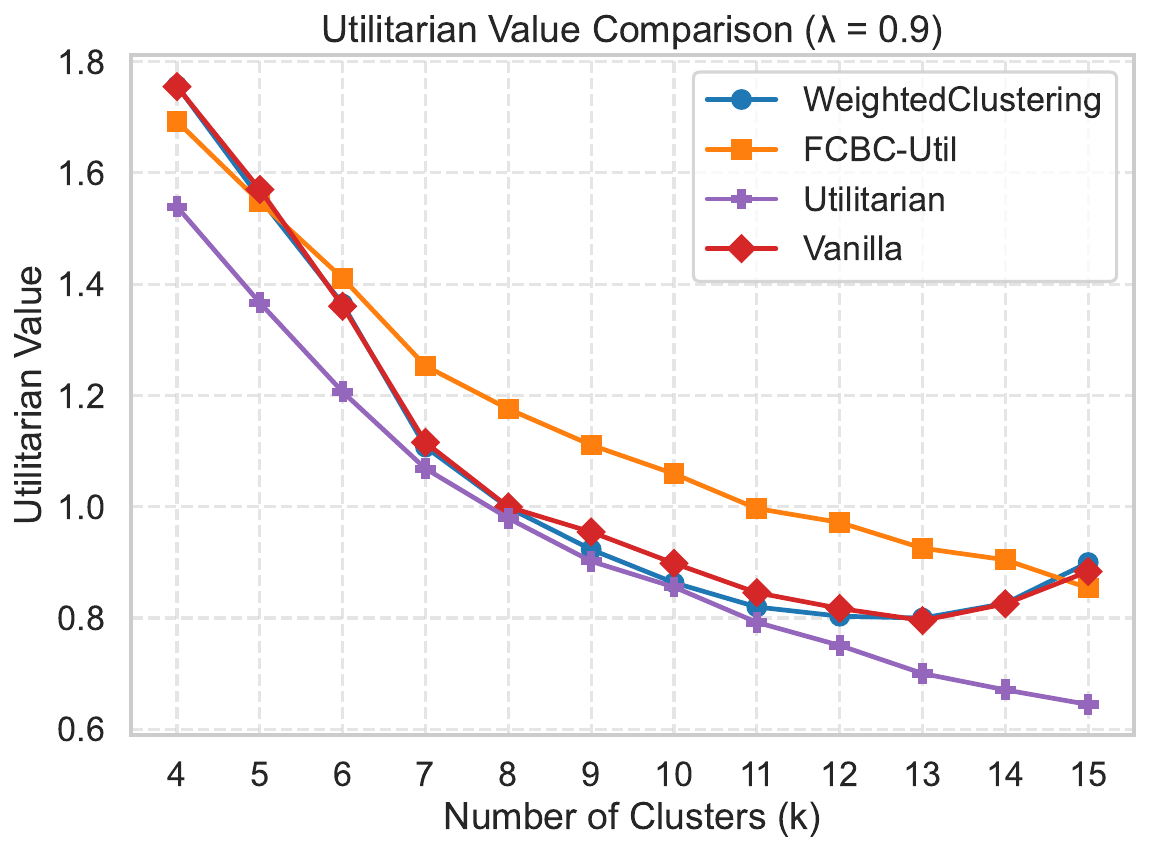}
        \caption{Utilitarian $\lambda = 0.9 $}
    \end{subfigure}

    \caption{Bank Data Utilitarian Value Comparison $|\Colors|=3$}
    \label{fig:eleven_subfigs}
\end{figure}
\FloatBarrier

\subsection{More Experimental Results in Setting with Two Groups}

In the Experiment section in the main body Section \ref{sec:experiments}, we presented experimental results and plots for $\lambda = 0.5$ for three datasets Adult, Credit Card and Census. In this section, we present plots for all other $\lambda$ values in $\{0.1, 0.2, \dots, 0.9\}$. We omit $\lambda = 1.0$ since it is equivalent with running socially fair or weighted clustering. We observe that for all $\lambda$ values, our proposed algorithm $\algr$ and $\algu$ find clustering solutions that dominate the other algorithms' solutions on the two objectives.

Below we present a set of plots for Rawlsian objective for $\lambda \in \{0.1, 0.2, \dots, 0.9\}$.

\begin{figure}[H]
    \centering

    \begin{minipage}[b]{0.32\textwidth}
        \centering
        \textbf{Adult}
    \end{minipage}
    \hfill
    \begin{minipage}[b]{0.32\textwidth}
        \centering
        \textbf{Credit Card }
    \end{minipage}
    \hfill
    \begin{minipage}[b]{0.32\textwidth}
        \centering
        \textbf{Census }
    \end{minipage}

    \vspace{0.3cm}

    \begin{subfigure}[b]{0.32\textwidth}
        \includegraphics[width=\textwidth]{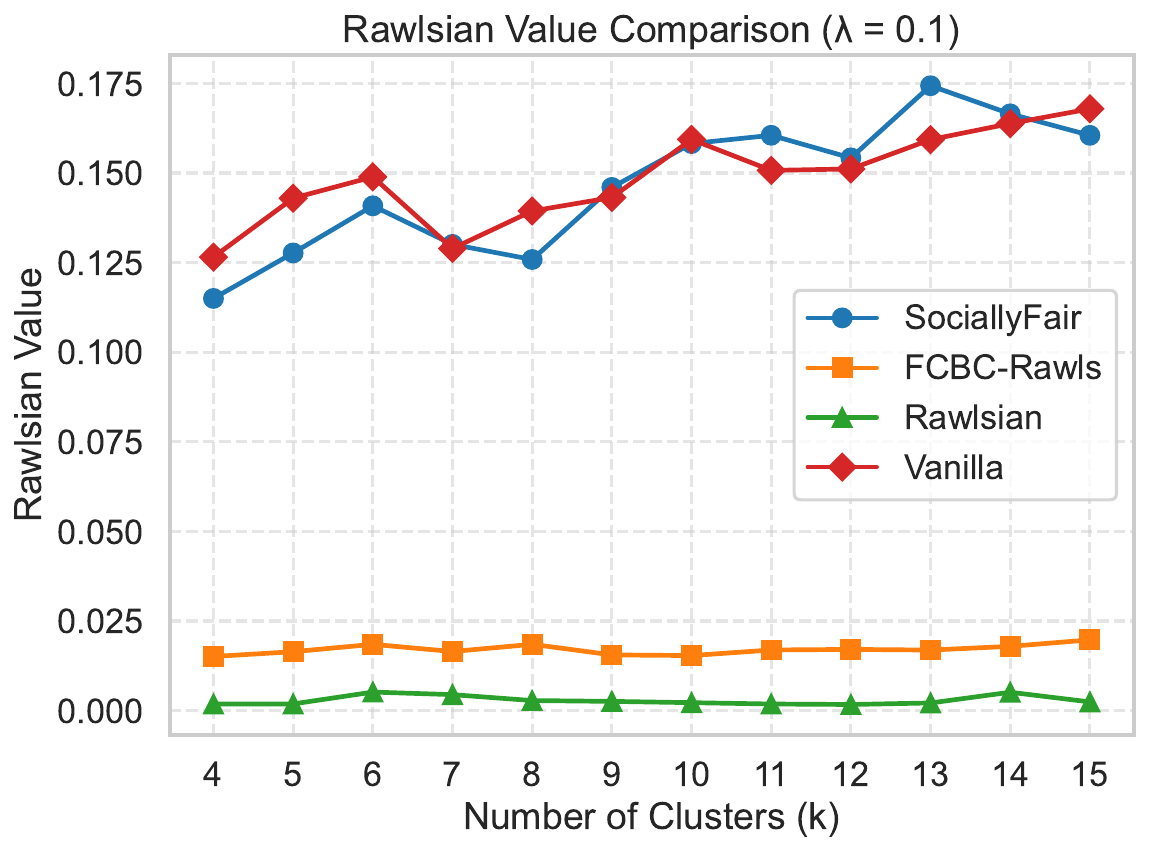}
        
        \label{fig:adult_a}
    \end{subfigure}
    \hfill
    \begin{subfigure}[b]{0.32\textwidth}
        \includegraphics[width=\textwidth]{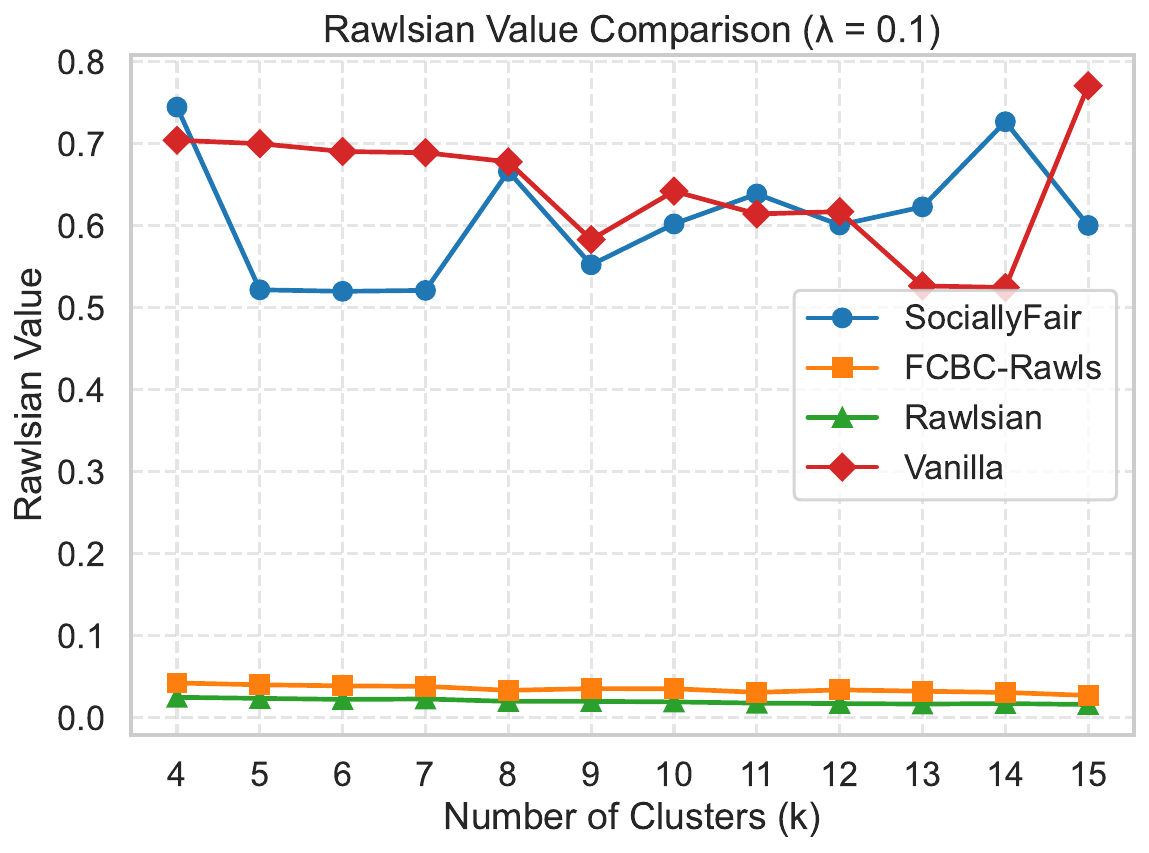}
        
        \label{fig:credit_a}
    \end{subfigure}
    \hfill
    \begin{subfigure}[b]{0.32\textwidth}
        \includegraphics[width=\textwidth]{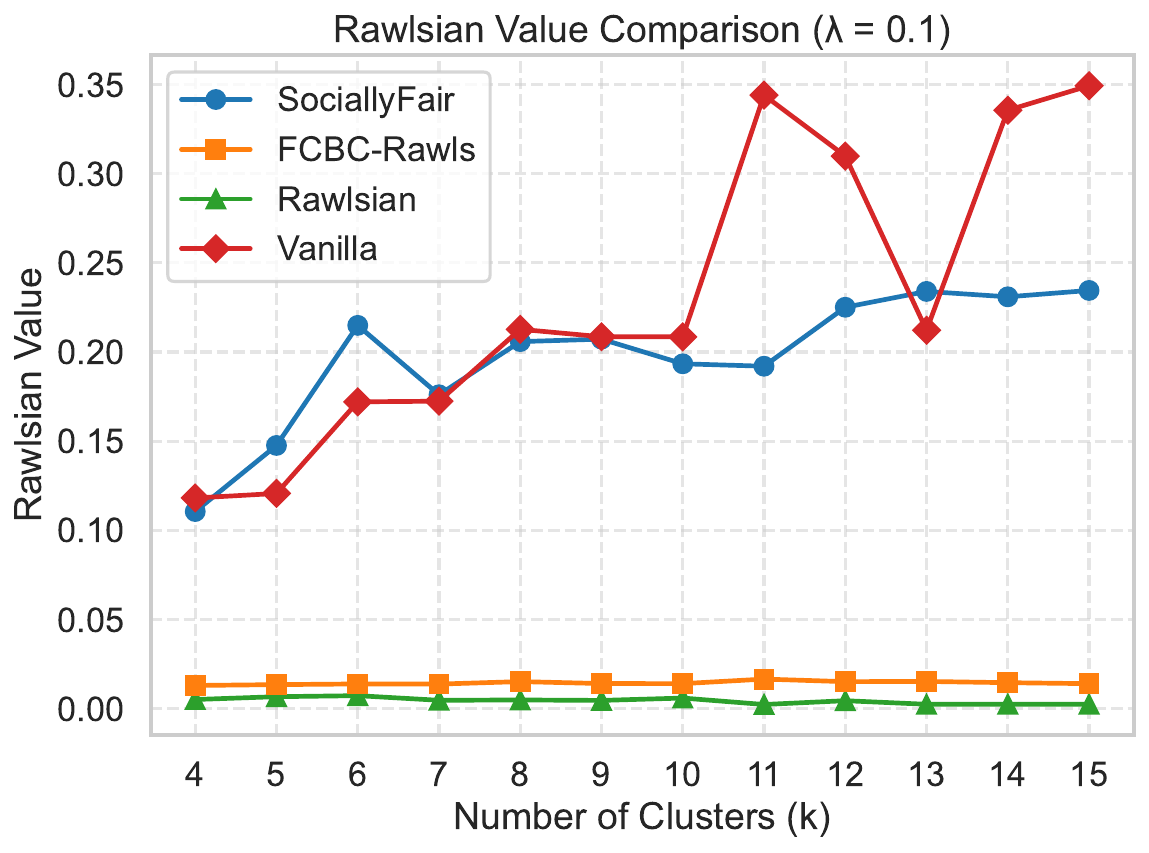}
        
        \label{fig:census_a}
    \end{subfigure}

     \begin{subfigure}[b]{0.32\textwidth}
        \includegraphics[width=\textwidth]{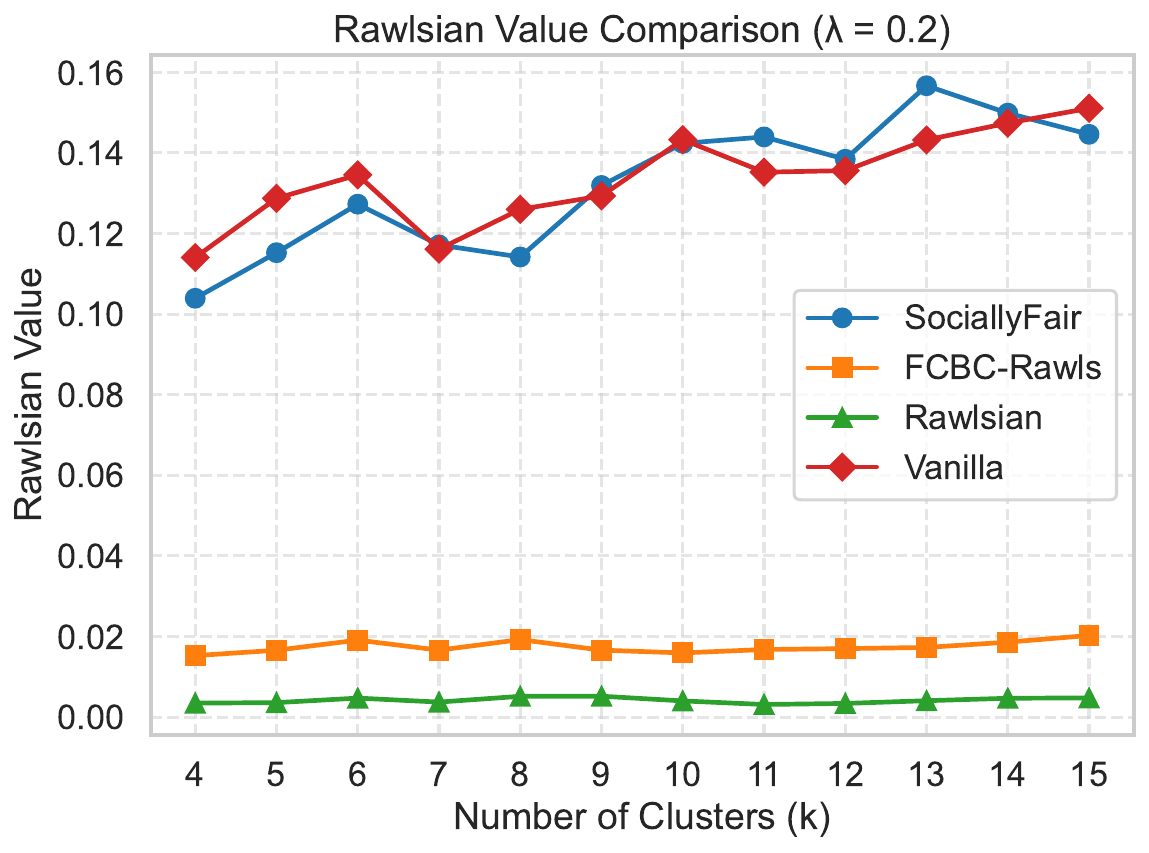}
        
        \label{fig:adult_a}
    \end{subfigure}
    \hfill
    \begin{subfigure}[b]{0.32\textwidth}
        \includegraphics[width=\textwidth]{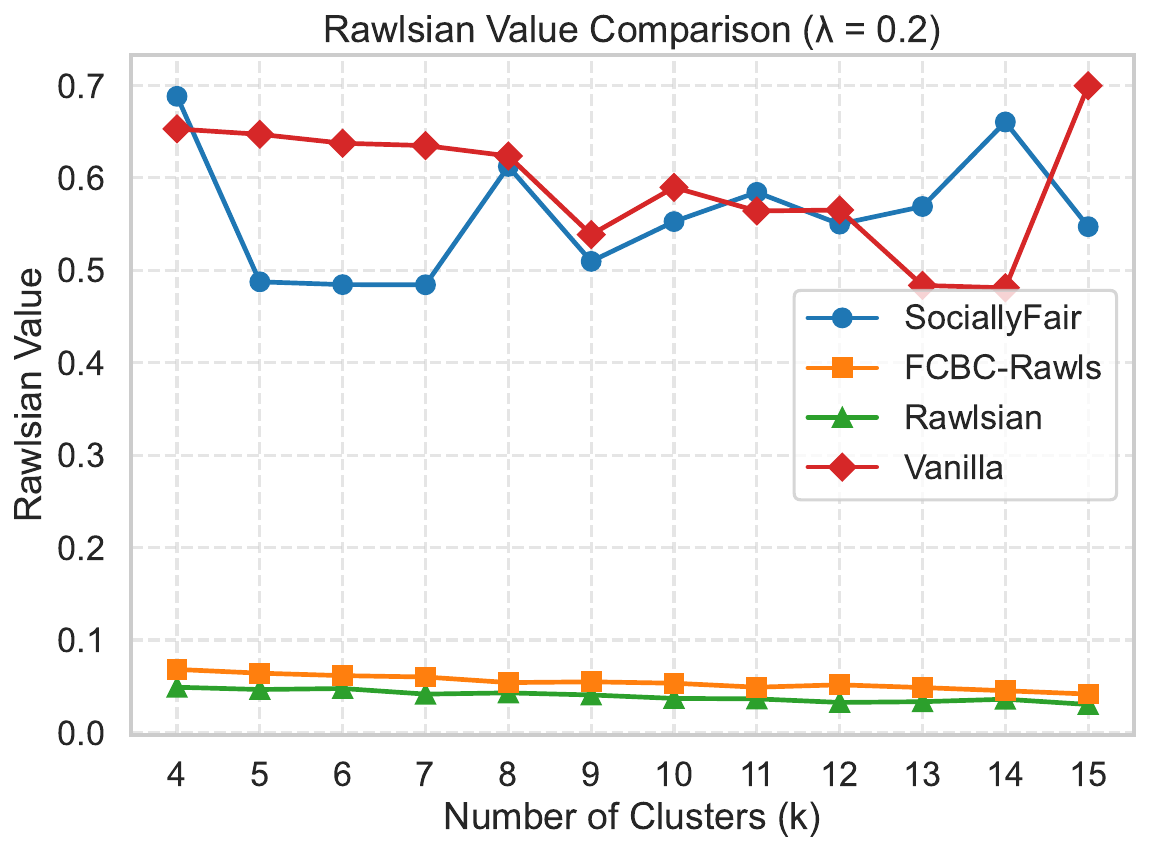}
        
        \label{fig:credit_a}
    \end{subfigure}
    \hfill
    \begin{subfigure}[b]{0.32\textwidth}
        \includegraphics[width=\textwidth]{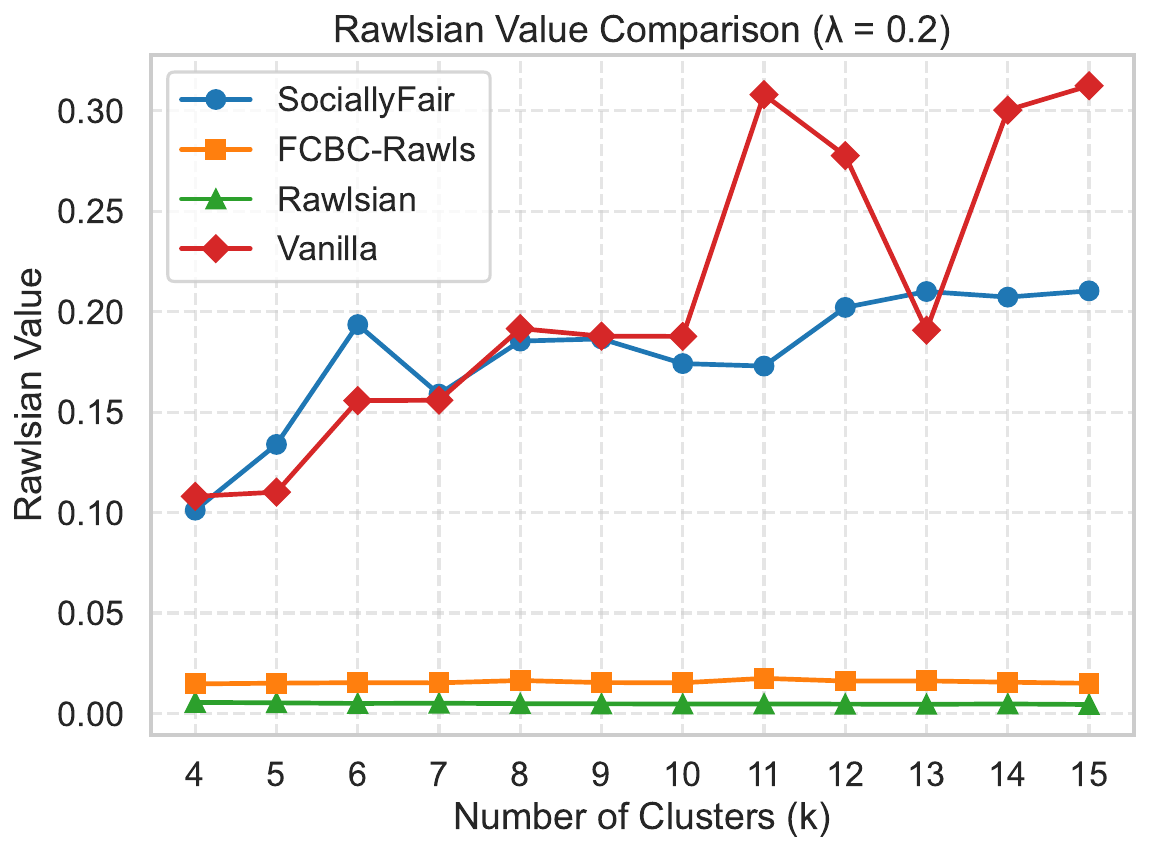}
        
        \label{fig:census_a}
    \end{subfigure}

     \begin{subfigure}[b]{0.32\textwidth}
        \includegraphics[width=\textwidth]{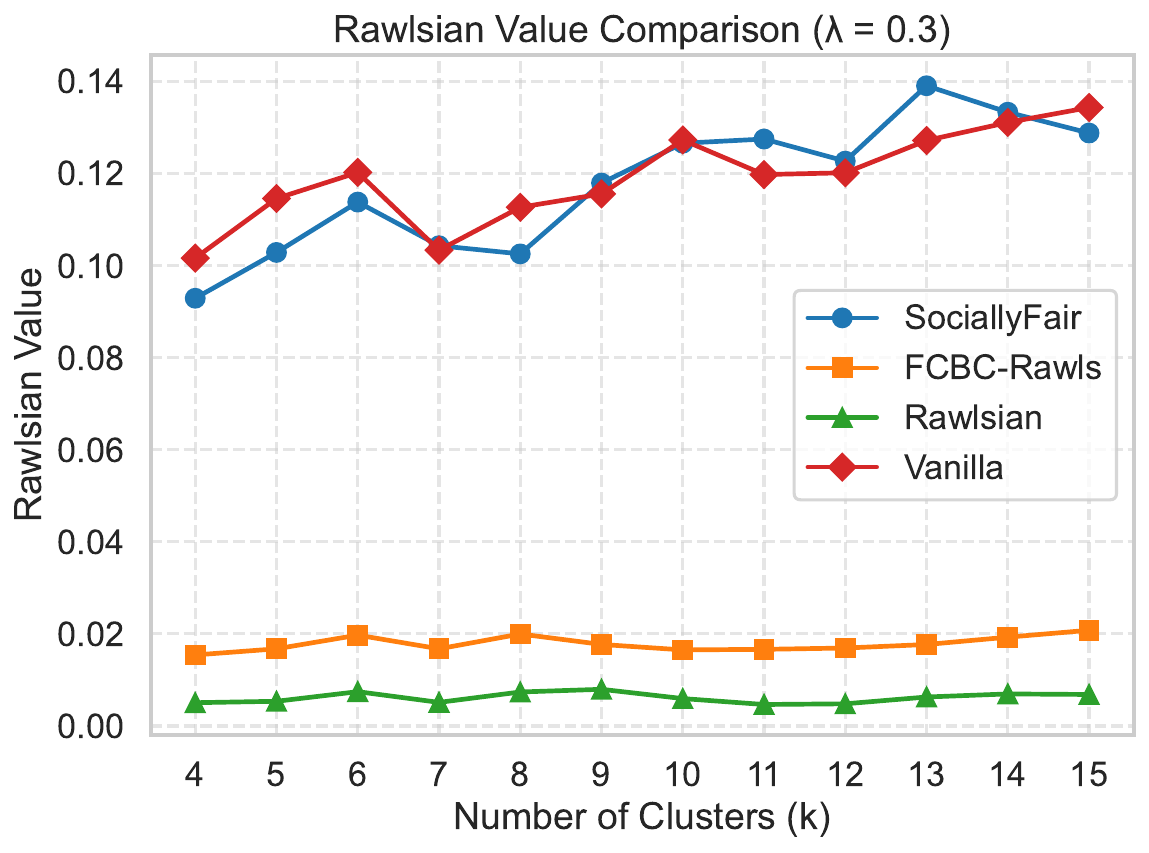}
        
        \label{fig:adult_a}
    \end{subfigure}
    \hfill
    \begin{subfigure}[b]{0.32\textwidth}
        \includegraphics[width=\textwidth]{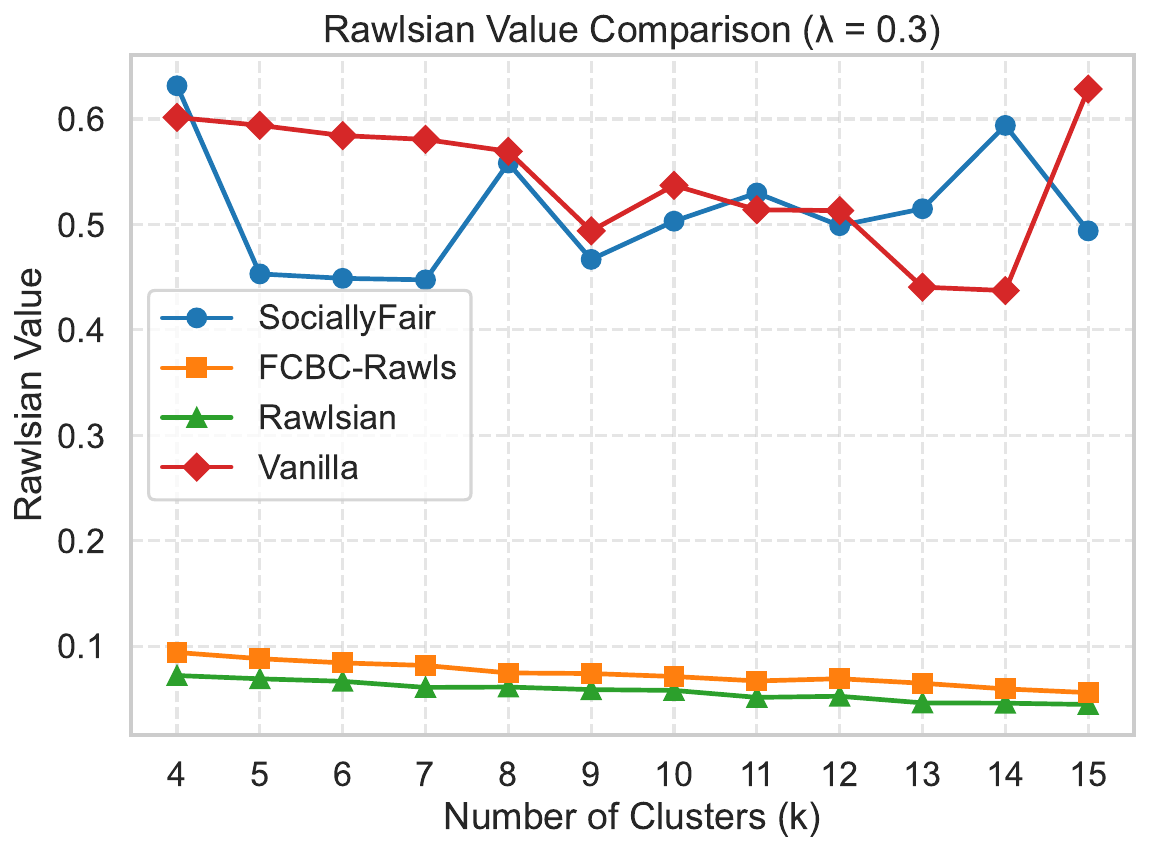}
        
        \label{fig:credit_a}
    \end{subfigure}
    \hfill
    \begin{subfigure}[b]{0.32\textwidth}
        \includegraphics[width=\textwidth]{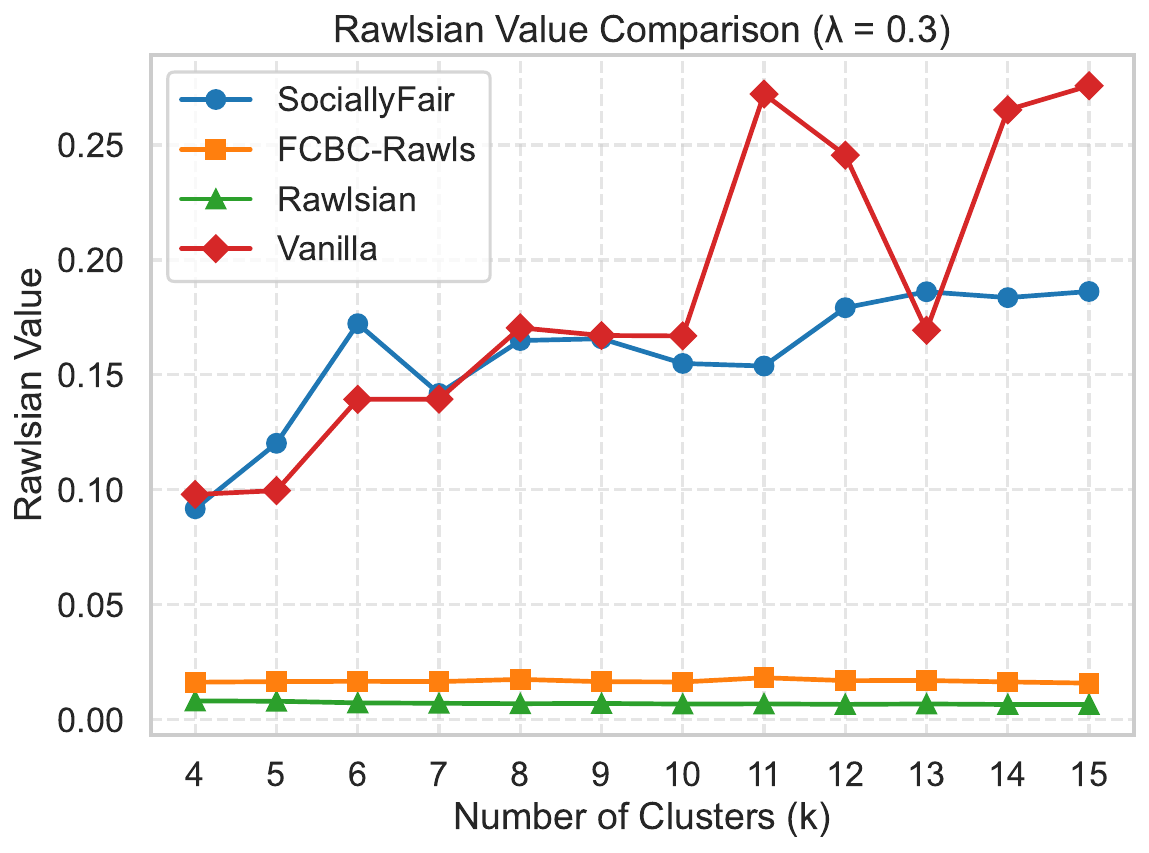}
        
        \label{fig:census_a}
    \end{subfigure}

     \begin{subfigure}[b]{0.32\textwidth}
        \includegraphics[width=\textwidth]{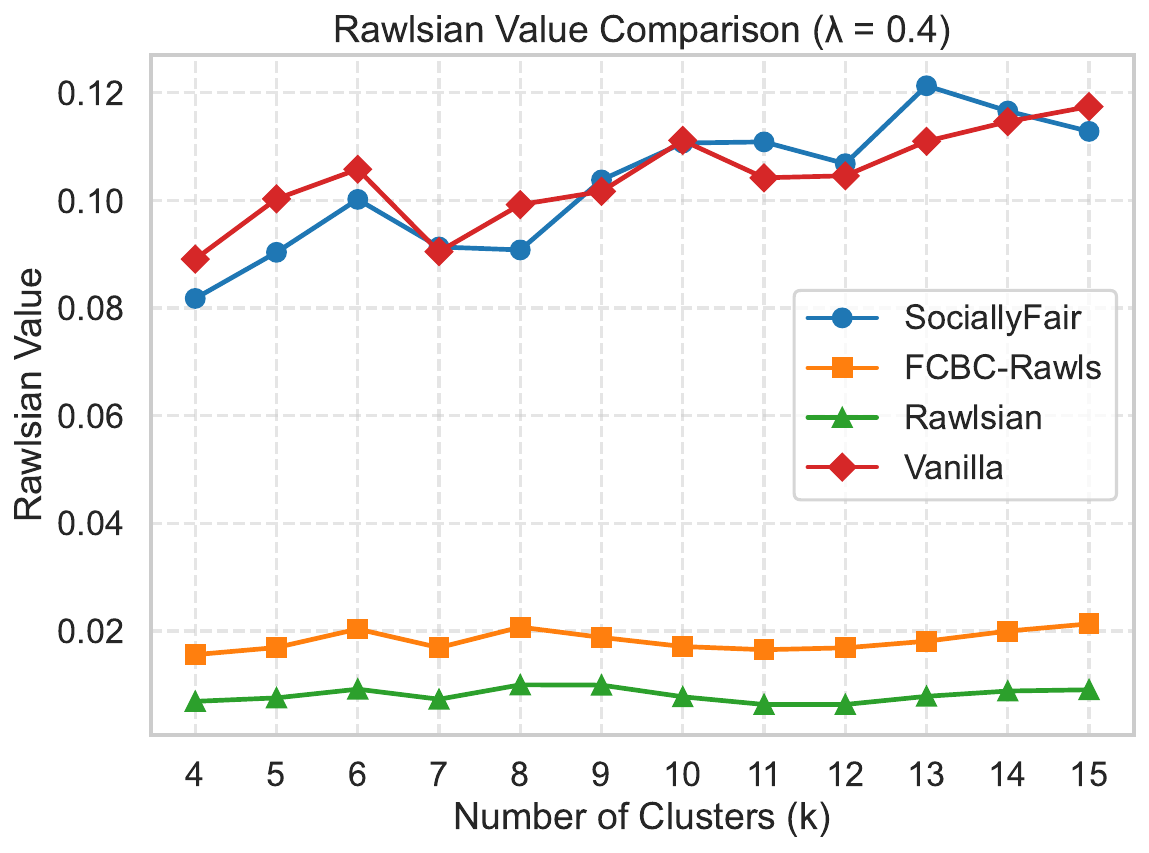}
        
        \label{fig:adult_a}
    \end{subfigure}
    \hfill
    \begin{subfigure}[b]{0.32\textwidth}
        \includegraphics[width=\textwidth]{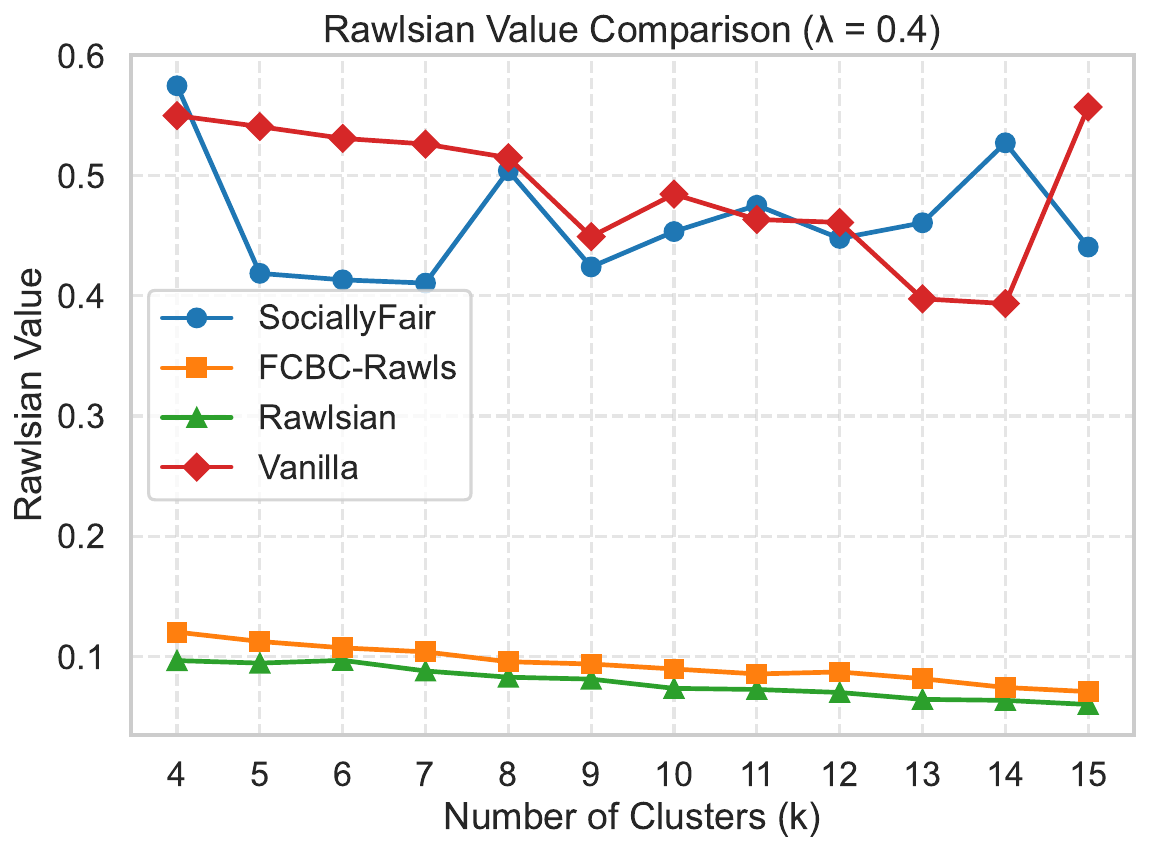}
        
        \label{fig:credit_a}
    \end{subfigure}
    \hfill
    \begin{subfigure}[b]{0.32\textwidth}
        \includegraphics[width=\textwidth]{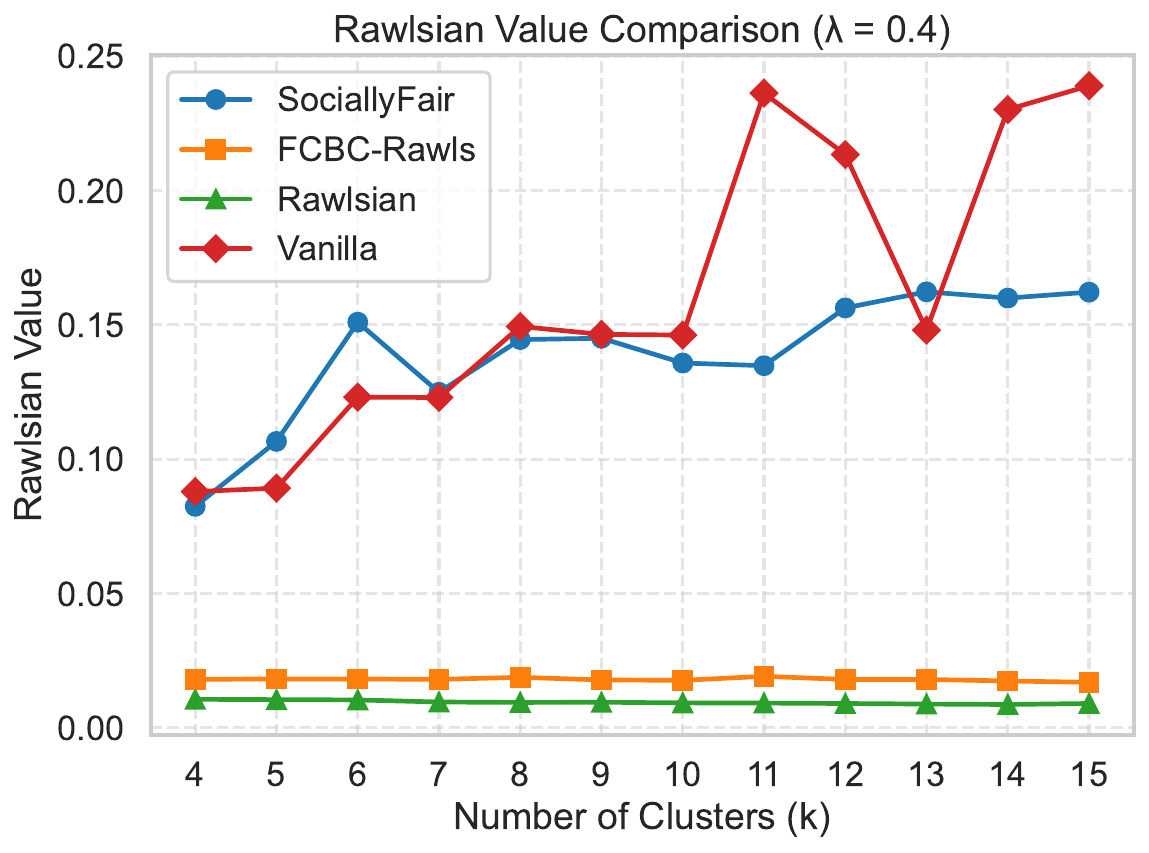}
        
        \label{fig:census_a}
    \end{subfigure}
    \caption{Rawlsian Value comparison for $\lambda=0.1$ to $0.4$}
    
\end{figure}
\FloatBarrier

\begin{figure}[H]
    \centering

     \begin{subfigure}[b]{0.32\textwidth}
        \includegraphics[width=\textwidth]{plots/two_group/ad_welfare_costs_comparison_lambda_0_5.pdf}
        
        \label{fig:adult_a}
    \end{subfigure}
    \hfill
    \begin{subfigure}[b]{0.32\textwidth}
        \includegraphics[width=\textwidth]{plots/two_group/cc_welfare_costs_comparison_lambda_0_5.pdf}
        
        \label{fig:credit_a}
    \end{subfigure}
    \hfill
    \begin{subfigure}[b]{0.32\textwidth}
        \includegraphics[width=\textwidth]{plots/two_group/cen_welfare_costs_comparison_lambda_0_5.pdf}
        
        \label{fig:census_a}
    \end{subfigure}

    \begin{subfigure}[b]{0.32\textwidth}
        \includegraphics[width=\textwidth]{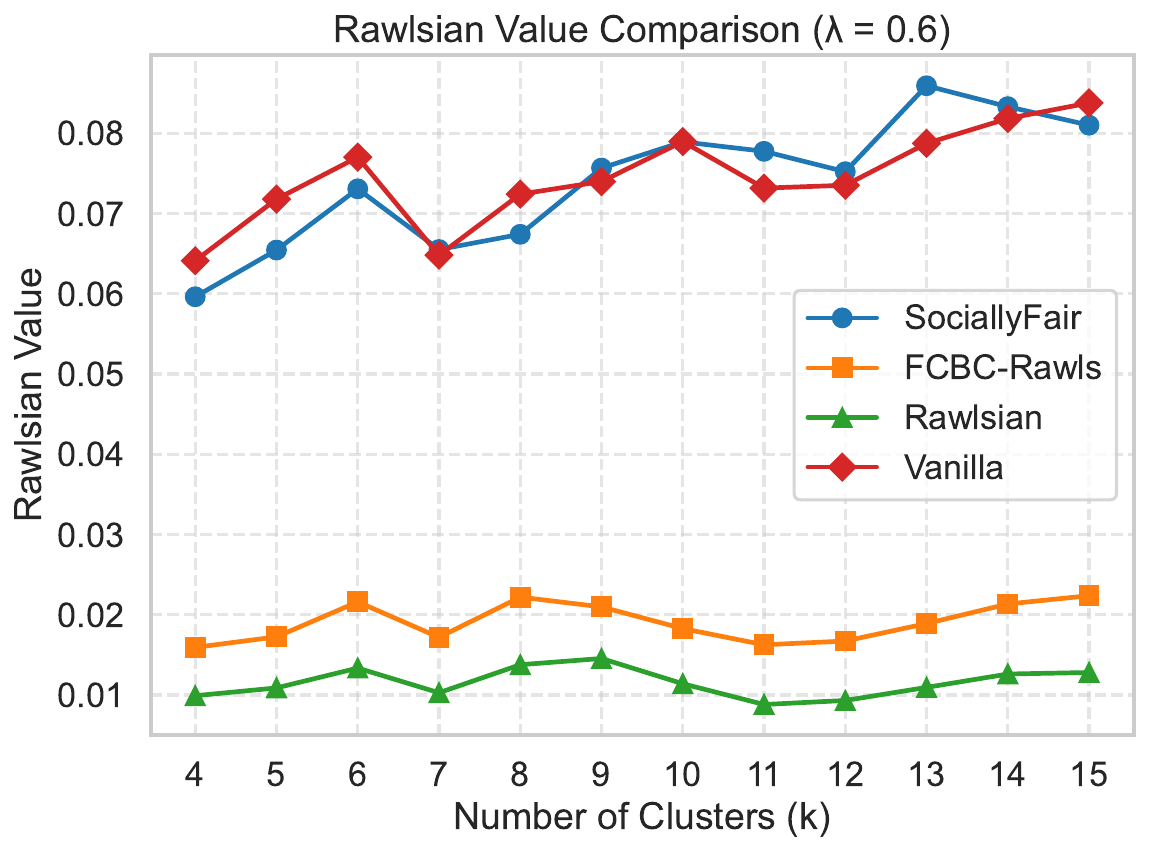}
        
        \label{fig:adult_a}
    \end{subfigure}
    \hfill
    \begin{subfigure}[b]{0.32\textwidth}
        \includegraphics[width=\textwidth]{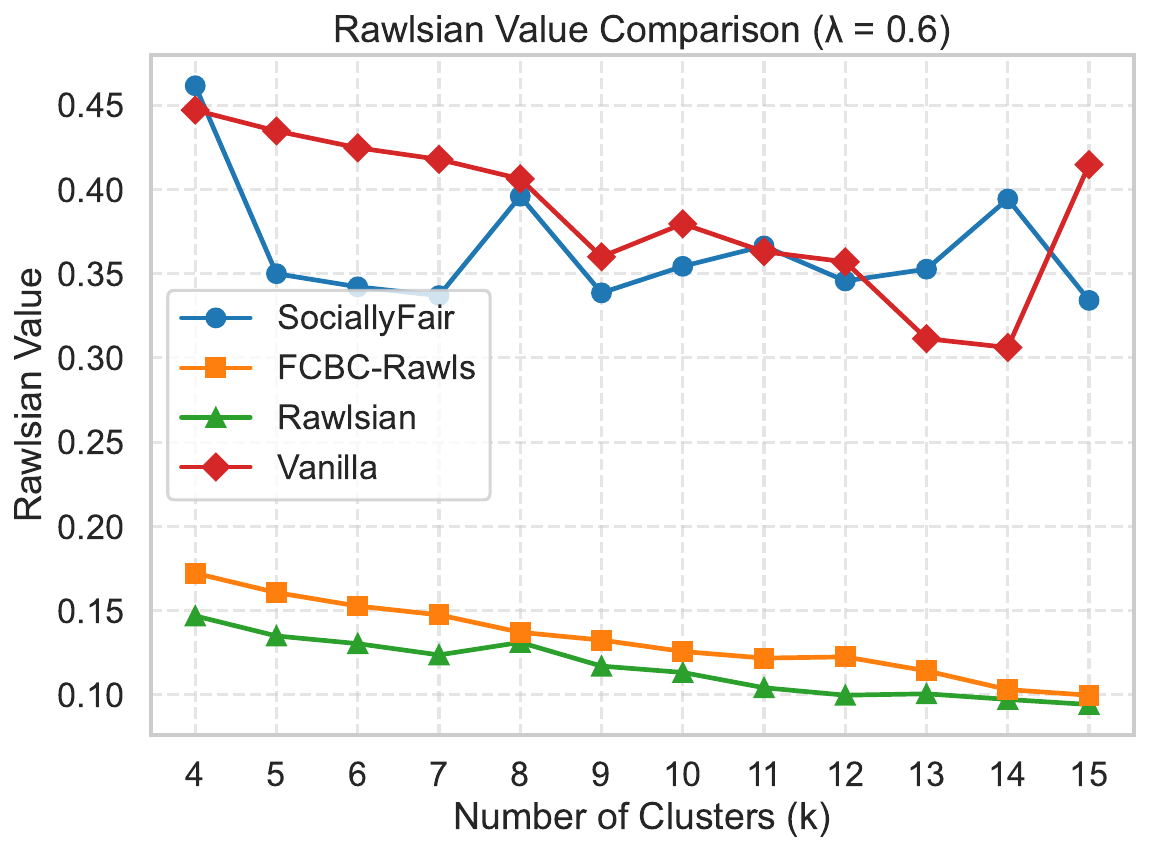}
        
        \label{fig:credit_a}
    \end{subfigure}
    \hfill
    \begin{subfigure}[b]{0.32\textwidth}
        \includegraphics[width=\textwidth]{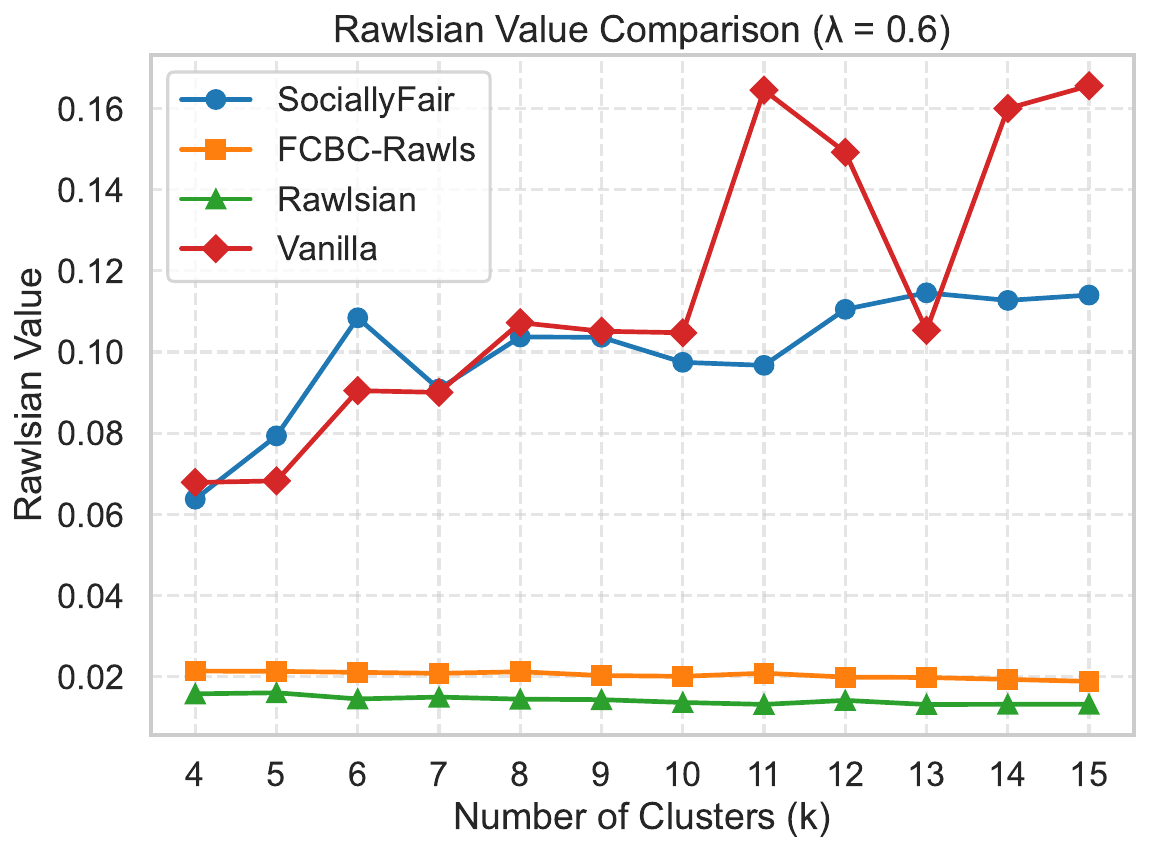}
        
        \label{fig:census_a}
    \end{subfigure}

    \begin{subfigure}[b]{0.32\textwidth}
        \includegraphics[width=\textwidth]{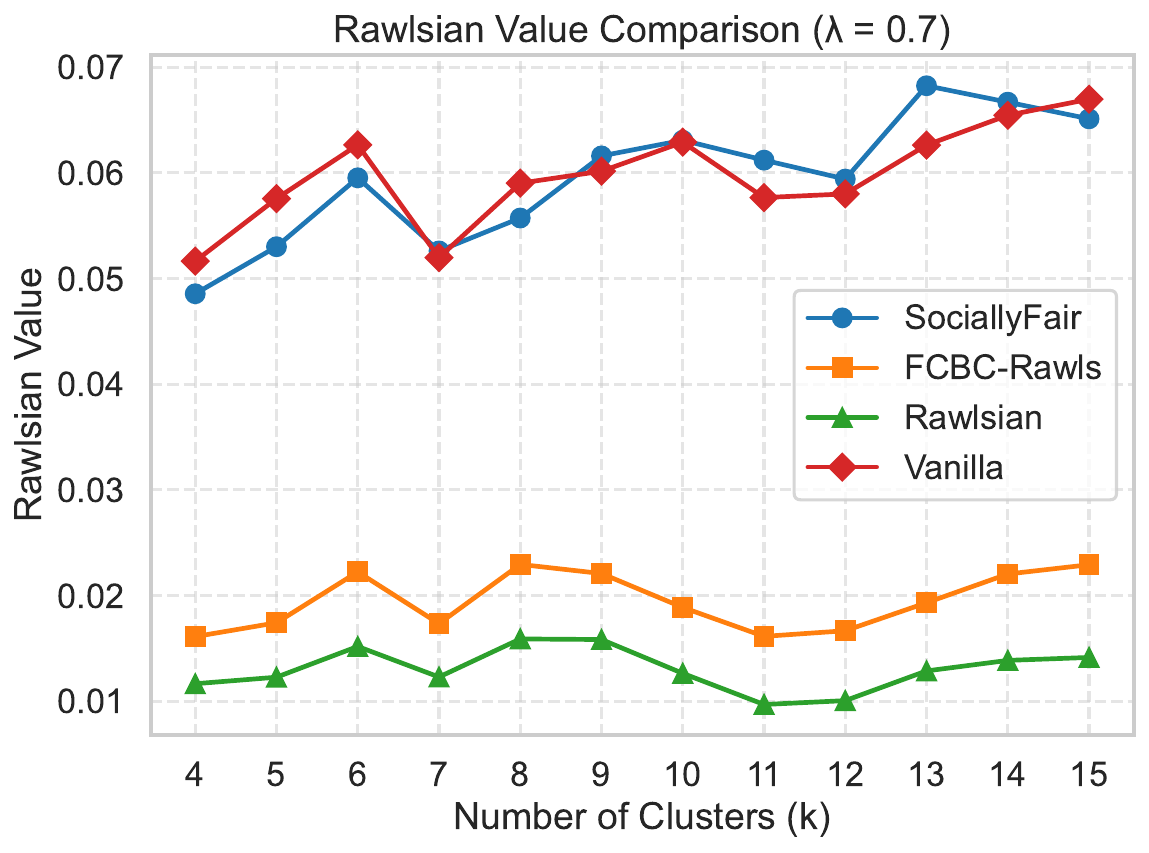}
        
        \label{fig:adult_a}
    \end{subfigure}
    \hfill
    \begin{subfigure}[b]{0.32\textwidth}
        \includegraphics[width=\textwidth]{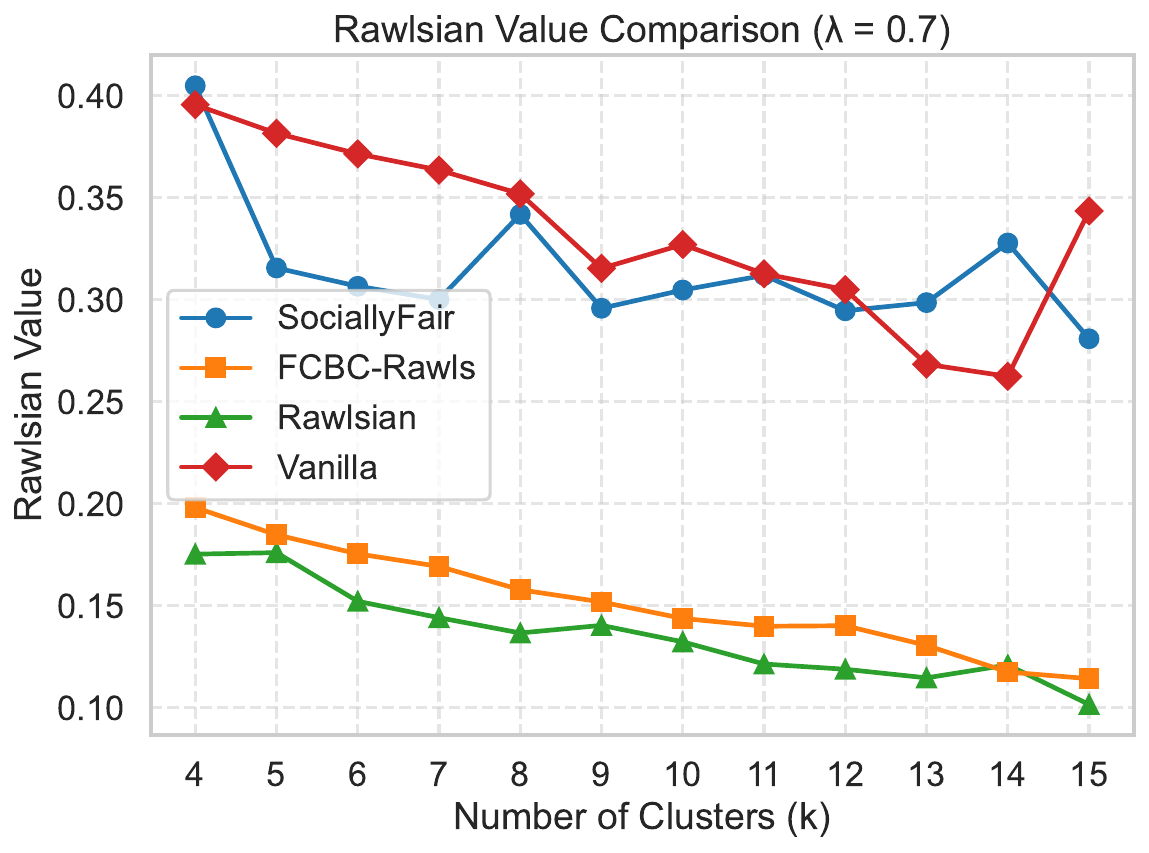}
        
        \label{fig:credit_a}
    \end{subfigure}
    \hfill
    \begin{subfigure}[b]{0.32\textwidth}
        \includegraphics[width=\textwidth]{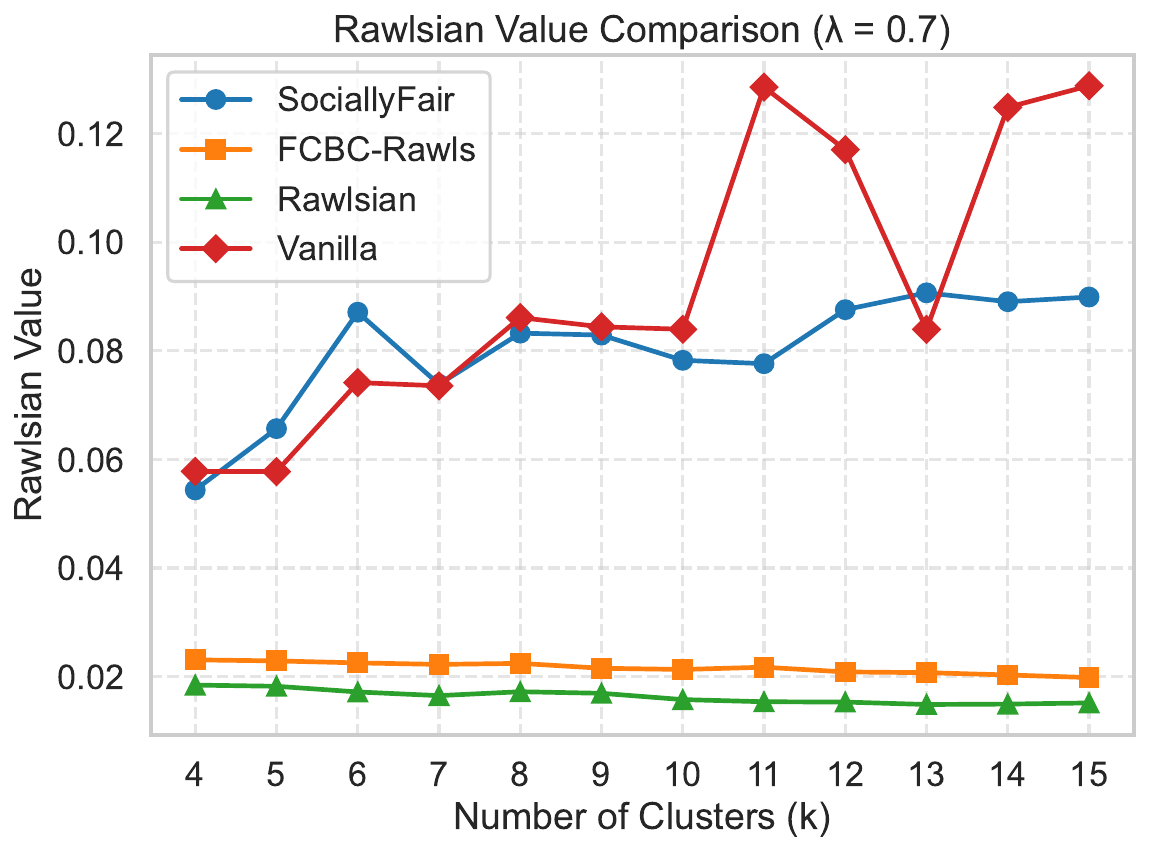}
        
        \label{fig:census_a}
    \end{subfigure}

     \begin{subfigure}[b]{0.32\textwidth}
        \includegraphics[width=\textwidth]{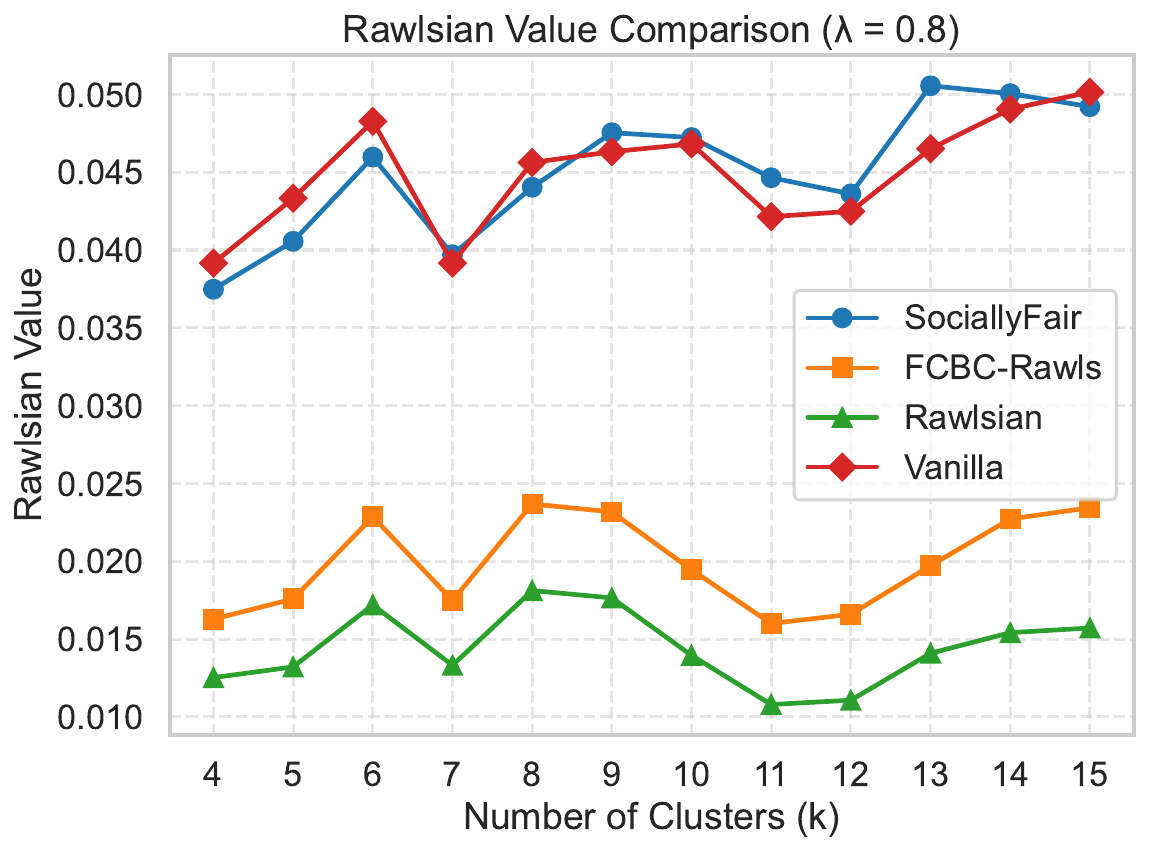}
        
        \label{fig:adult_a}
    \end{subfigure}
    \hfill
    \begin{subfigure}[b]{0.32\textwidth}
        \includegraphics[width=\textwidth]{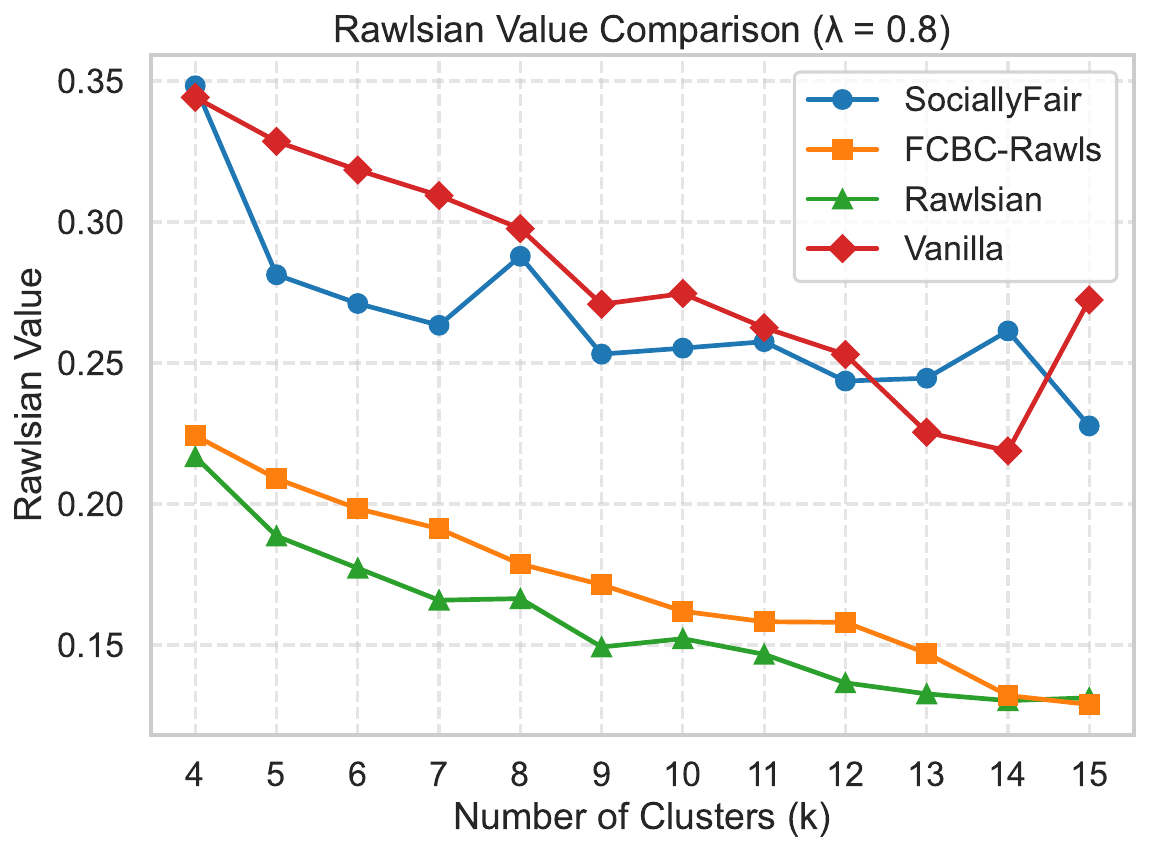}
        
        \label{fig:credit_a}
    \end{subfigure}
    \hfill
    \begin{subfigure}[b]{0.32\textwidth}
        \includegraphics[width=\textwidth]{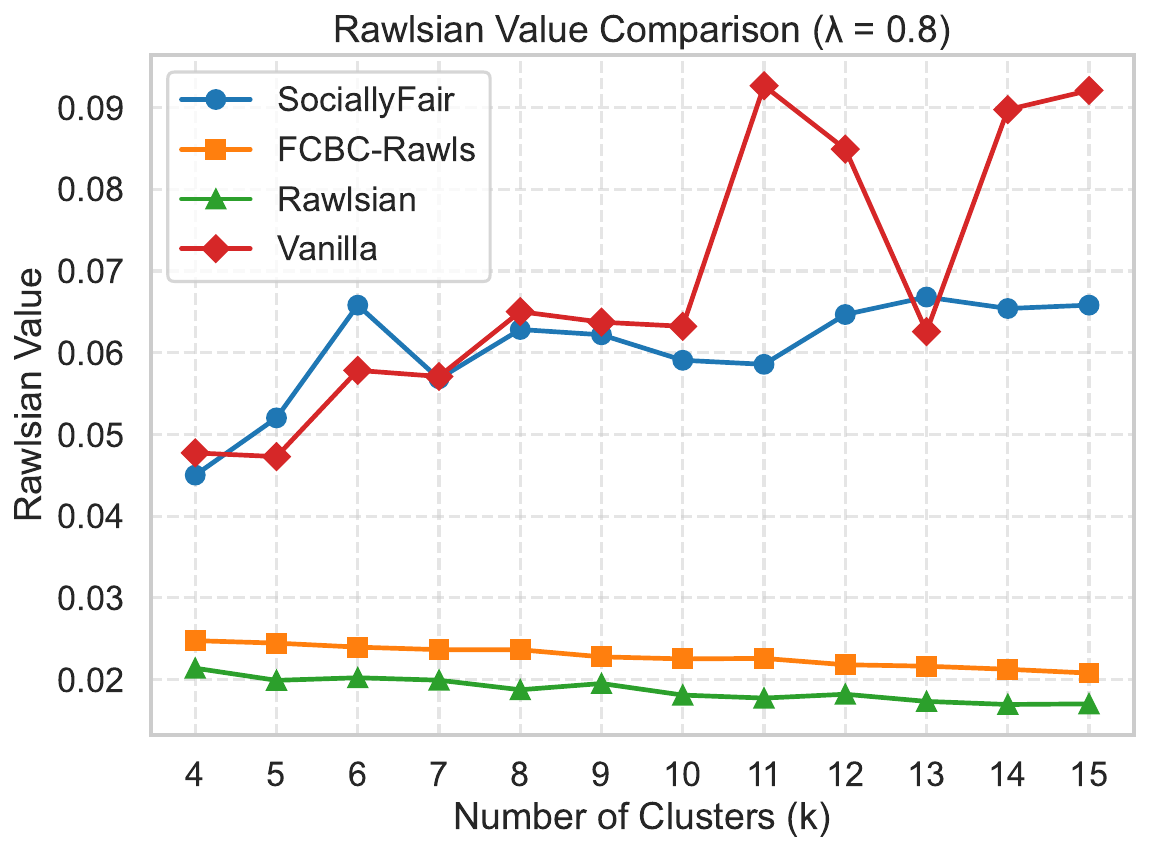}
        
        \label{fig:census_a}
    \end{subfigure}

     \begin{subfigure}[b]{0.32\textwidth}
        \includegraphics[width=\textwidth]{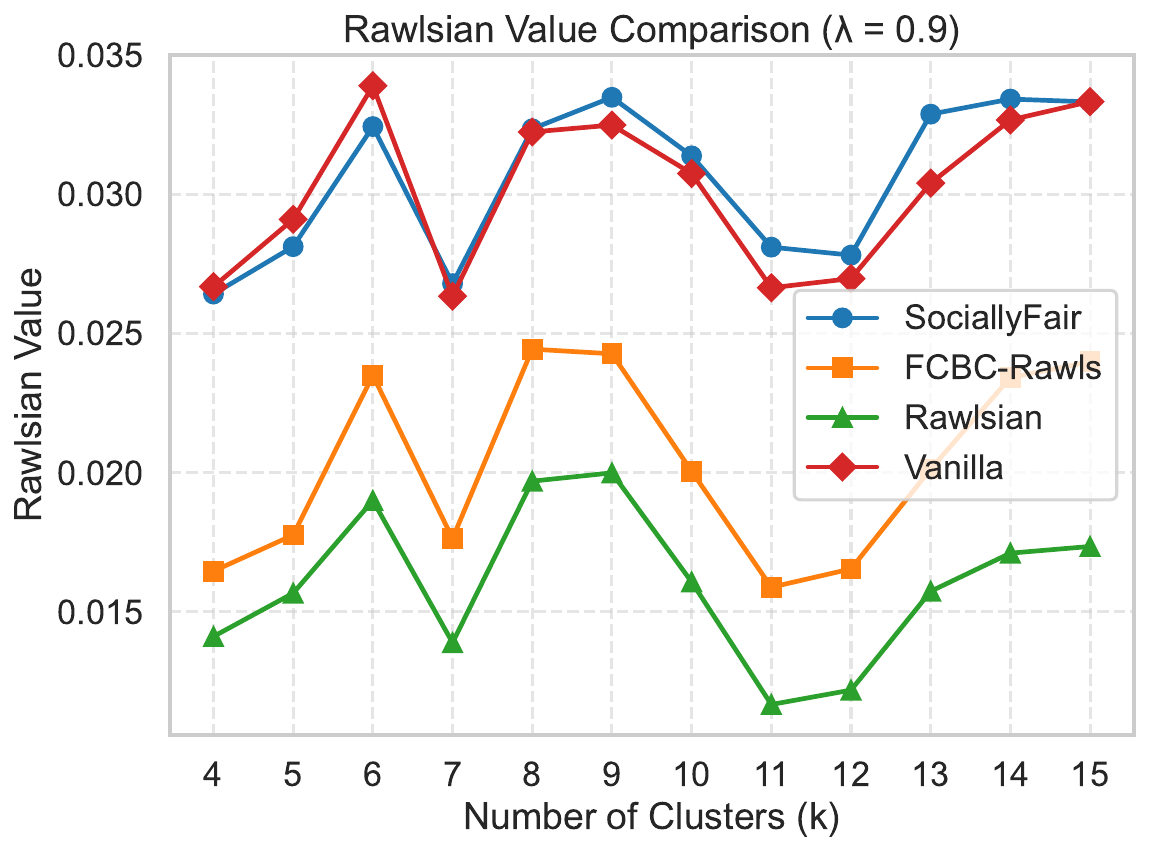}
        
        \label{fig:adult_a}
    \end{subfigure}
    \hfill
    \begin{subfigure}[b]{0.32\textwidth}
        \includegraphics[width=\textwidth]{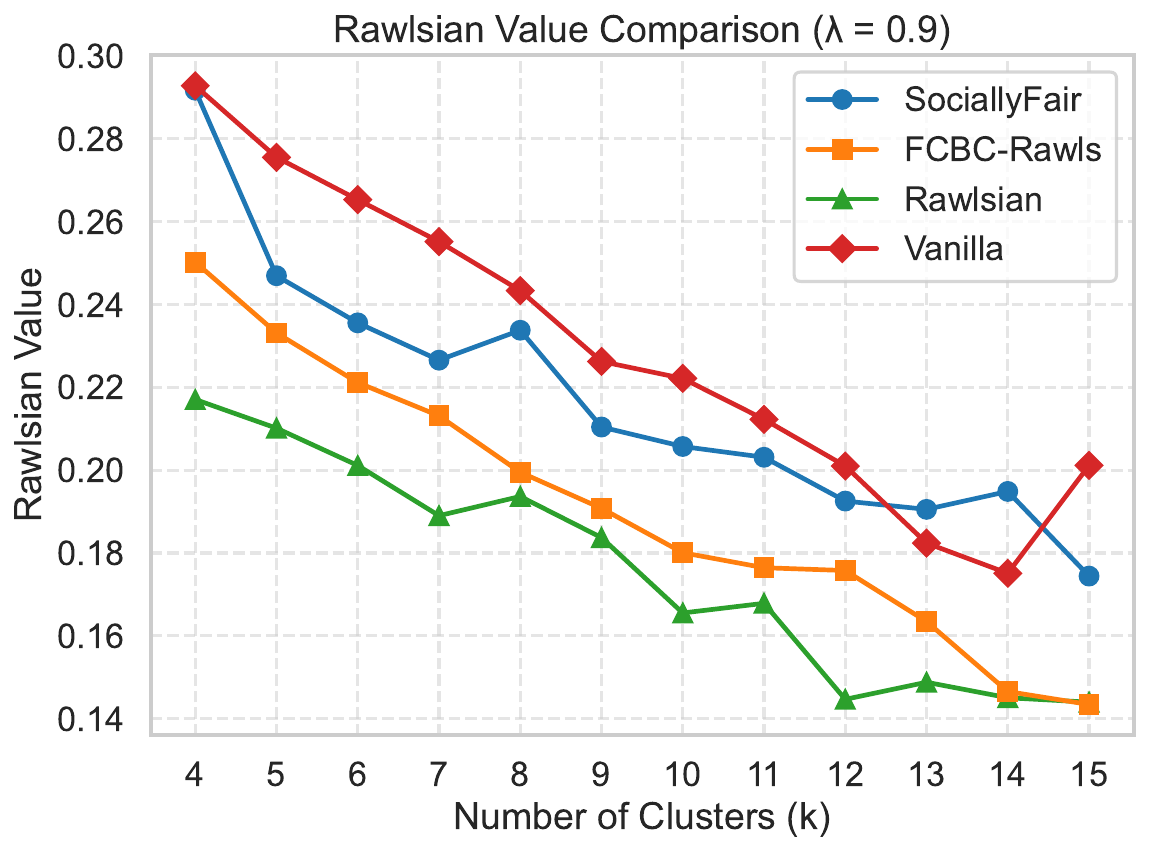}
        
        \label{fig:credit_a}
    \end{subfigure}
    \hfill
    \begin{subfigure}[b]{0.32\textwidth}
        \includegraphics[width=\textwidth]{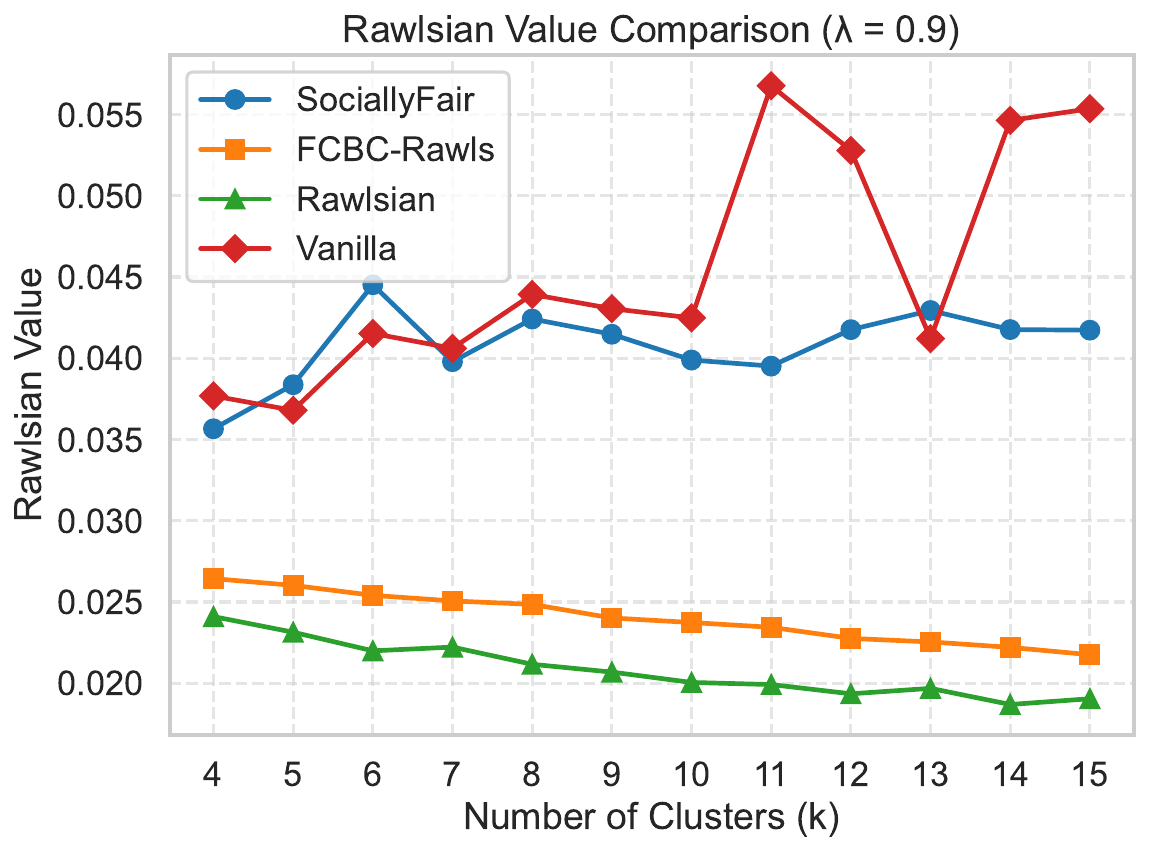}
        
        \label{fig:census_a}
    \end{subfigure}

    \caption{Rawlsian Value comparison for $\lambda=0.5$ to $0.9$}
\end{figure}
\FloatBarrier

\clearpage
Below we present a set of plots for Utilitarian objective for $\lambda \in \{0.1, 0.2, \dots, 0.9\}$.

\begin{figure}[H]
    \centering

    \begin{minipage}[b]{0.32\textwidth}
        \centering
        \textbf{Adult}
    \end{minipage}
    \hfill
    \begin{minipage}[b]{0.32\textwidth}
        \centering
        \textbf{Credit Card }
    \end{minipage}
    \hfill
    \begin{minipage}[b]{0.32\textwidth}
        \centering
        \textbf{Census }
    \end{minipage}

    \begin{subfigure}[b]{0.32\textwidth}
        \includegraphics[width=\textwidth]{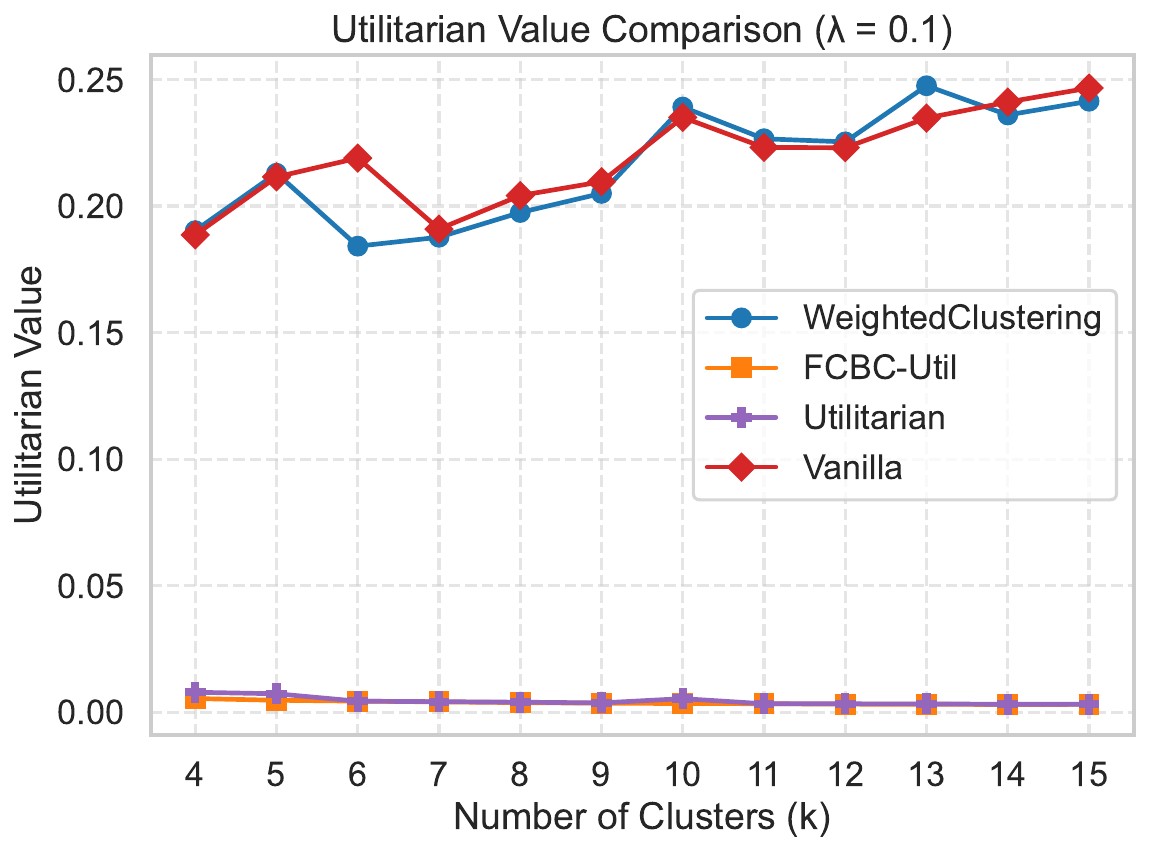}
        
        \label{fig:adult_a}
    \end{subfigure}
    \hfill
    \begin{subfigure}[b]{0.32\textwidth}
        \includegraphics[width=\textwidth]{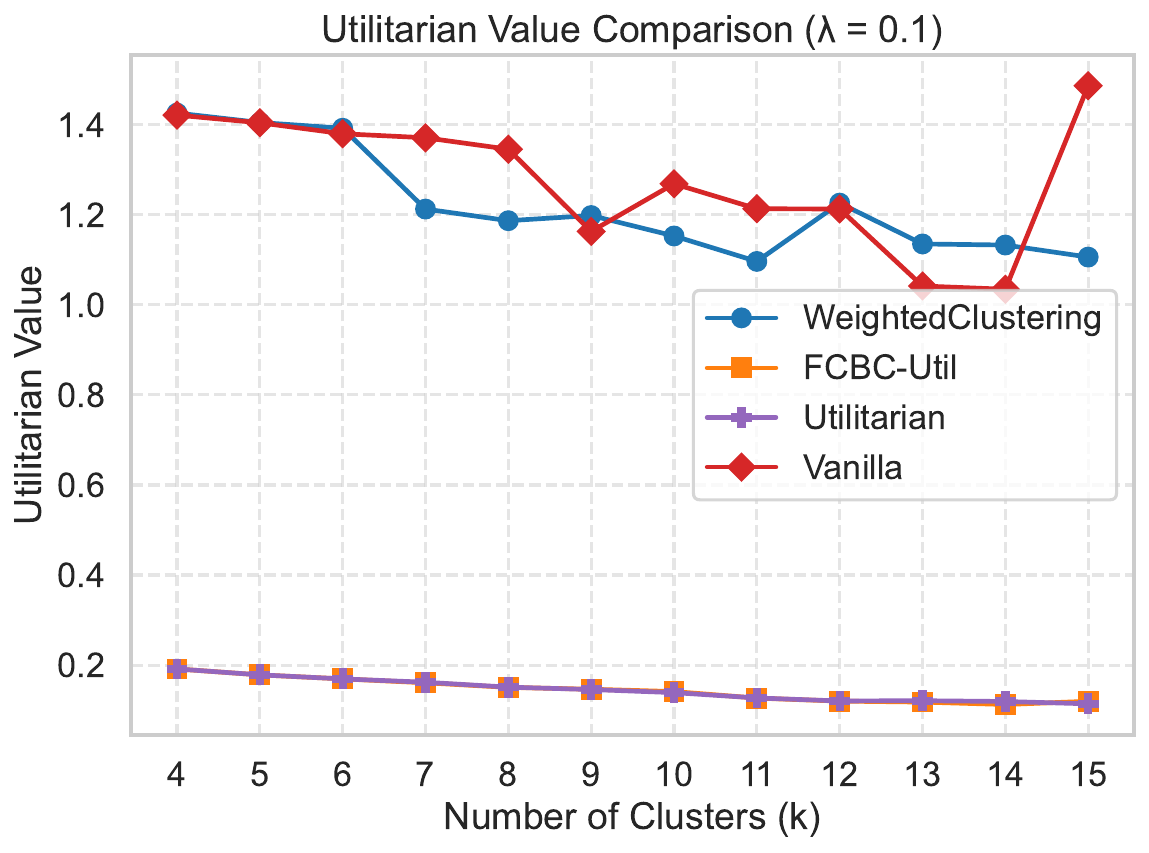}
        
        \label{fig:credit_a}
    \end{subfigure}
    \hfill
    \begin{subfigure}[b]{0.32\textwidth}
        \includegraphics[width=\textwidth]{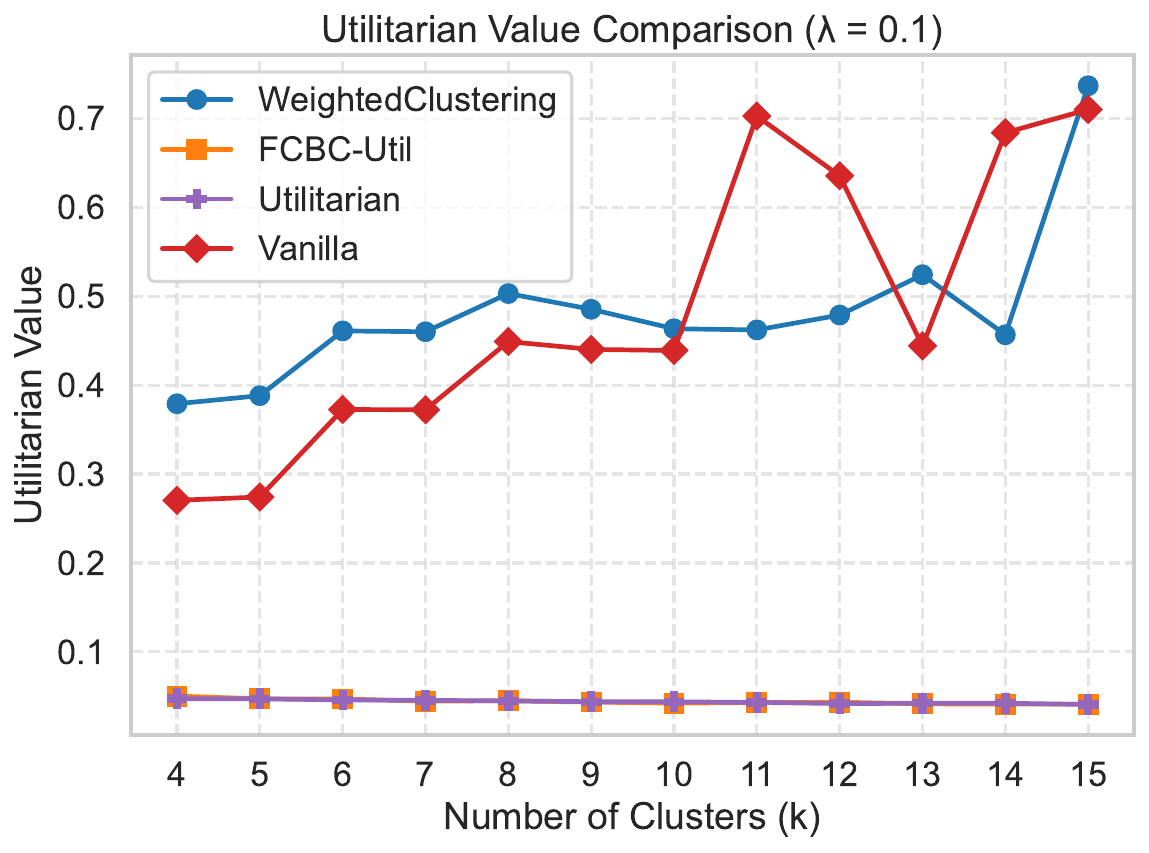}
        
        \label{fig:census_a}
    \end{subfigure}

     \begin{subfigure}[b]{0.32\textwidth}
        \includegraphics[width=\textwidth]{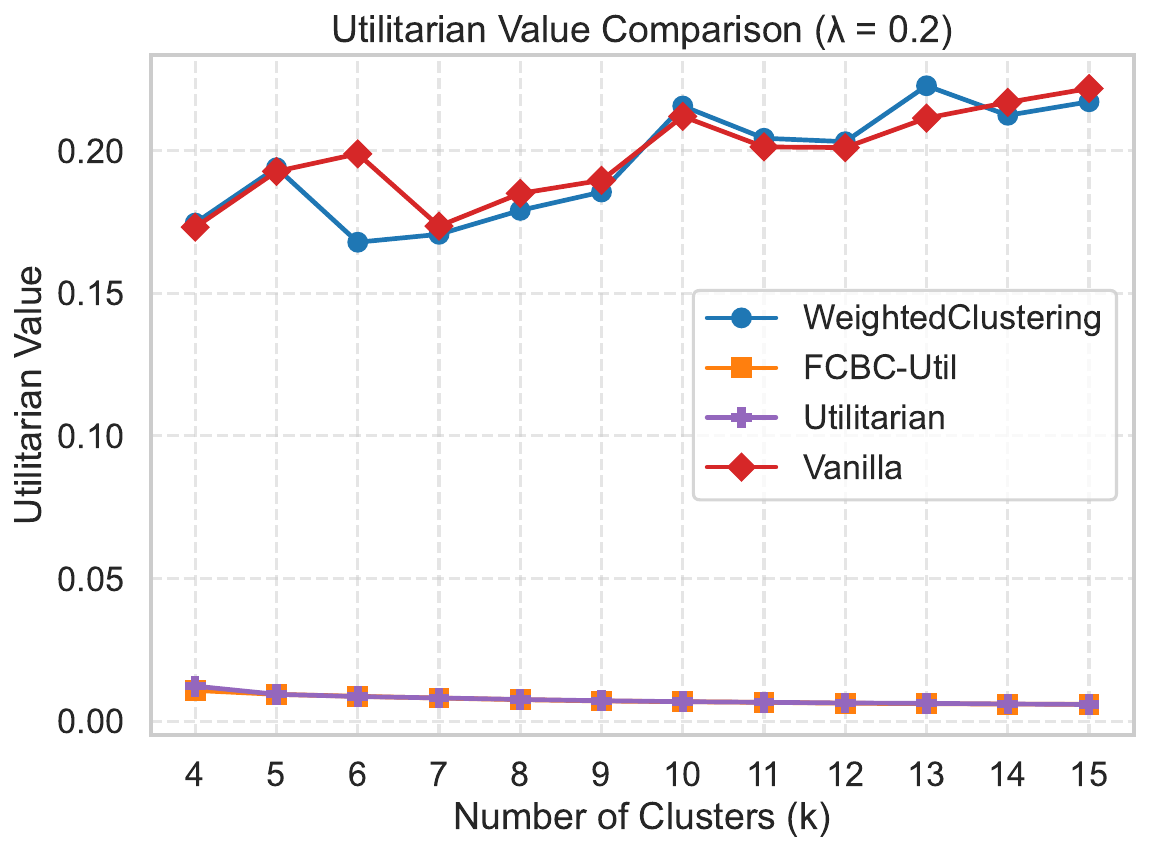}
        
        \label{fig:adult_a}
    \end{subfigure}
    \hfill
    \begin{subfigure}[b]{0.32\textwidth}
        \includegraphics[width=\textwidth]{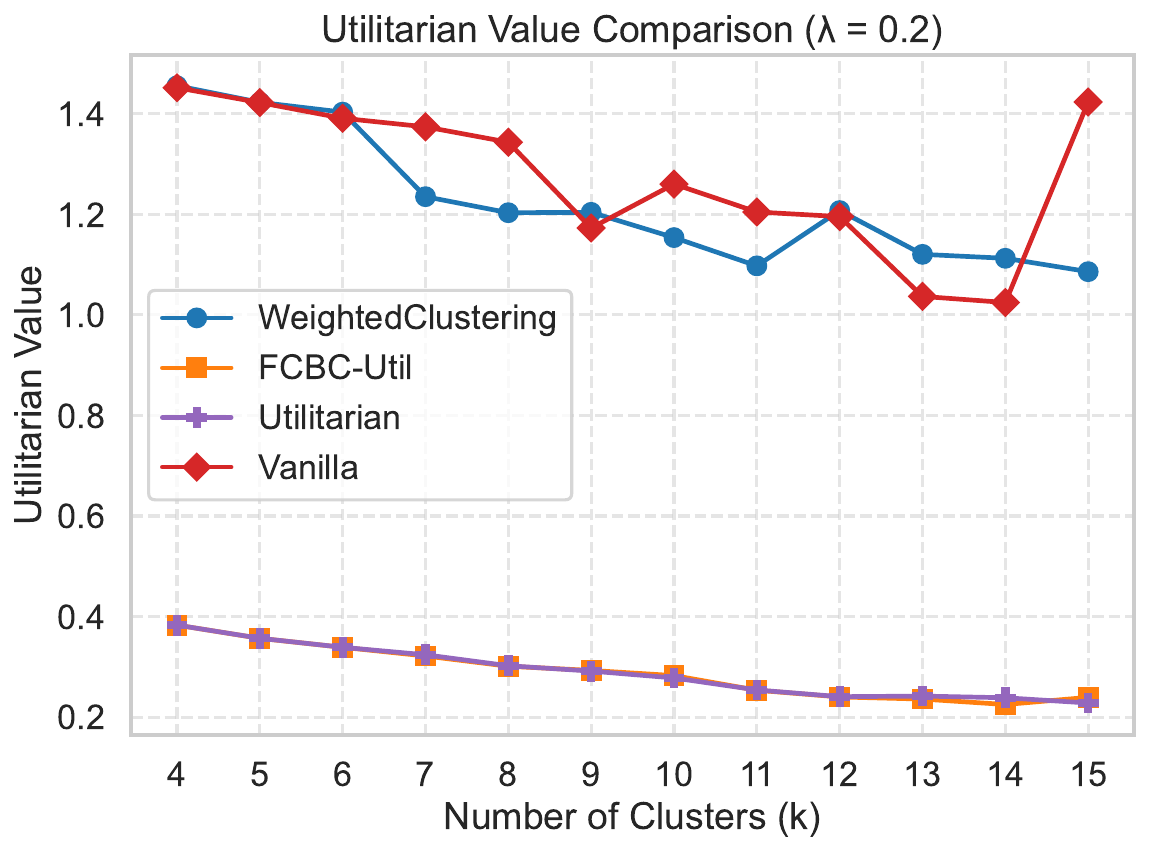}
        
        \label{fig:credit_a}
    \end{subfigure}
    \hfill
    \begin{subfigure}[b]{0.32\textwidth}
        \includegraphics[width=\textwidth]{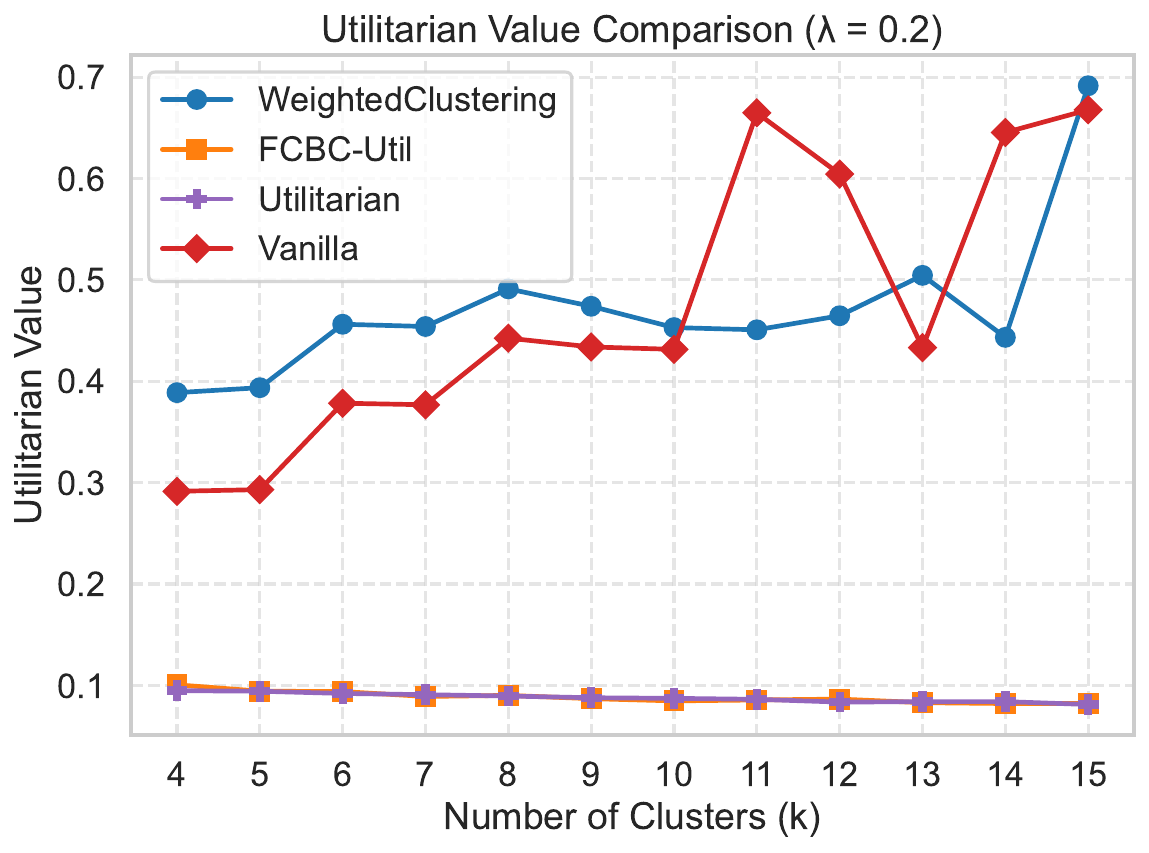}
        
        \label{fig:census_a}
    \end{subfigure}

     \begin{subfigure}[b]{0.32\textwidth}
        \includegraphics[width=\textwidth]{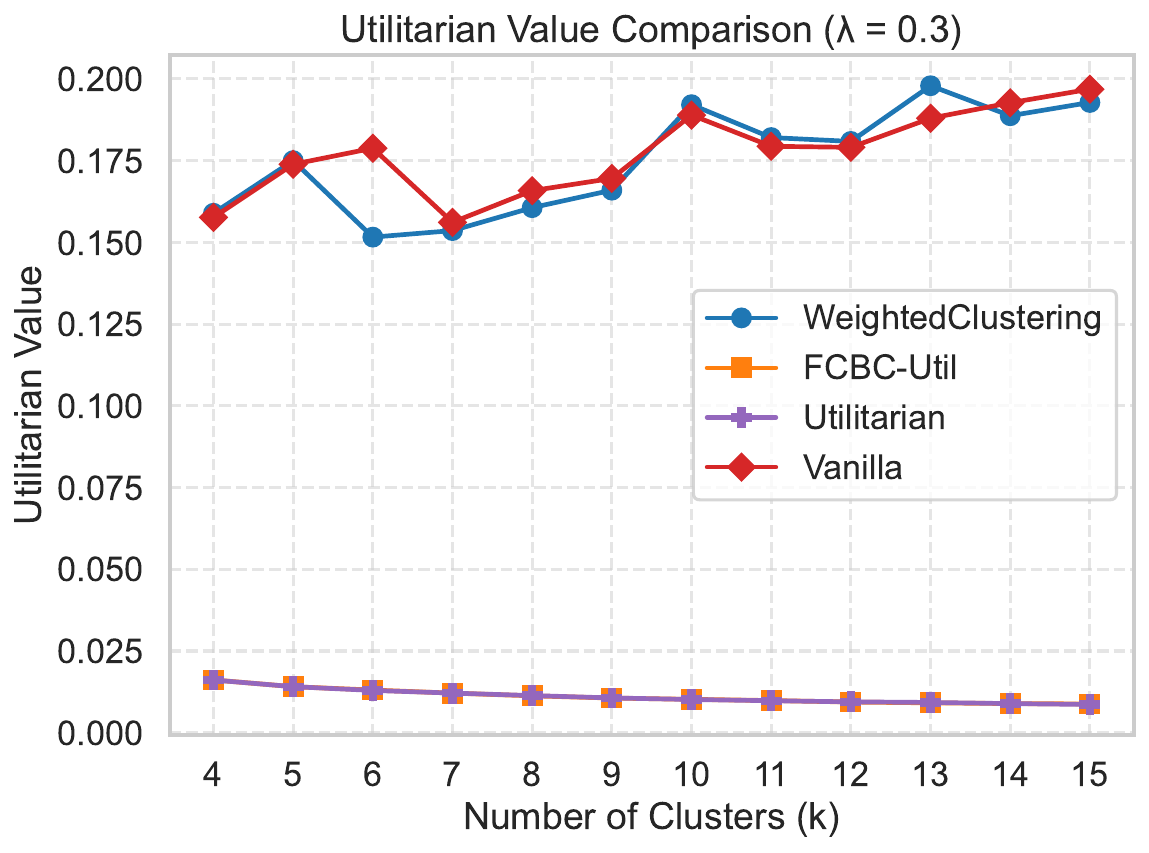}
        
        \label{fig:adult_a}
    \end{subfigure}
    \hfill
    \begin{subfigure}[b]{0.32\textwidth}
        \includegraphics[width=\textwidth]{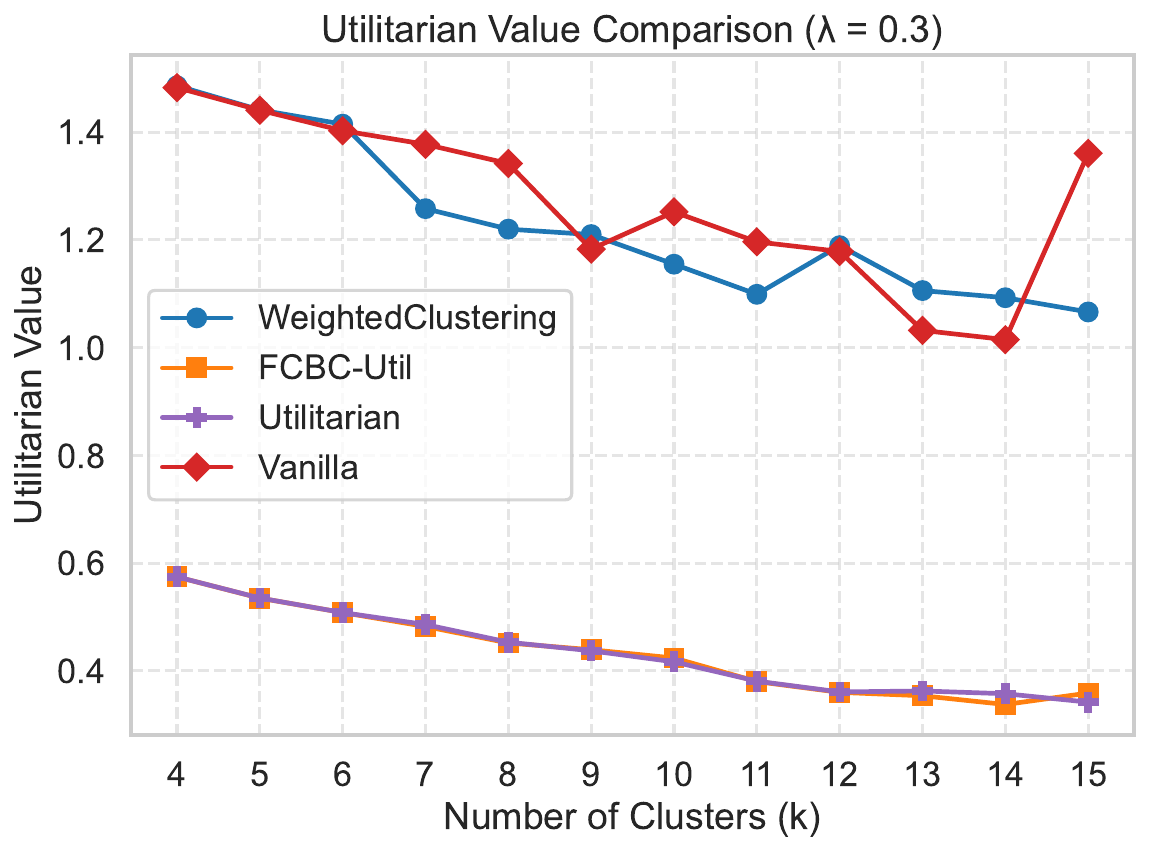}
        
        \label{fig:credit_a}
    \end{subfigure}
    \hfill
    \begin{subfigure}[b]{0.32\textwidth}
        \includegraphics[width=\textwidth]{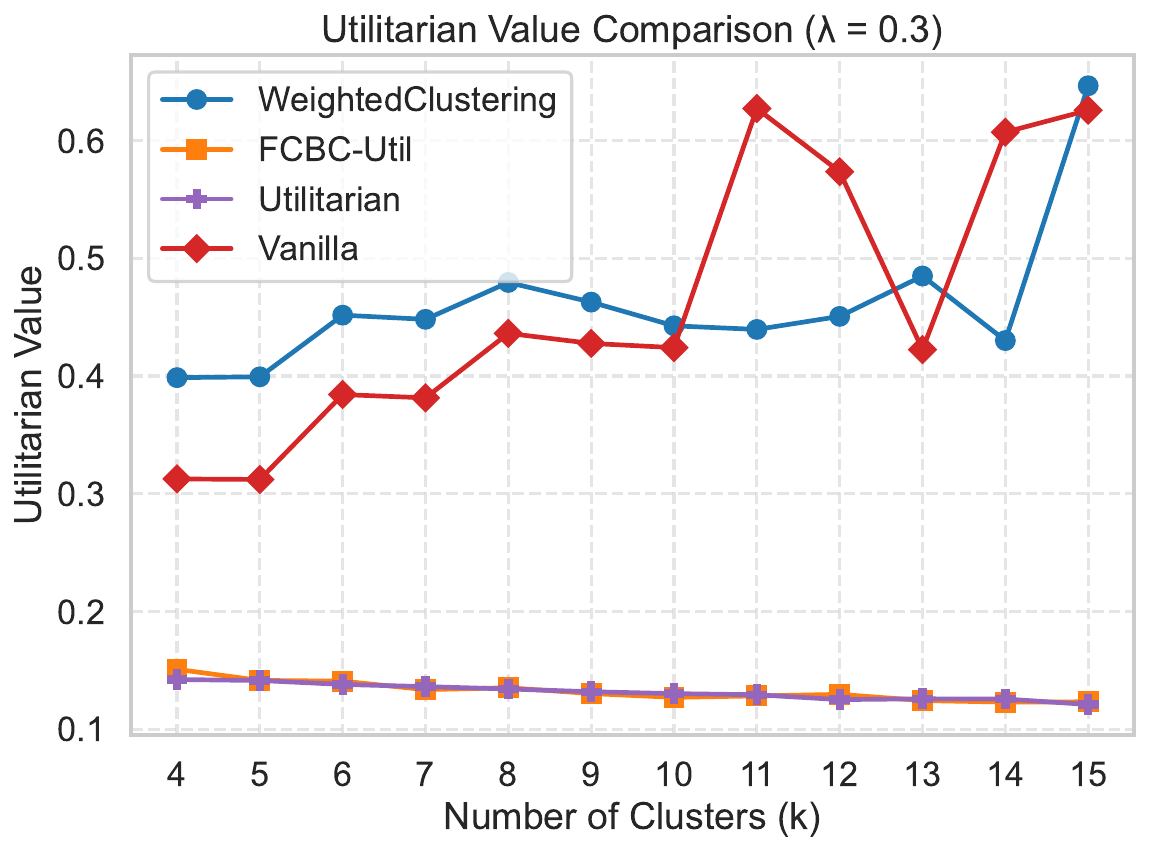}
        
        \label{fig:census_a}
    \end{subfigure}

     \begin{subfigure}[b]{0.32\textwidth}
        \includegraphics[width=\textwidth]{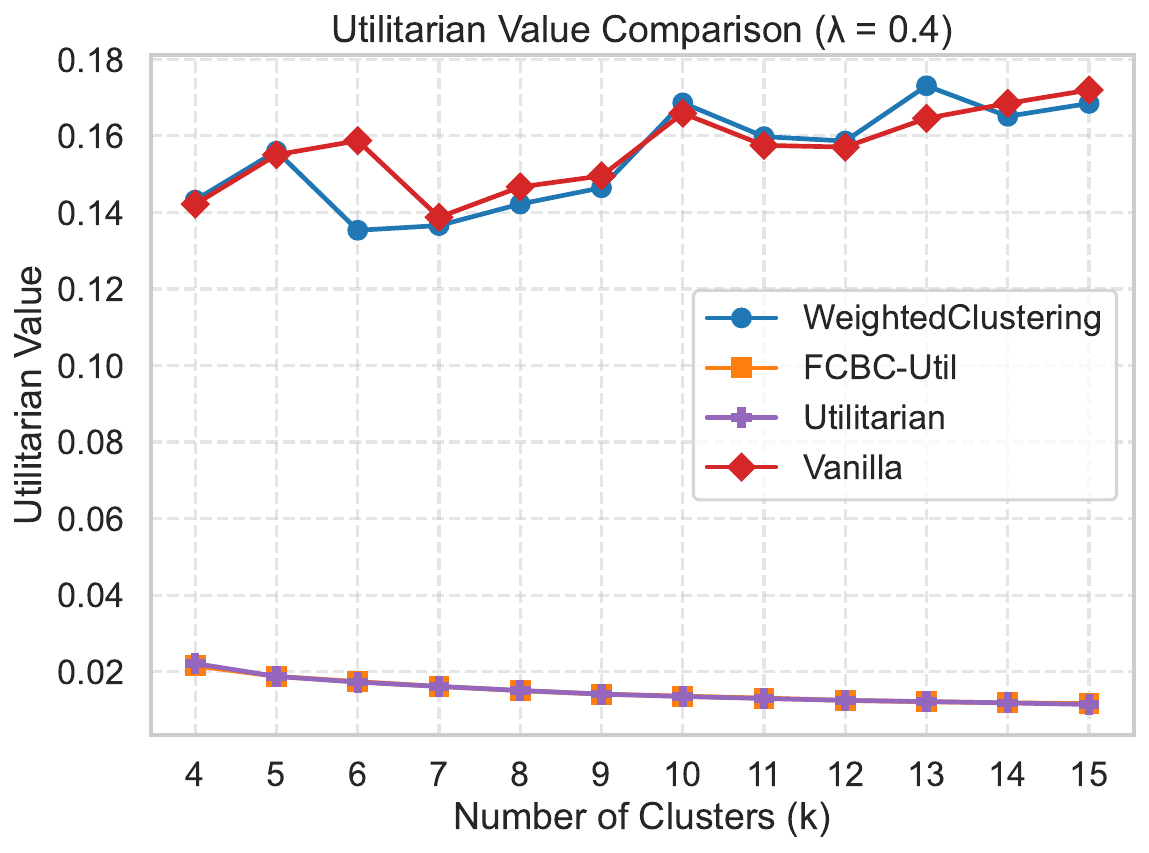}
        
        \label{fig:adult_a}
    \end{subfigure}
    \hfill
    \begin{subfigure}[b]{0.32\textwidth}
        \includegraphics[width=\textwidth]{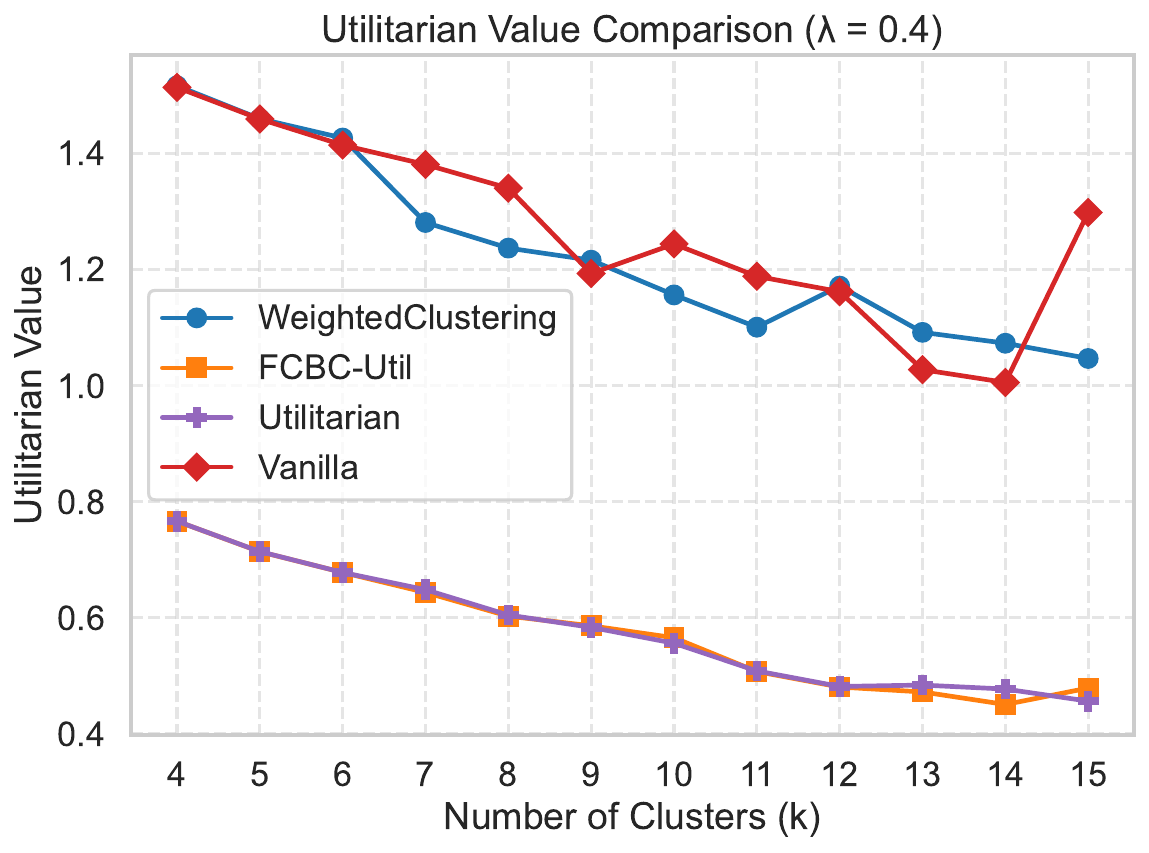}
        
        \label{fig:credit_a}
    \end{subfigure}
    \hfill
    \begin{subfigure}[b]{0.32\textwidth}
        \includegraphics[width=\textwidth]{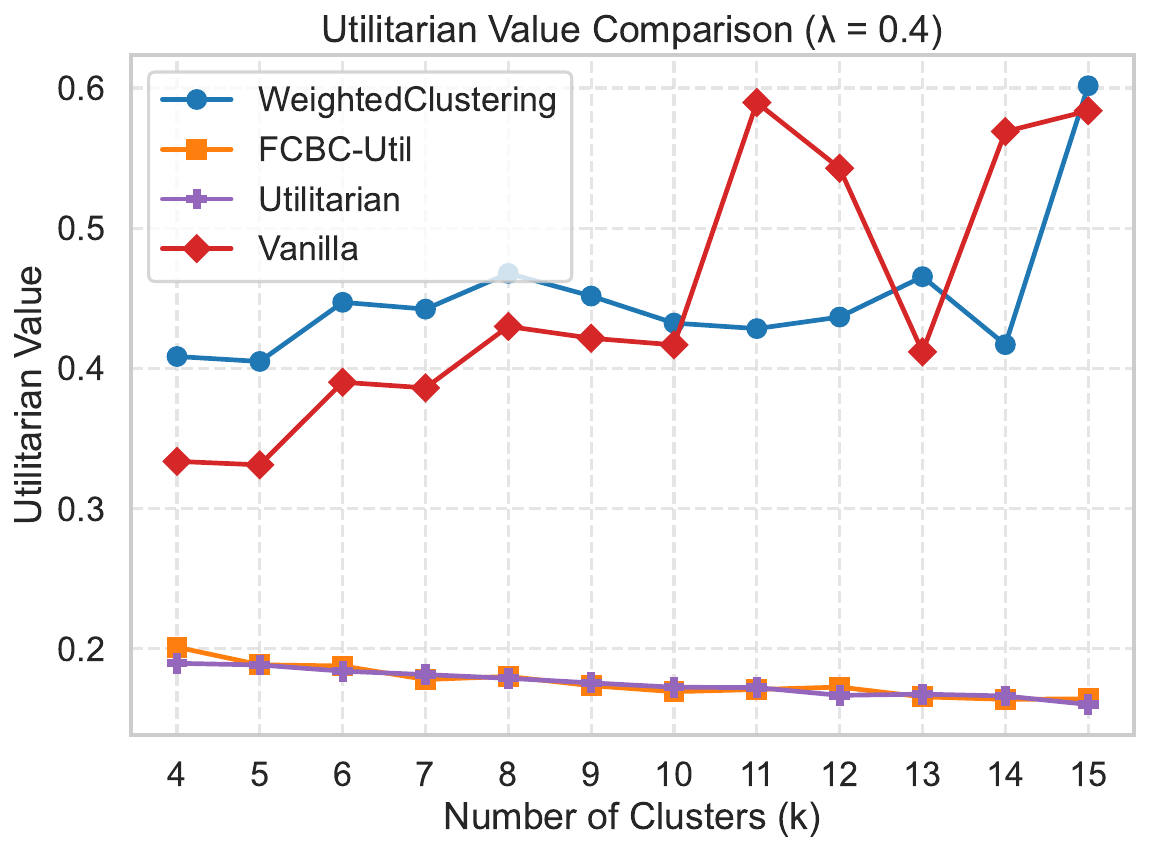}
        
        \label{fig:census_a}
    \end{subfigure}

    \begin{subfigure}[b]{0.32\textwidth}
        \includegraphics[width=\textwidth]{plots/two_group/ad_utilitarian_costs_comparison_lambda_0_500.pdf}
        
        \label{fig:adult_a}
    \end{subfigure}
    \hfill
    \begin{subfigure}[b]{0.32\textwidth}
        \includegraphics[width=\textwidth]{plots/two_group/cc_utilitarian_costs_comparison_lambda_0_500.pdf}
        
        \label{fig:credit_a}
    \end{subfigure}
    \hfill
    \begin{subfigure}[b]{0.32\textwidth}
        \includegraphics[width=\textwidth]{plots/two_group/cen_utilitarian_costs_comparison_lambda_0_500.pdf}
        
        \label{fig:census_a}
    \end{subfigure}
    \caption{Utilitarian Value comparison for $\lambda=0.1$ to $0.5$}

    \end{figure}
\FloatBarrier

\begin{figure}[H]
    \centering

    \begin{subfigure}[b]{0.32\textwidth}
        \includegraphics[width=\textwidth]{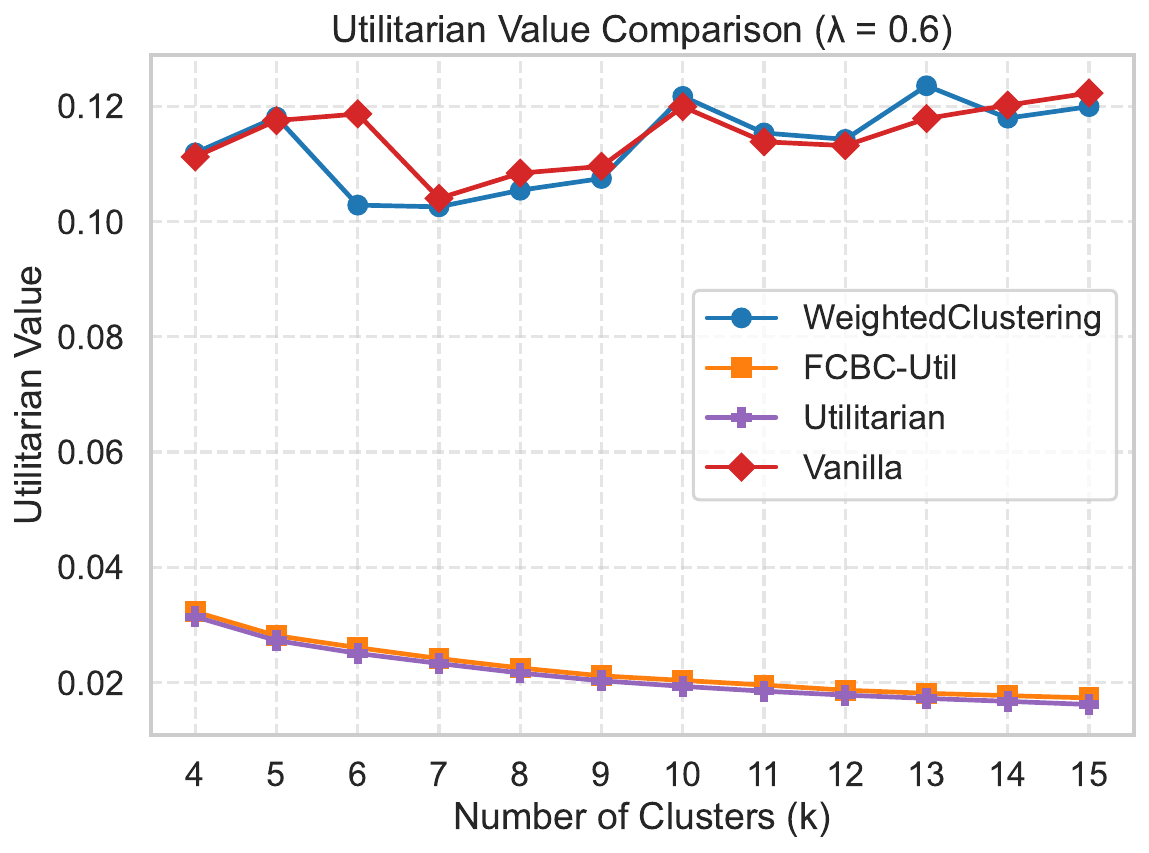}
        
        \label{fig:adult_a}
    \end{subfigure}
    \hfill
    \begin{subfigure}[b]{0.32\textwidth}
        \includegraphics[width=\textwidth]{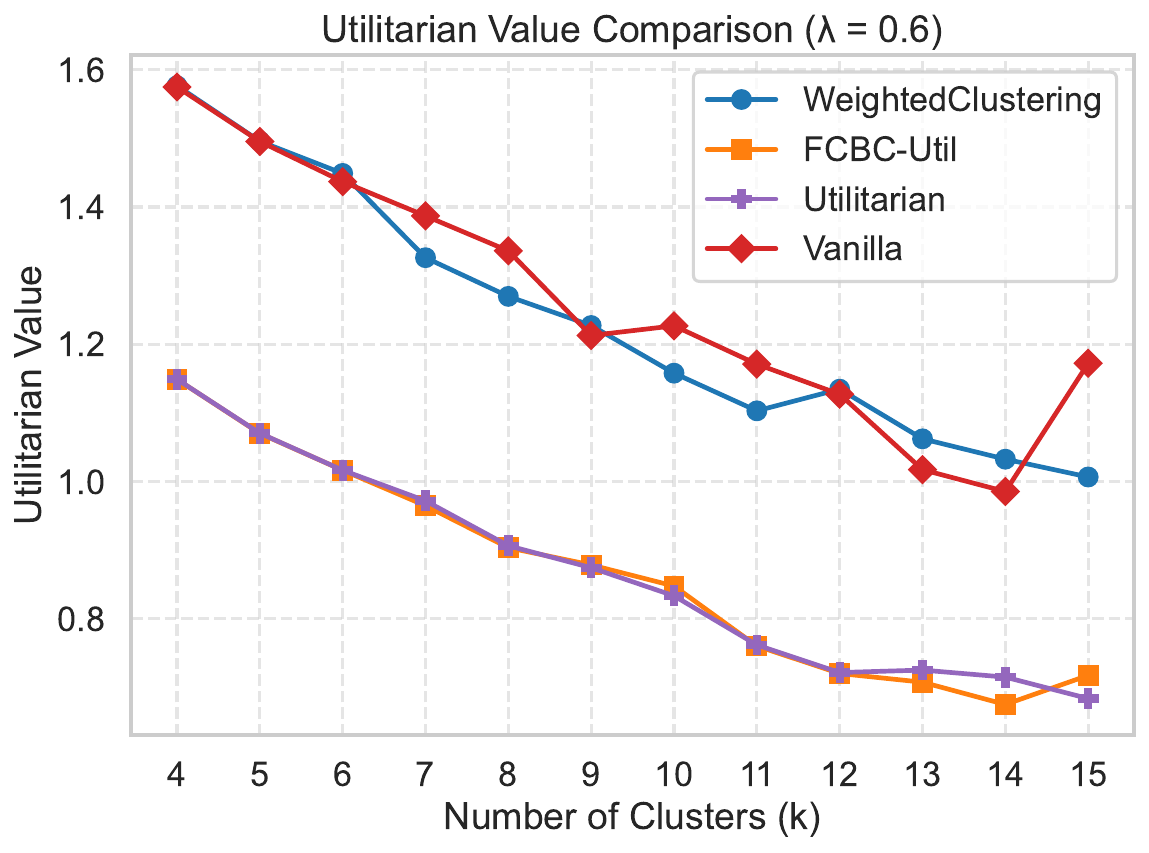}
        
        \label{fig:credit_a}
    \end{subfigure}
    \hfill
    \begin{subfigure}[b]{0.32\textwidth}
        \includegraphics[width=\textwidth]{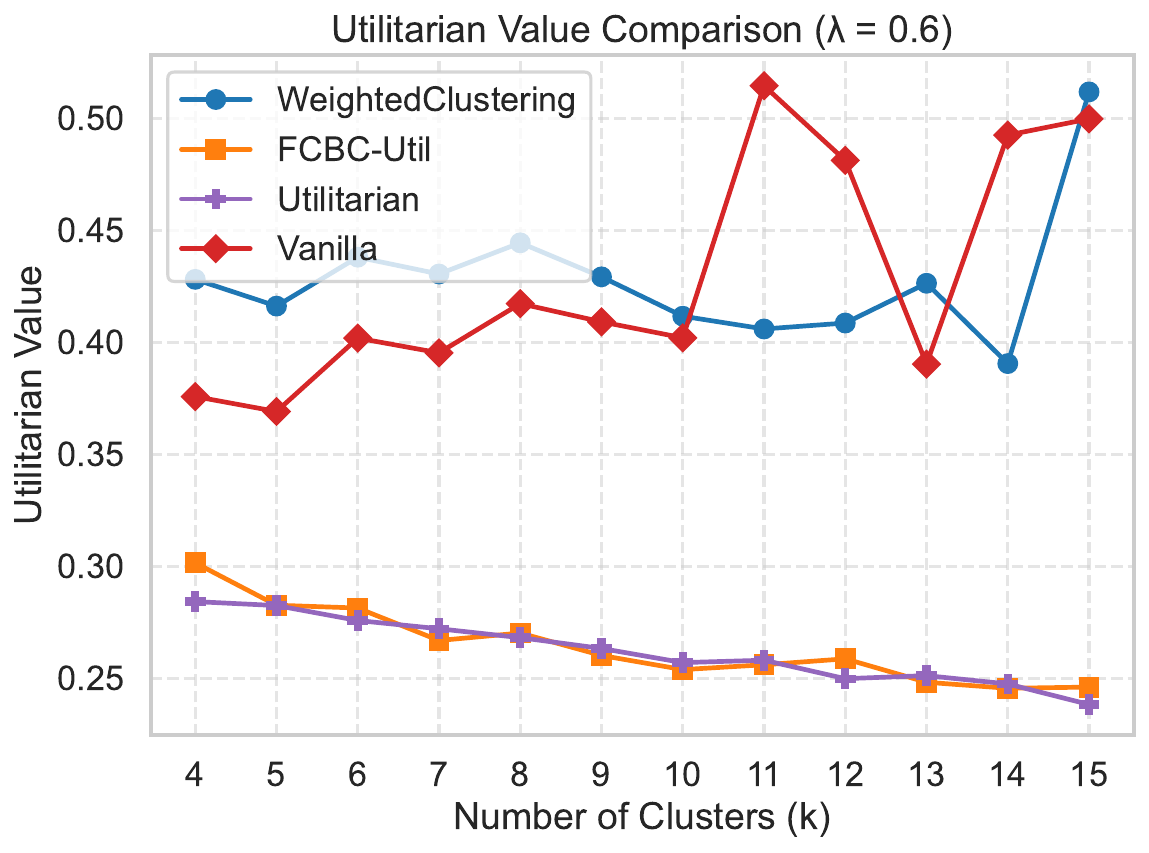}
        
        \label{fig:census_a}
    \end{subfigure}

    \begin{subfigure}[b]{0.32\textwidth}
        \includegraphics[width=\textwidth]{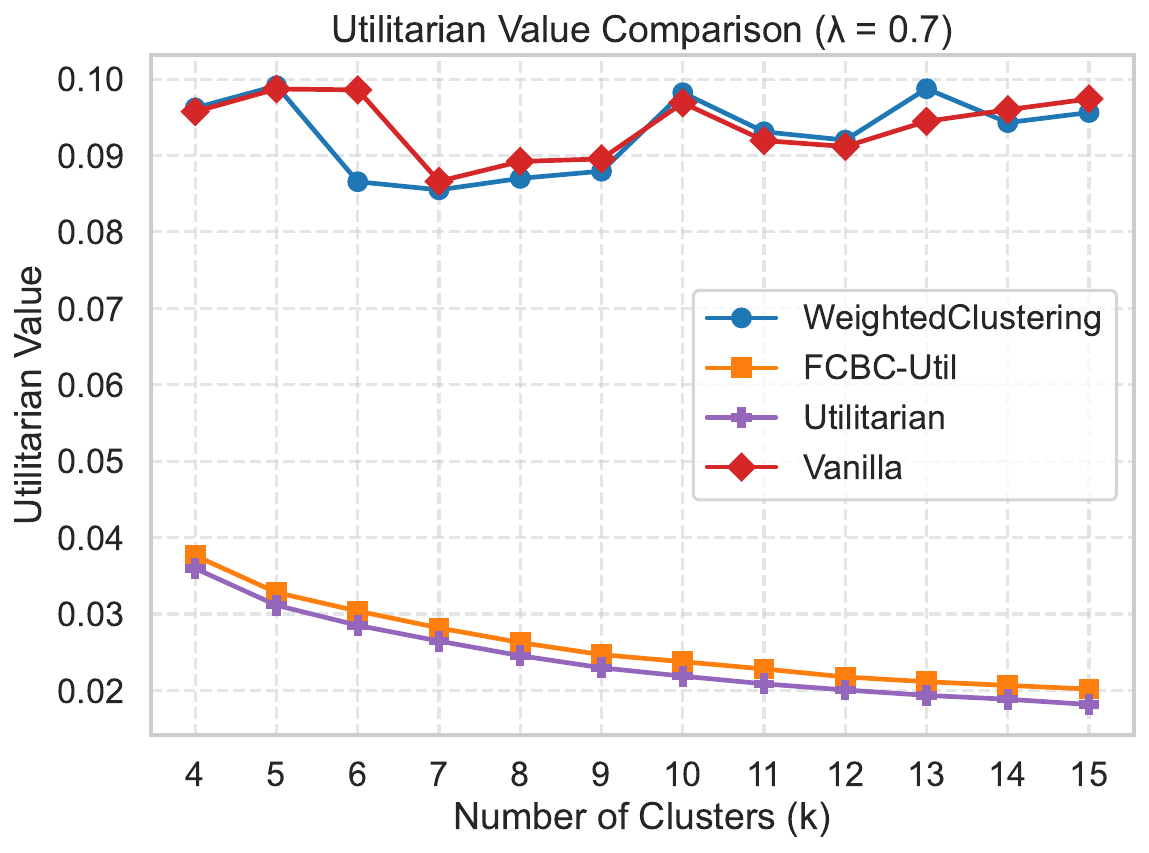}
        
        \label{fig:adult_a}
    \end{subfigure}
    \hfill
    \begin{subfigure}[b]{0.32\textwidth}
        \includegraphics[width=\textwidth]{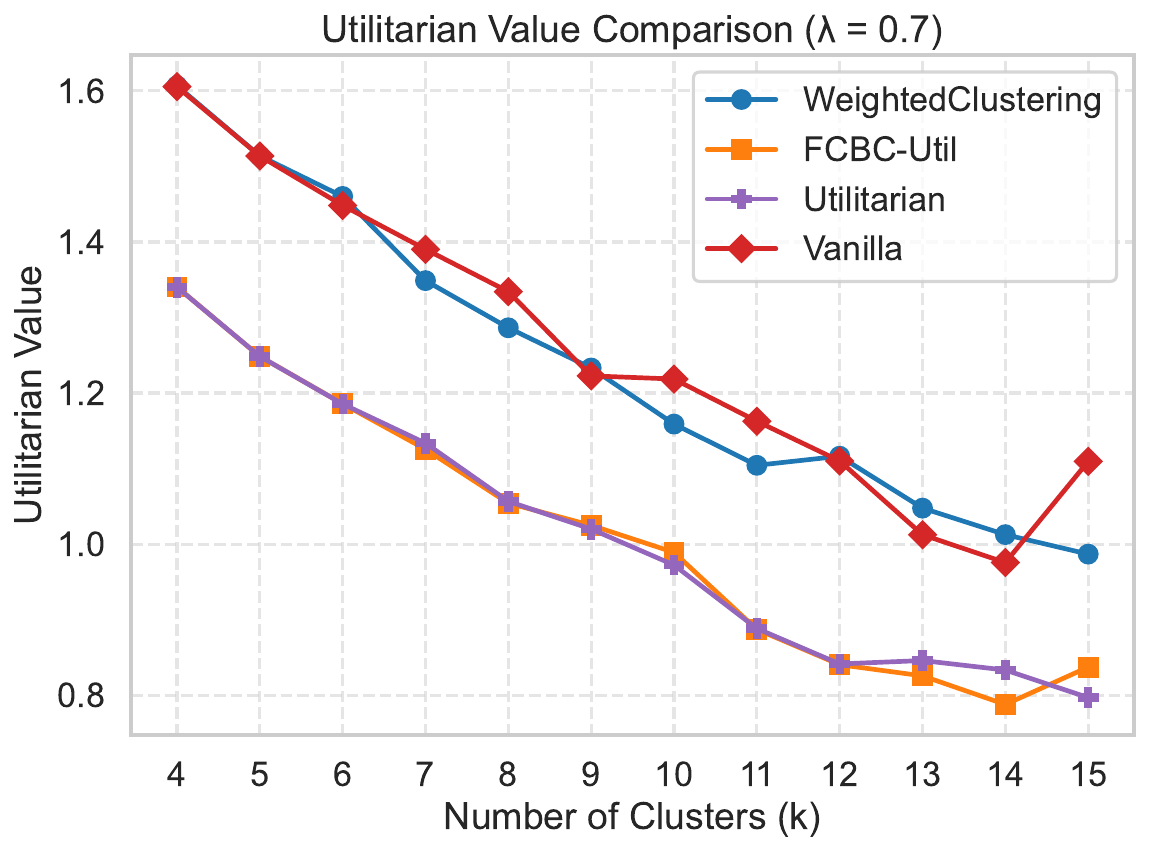}
        
        \label{fig:credit_a}
    \end{subfigure}
    \hfill
    \begin{subfigure}[b]{0.32\textwidth}
        \includegraphics[width=\textwidth]{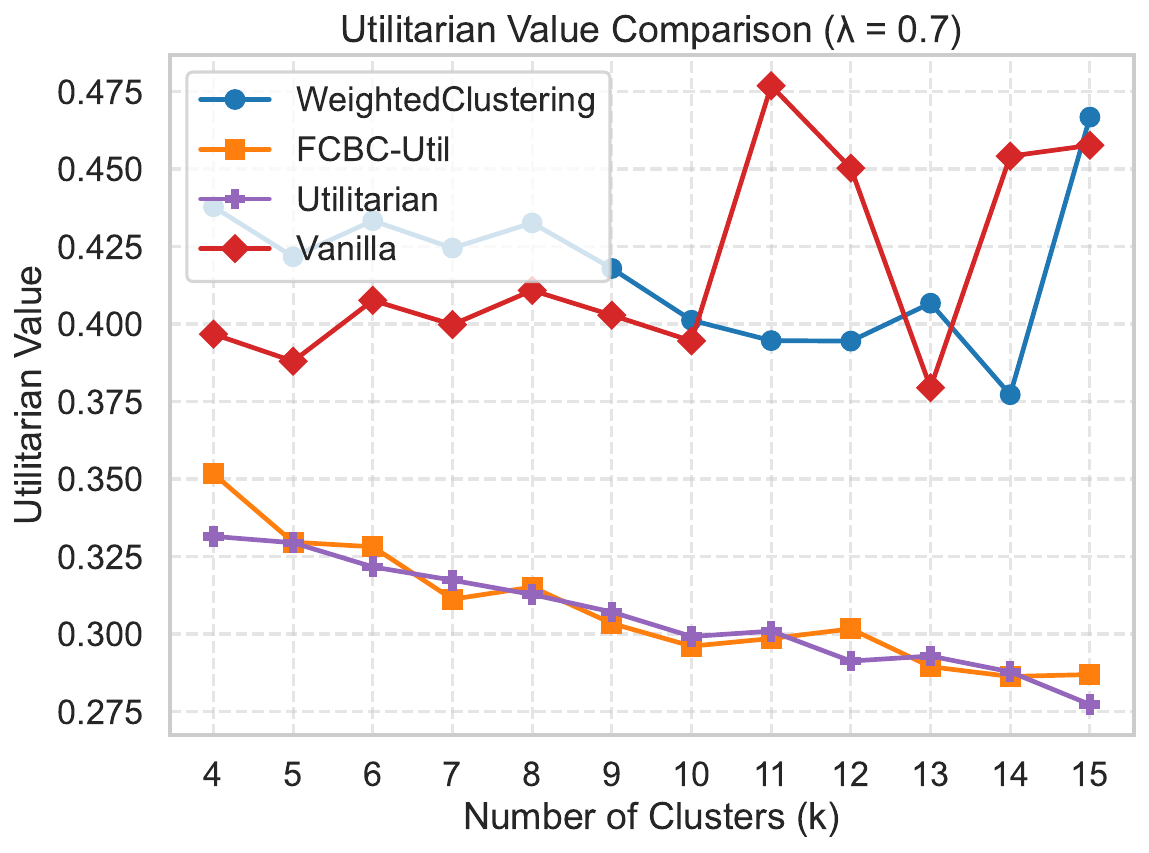}
        
        \label{fig:census_a}
    \end{subfigure}

     \begin{subfigure}[b]{0.32\textwidth}
        \includegraphics[width=\textwidth]{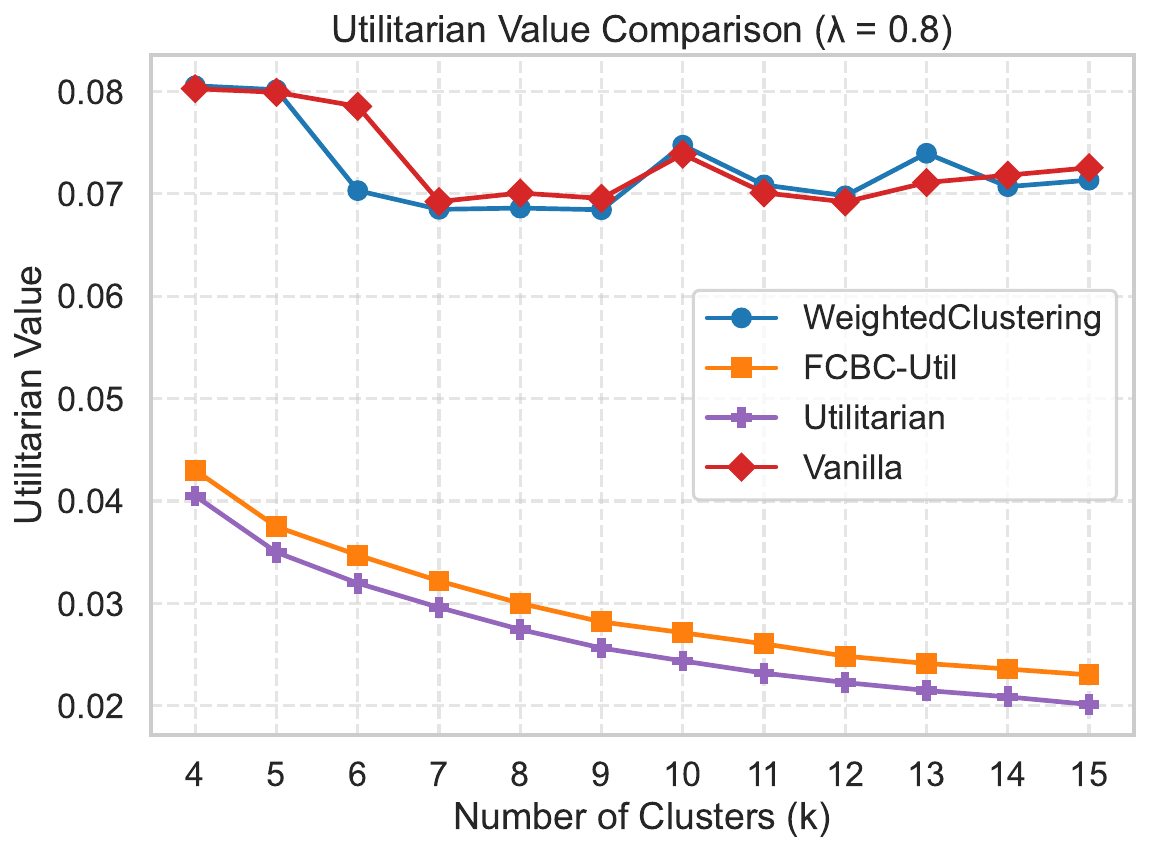}
        
        \label{fig:adult_a}
    \end{subfigure}
    \hfill
    \begin{subfigure}[b]{0.32\textwidth}
        \includegraphics[width=\textwidth]{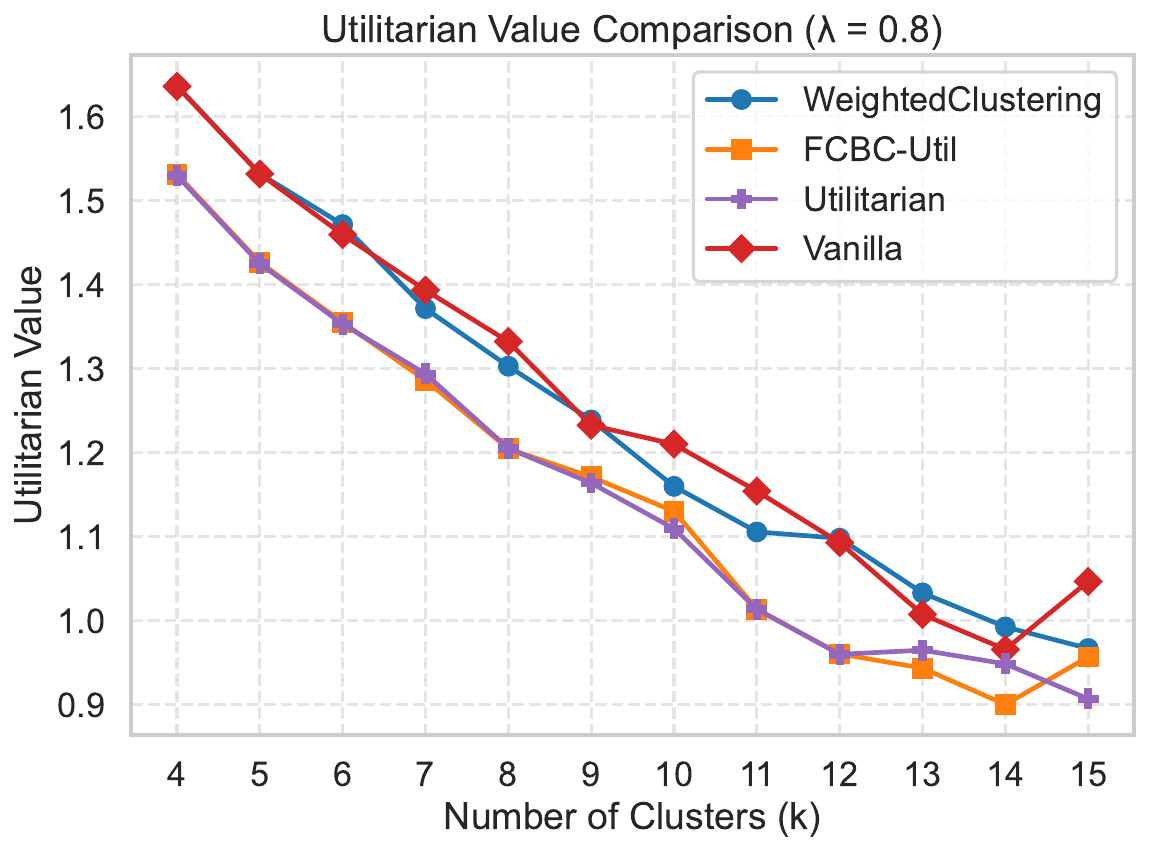}
        
        \label{fig:credit_a}
    \end{subfigure}
    \hfill
    \begin{subfigure}[b]{0.32\textwidth}
        \includegraphics[width=\textwidth]{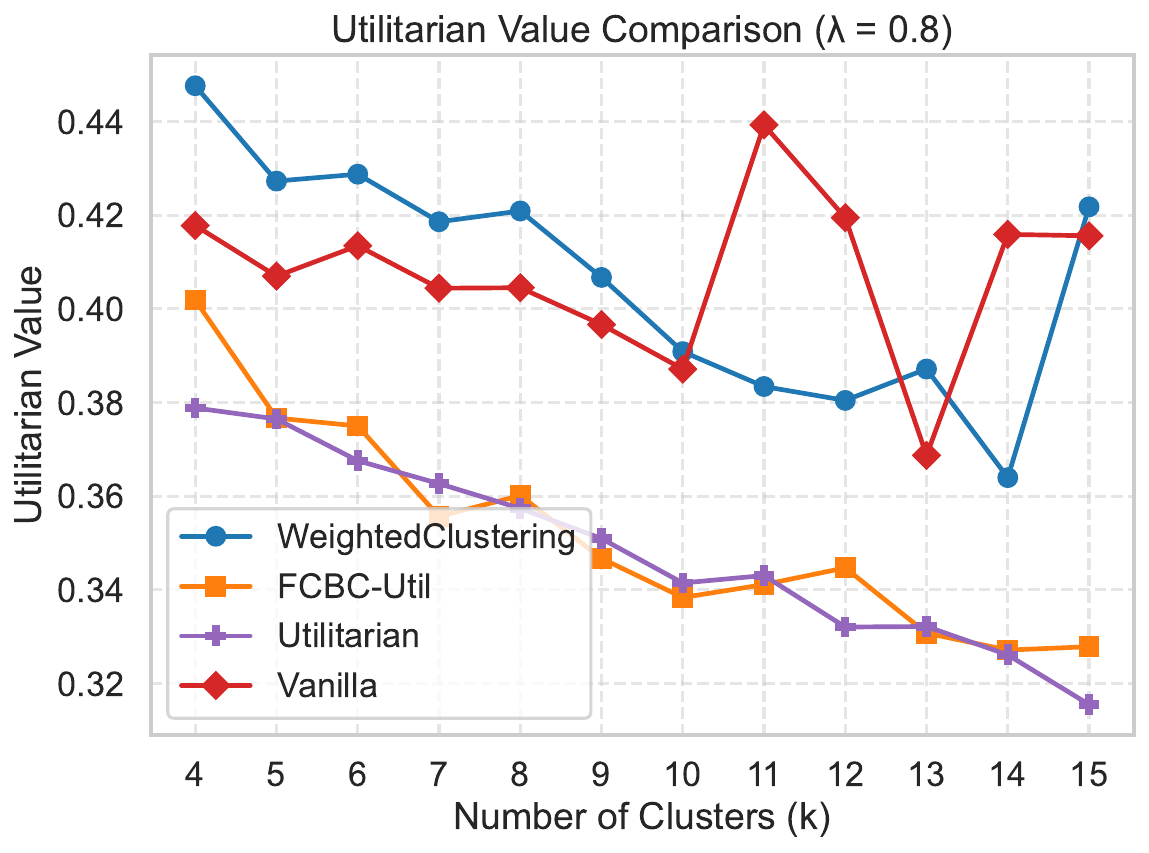}
        
        \label{fig:census_a}
    \end{subfigure}

     \begin{subfigure}[b]{0.32\textwidth}
        \includegraphics[width=\textwidth]{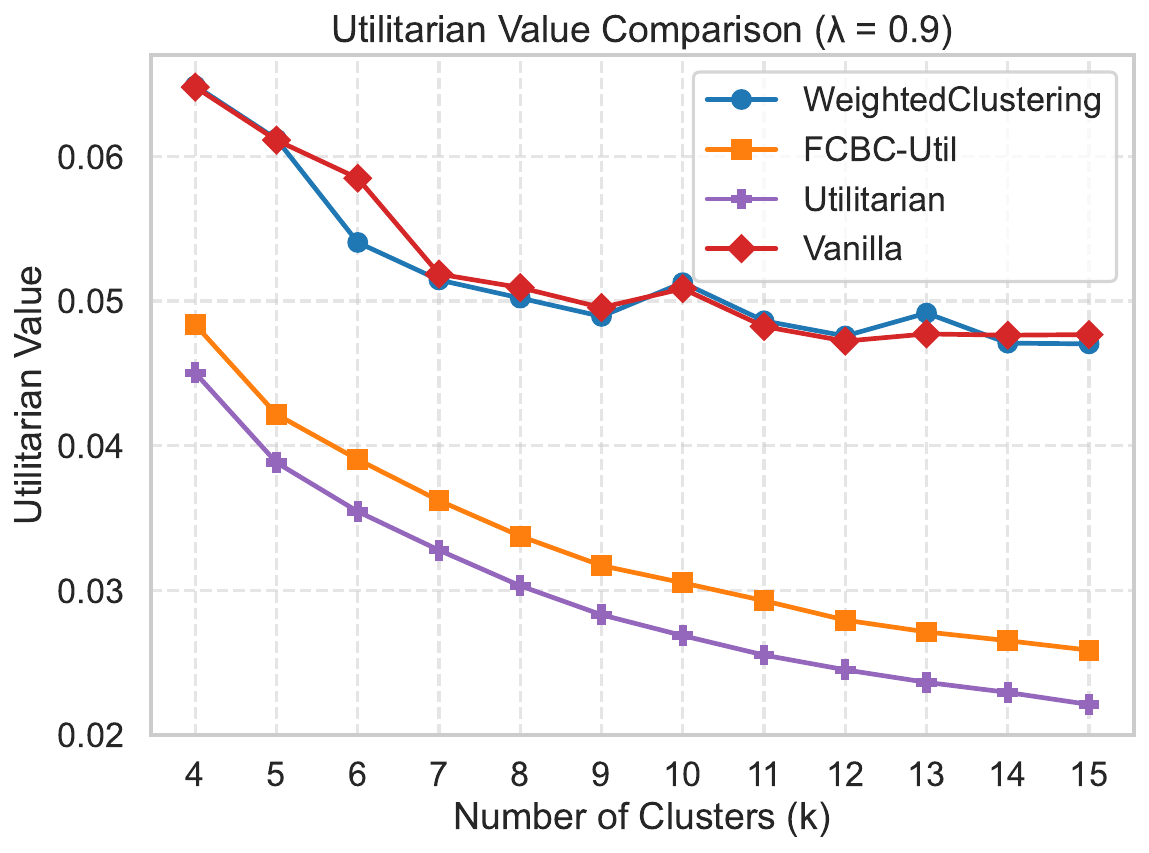}
        
        \label{fig:adult_a}
    \end{subfigure}
    \hfill
    \begin{subfigure}[b]{0.32\textwidth}
        \includegraphics[width=\textwidth]{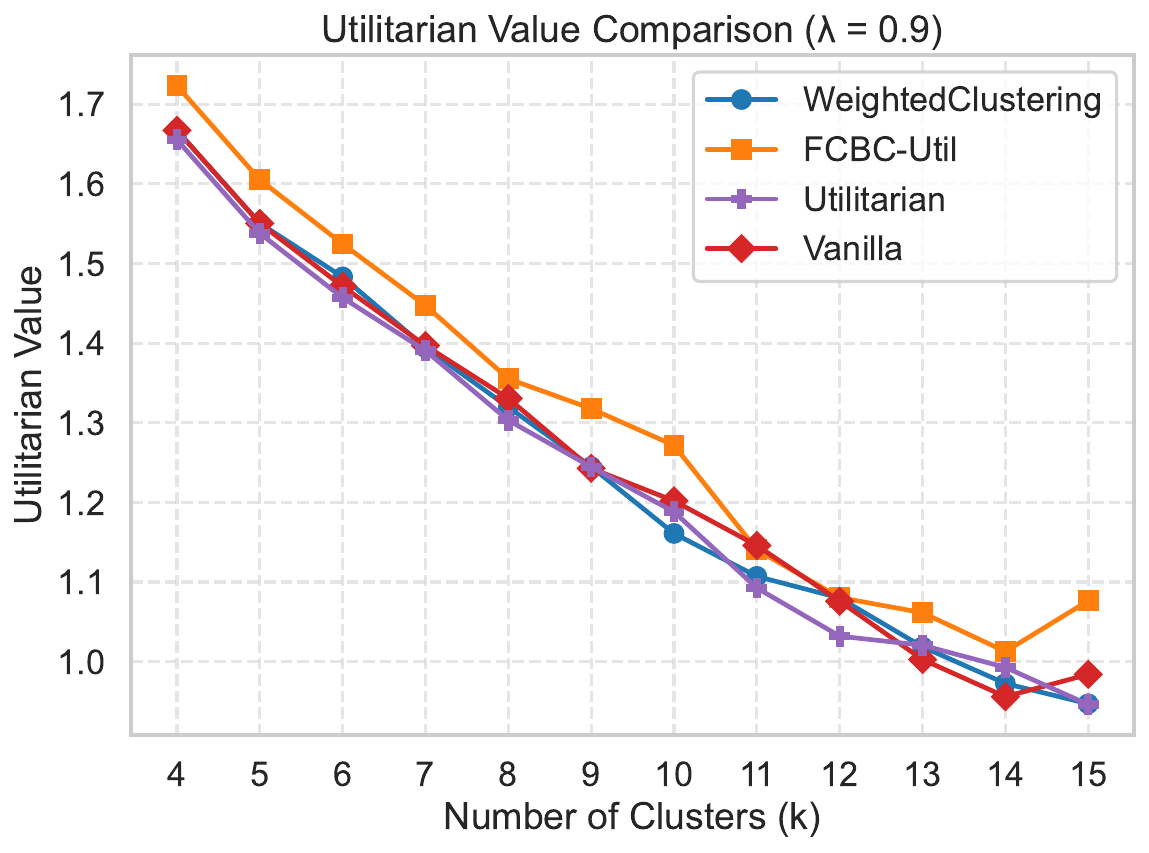}
        
        \label{fig:credit_a}
    \end{subfigure}
    \hfill
    \begin{subfigure}[b]{0.32\textwidth}
        \includegraphics[width=\textwidth]{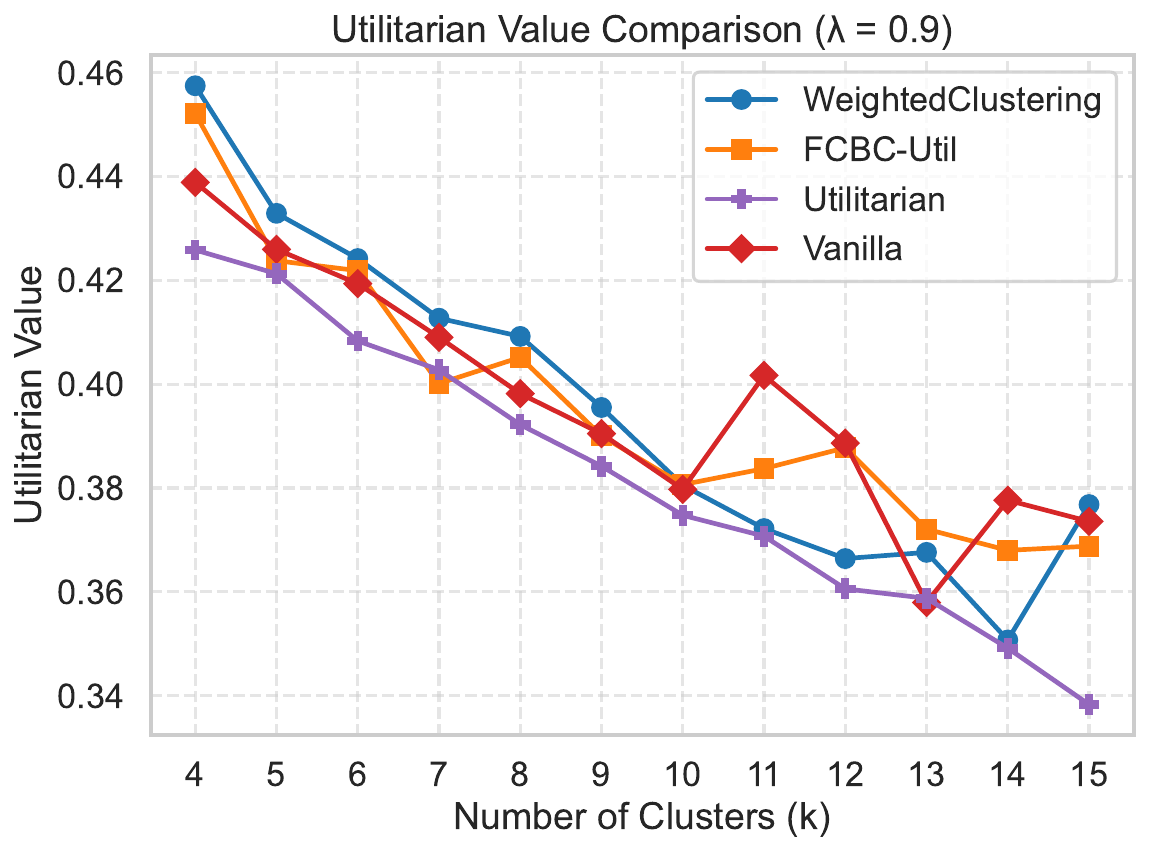}
        
        \label{fig:census_a}
    \end{subfigure}
    \caption{Utilitarian Value comparison for $\lambda=0.6$ to $0.9$}
\end{figure}
\FloatBarrier

\subsection{Comparison with FCBC with other Cost Upper Bounds}
In this section, we provide evidence to support our choice of clustering cost upper bound given to \textsc{FCBC-Rawl} and  \textsc{FCBC-Util}. We will denote the upper bound as $\text{pof}$, which stands for Price of Fairness, which is formally defined as $\text{pof}(c) = \frac{\text{Clustering Cost of Fair Solution $c$}}{\text{Cost of Optimal Agnostic Solution}}$ in fair clustering literature \cite{esmaeili2021fair, dickerson2024fair}. This quantity allows us to succinctly refer to upper bounds on clustering costs given to \textsc{FCBC-Rawl}  and \textsc{FCBC-Util}.

We set parameter $\text{pof}=1.5$ for FCBC algorithms in the comparison experiments. We run FCBC algorithms \textsc{FCBC-Rawl} and \textsc{FCBC-Util} on multiple $\text{pof}$ parameter values and plot Rawlsian and Utilitarian values respectively of solutions find. We observe at $\text{pof}=1.5$ and above the two FCBC algorithms' improvement on Rawlsian and Utilitarian values saturates. This is the case for all $\lambda$ choices in range $[0.1, 0.9]$. We present plot on $\lambda = 0.5$ for three datasets \adult{}, \credit{} and \cens{}.

\begin{figure}[htbp]
    \centering
    \includegraphics[width=0.8\textwidth]{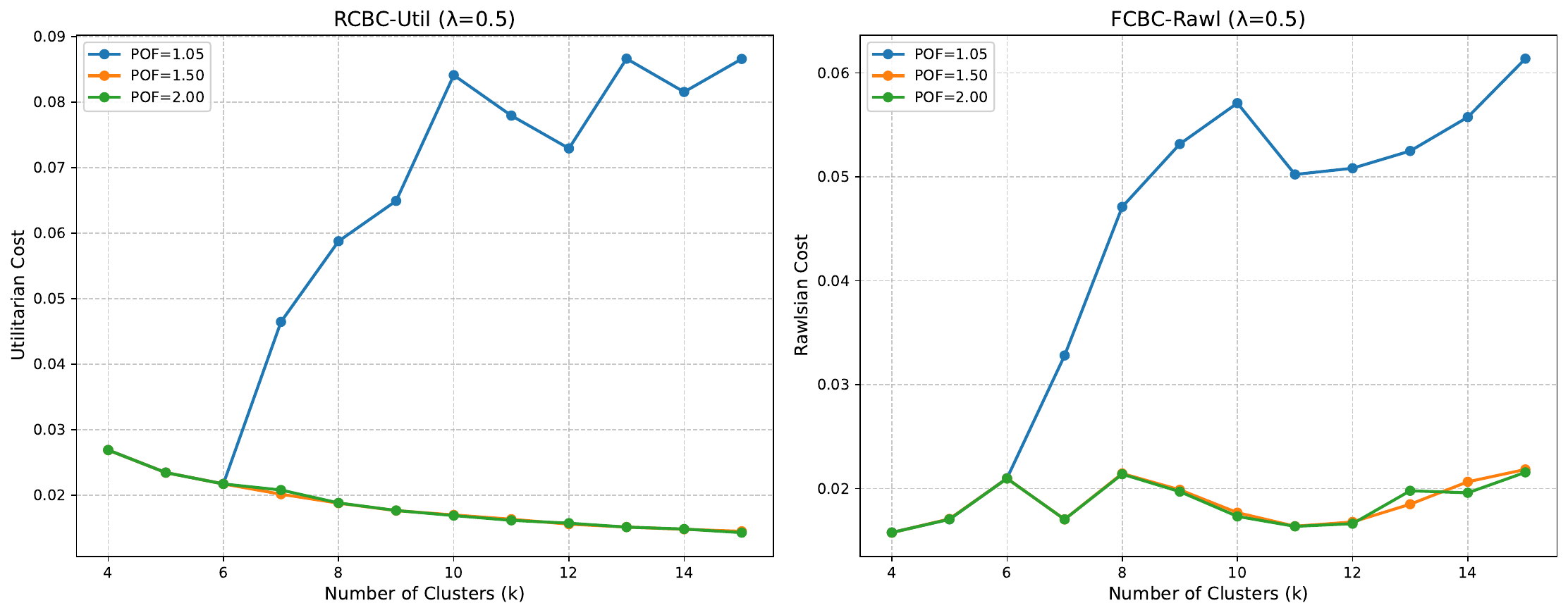}
    \caption{Welfare values for FCBC algorithms with different POF value on \adult{}.}
    \label{fig:example}
\end{figure}

\begin{figure}[htbp]
    \centering
    \includegraphics[width=0.8\textwidth]{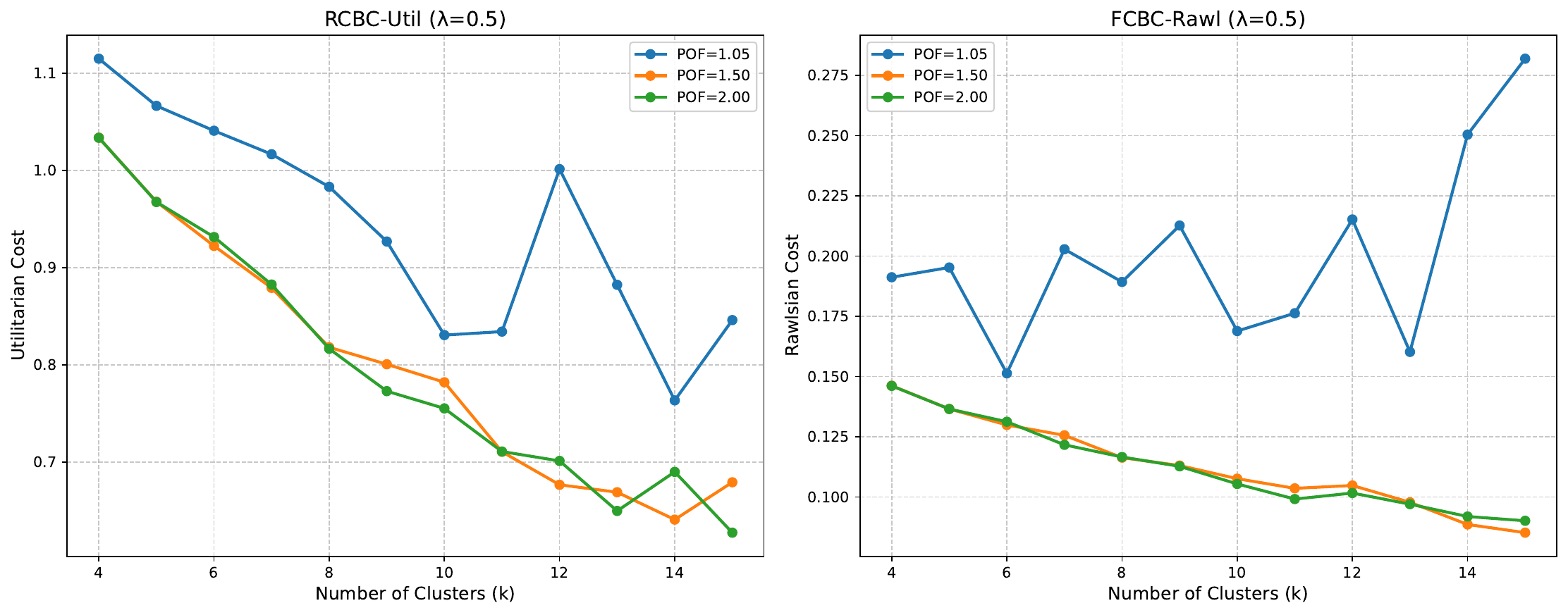}
    \caption{Welfare values for FCBC algorithms with different POF value on \credit{}{}.}
    \label{fig:example}
\end{figure}

\begin{figure}[htbp]
    \centering
    \includegraphics[width=0.8\textwidth]{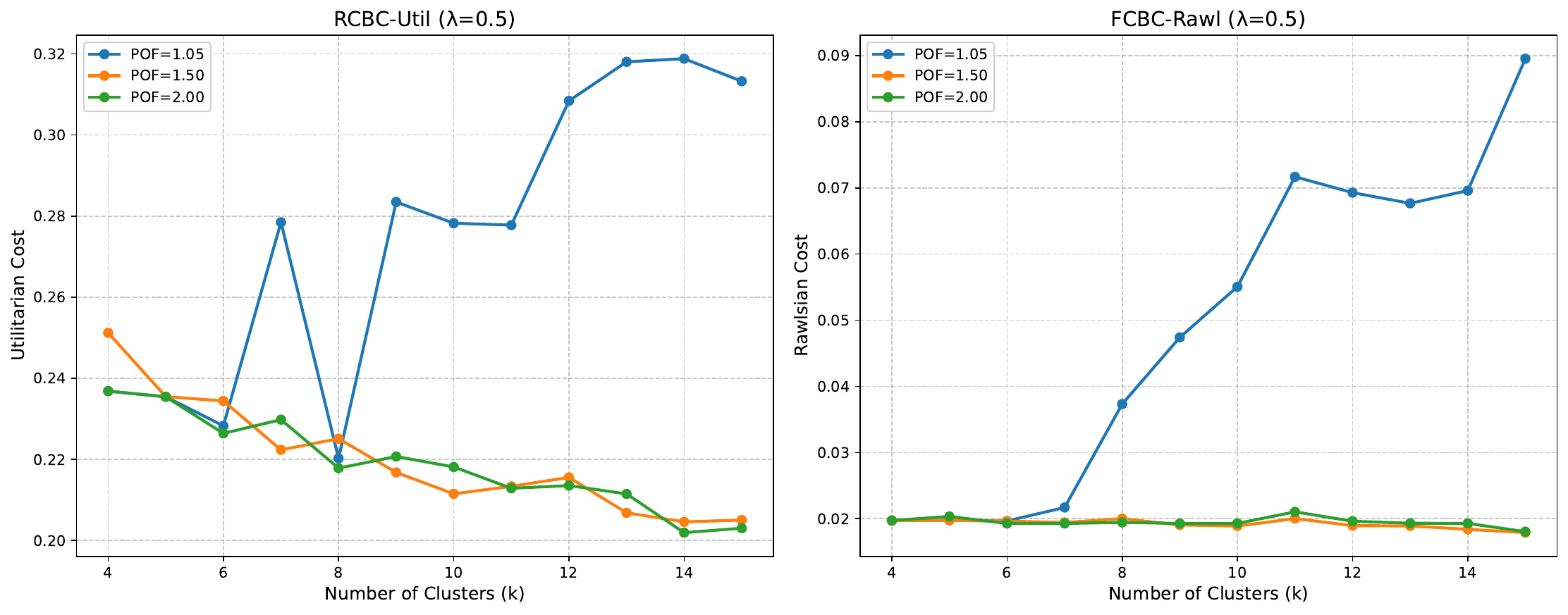}
    \caption{Welfare values for FCBC algorithms with different POF value on \cens{}.}
    \label{fig:example}
\end{figure}
\FloatBarrier

\subsection{Effect of Rounding on the Rawlsian and Utilitarian}
In this section, we empirically show that rounding does not change LP objective value too much for both Rawlsian and Utilitarian objectives. Here we are referring to the constants $\constr$ and $\constu$ that appear in Lemmas \ref{lemma:roundedfromLPrawls} and \ref{lemma:roundedfromLPutil}, respectively. Interestingly, as shown in the plots below, the change is quite small never exceeding $8\times10^{-3}$.

For Rawlsian, we plot difference of program \eqref{eq:assignment_lp} evaluated at $\x^{\text{integ}}$ and evaluated at   $\x^{\text{frac}}$ on \cens{} data. 

\begin{figure}[htbp]
    \centering

    \begin{subfigure}[b]{0.3\textwidth}
        \includegraphics[width=\textwidth]{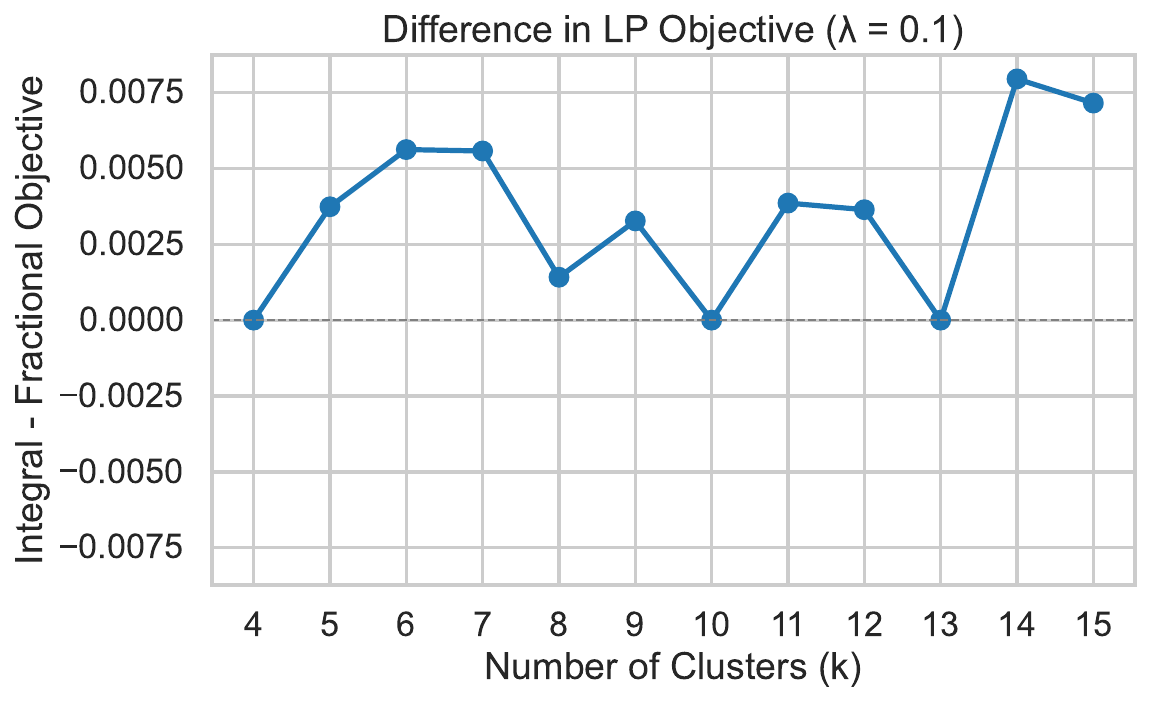}
        
    \end{subfigure}
    \hfill
    \begin{subfigure}[b]{0.3\textwidth}
        \includegraphics[width=\textwidth]{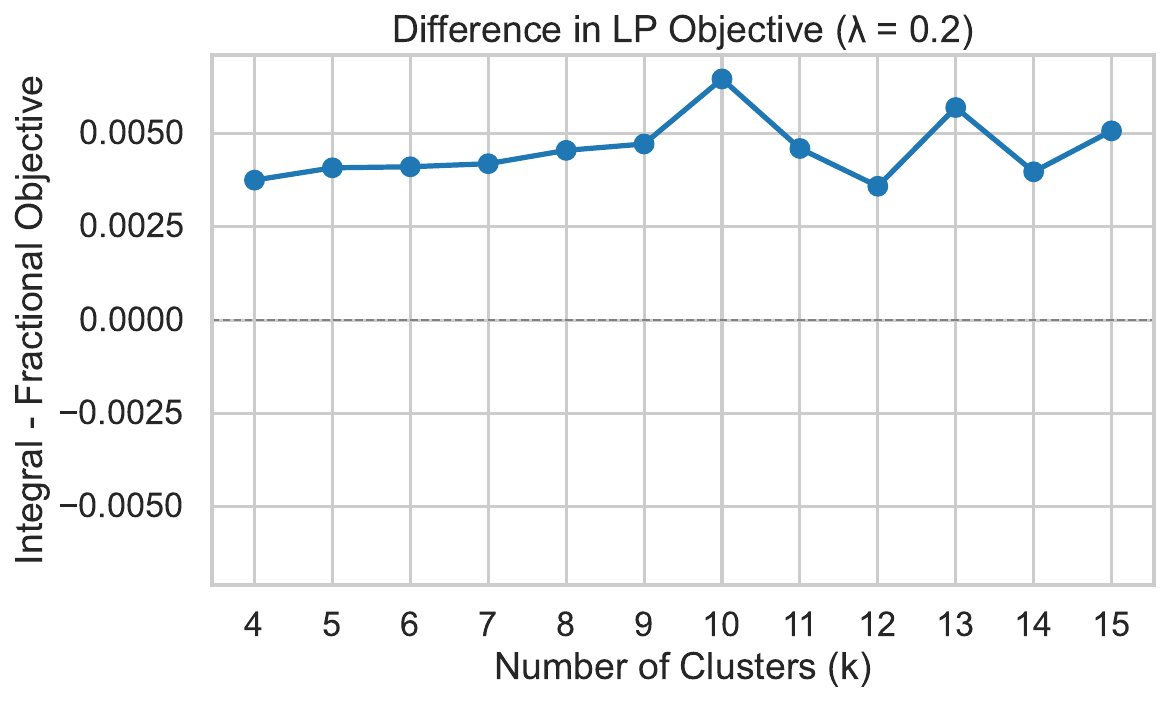}
        
    \end{subfigure}
    \hfill
    \begin{subfigure}[b]{0.3\textwidth}
        \includegraphics[width=\textwidth]{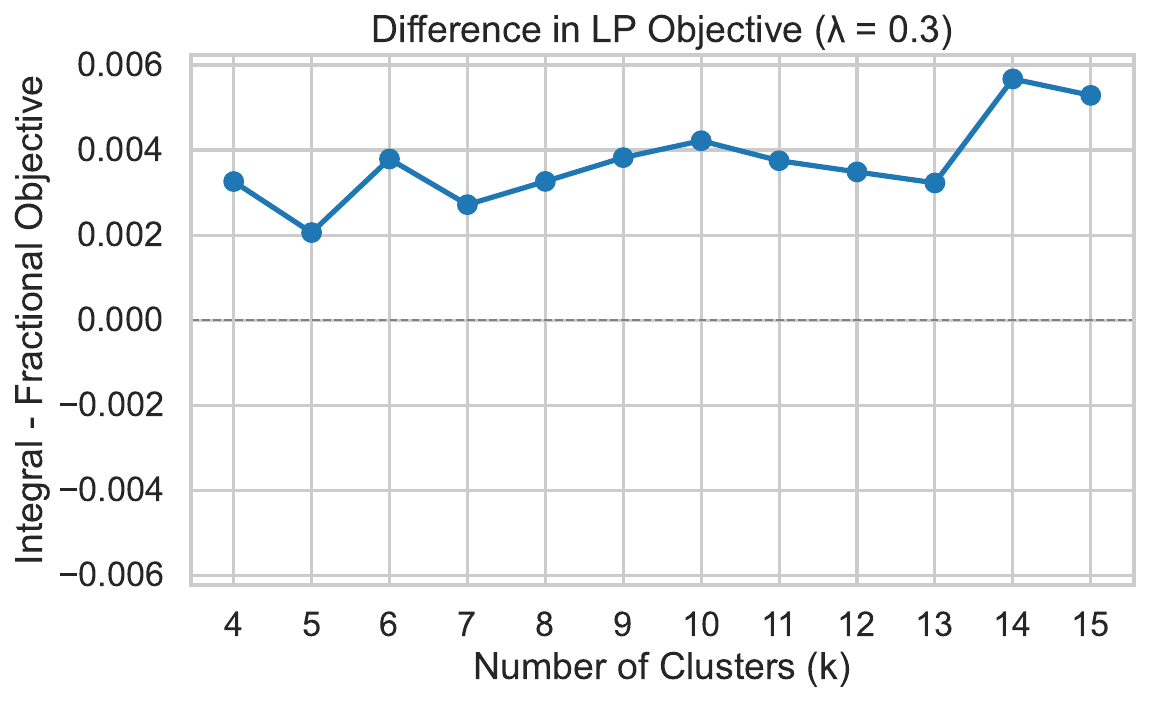}
        
    \end{subfigure}

    \begin{subfigure}[b]{0.3\textwidth}
        \includegraphics[width=\textwidth]{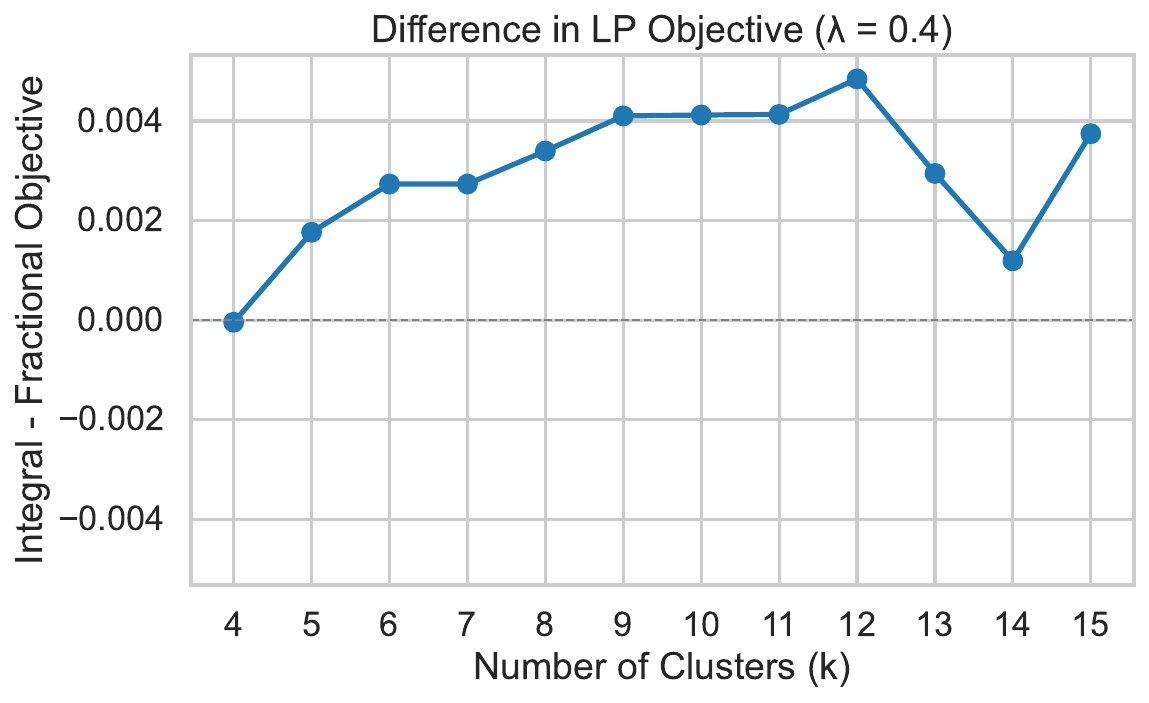}
        
    \end{subfigure}
    \hfill
    \begin{subfigure}[b]{0.3\textwidth}
        \includegraphics[width=\textwidth]{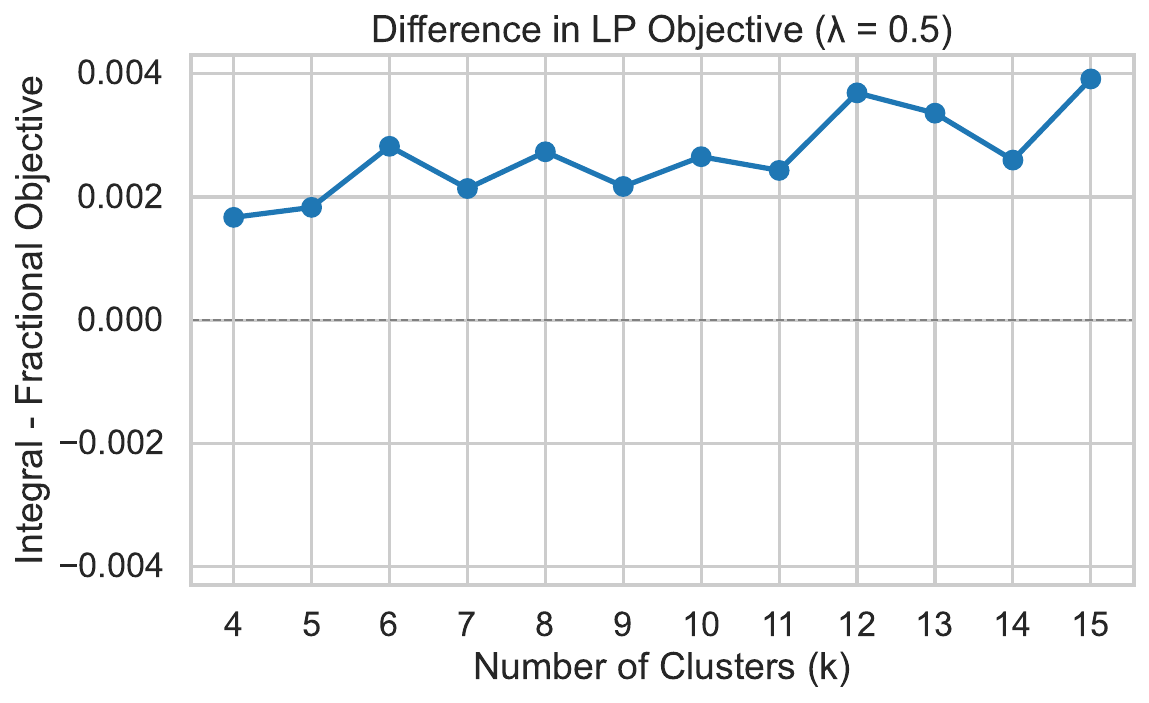}
        
    \end{subfigure}
  \hfill
    \begin{subfigure}[b]{0.3\textwidth}
        \includegraphics[width=\textwidth]{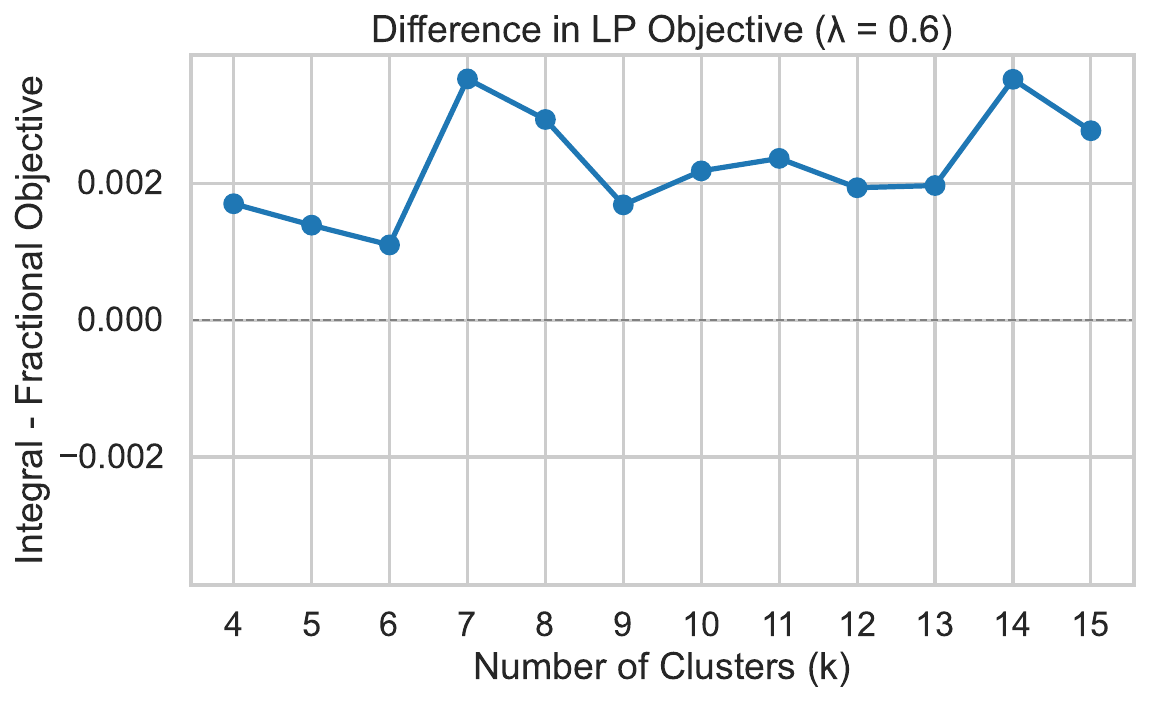}
        
    \end{subfigure}
 
    \begin{subfigure}[b]{0.3\textwidth}
        \includegraphics[width=\textwidth]{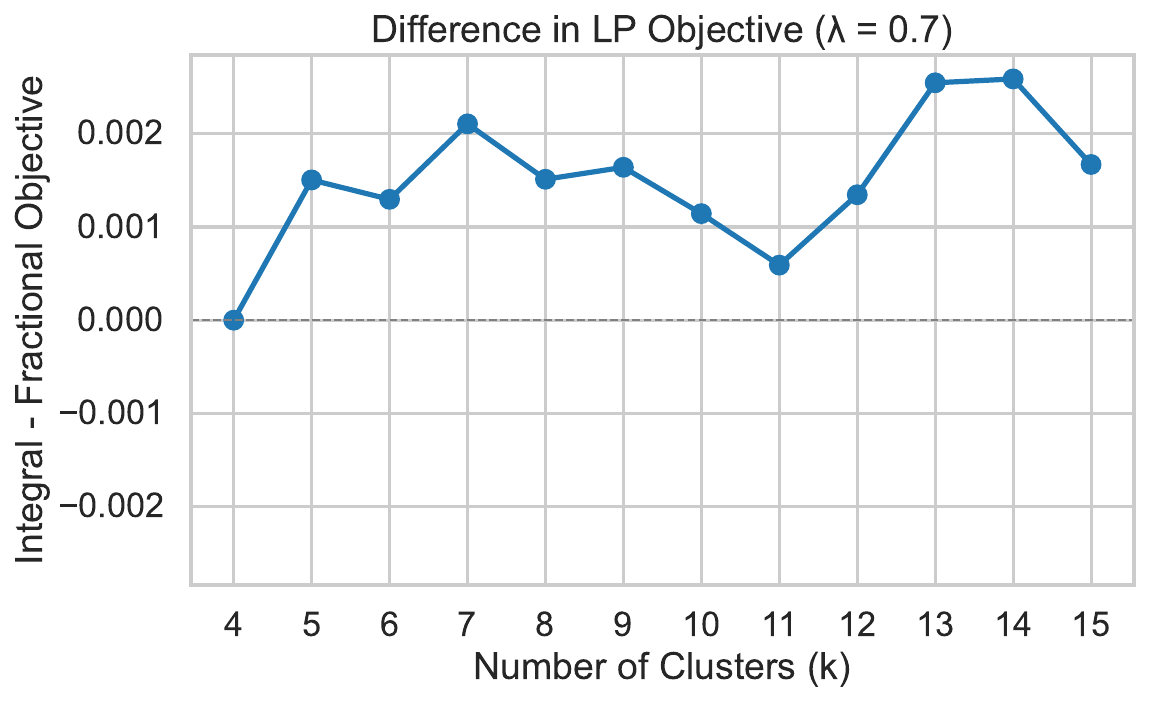}
        
    \end{subfigure}
    \hfill
    \begin{subfigure}[b]{0.3\textwidth}
        \includegraphics[width=\textwidth]{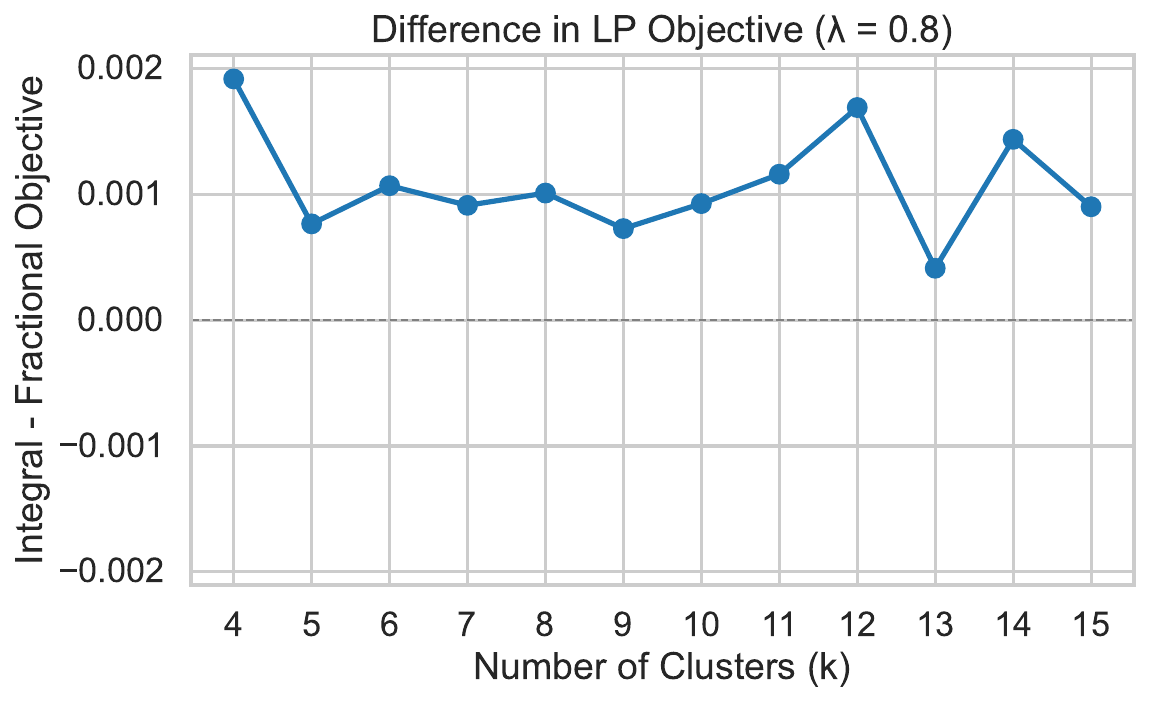}
        
    \end{subfigure}
    \hfill
    \begin{subfigure}[b]{0.3\textwidth}
        \includegraphics[width=\textwidth]{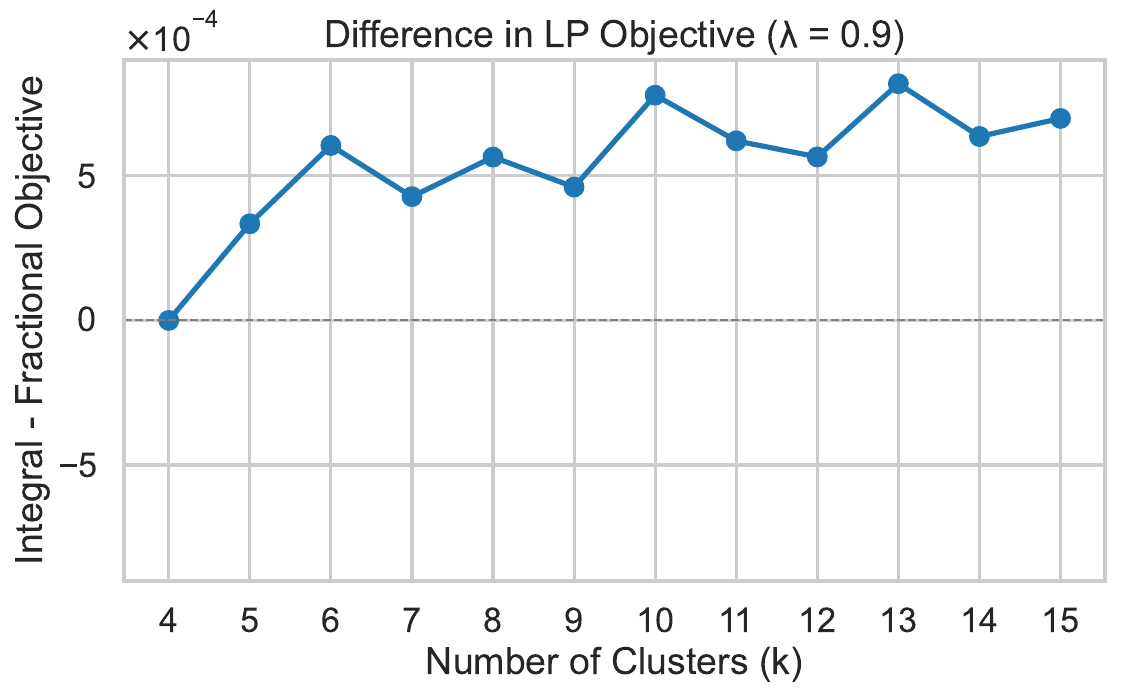}
        
    \end{subfigure}

    \caption{Integral Solution Rawlsian Value - Fractional Solution Rawlsian Value}
    \label{fig:eleven_subfigs}
\end{figure}
\FloatBarrier

Similarly for Utilitarian, we plot the difference of program \eqref{eq:assignment_lp_util} evaluated at $\x^{\text{integ}}$ and evaluated at   $\x^{\text{frac}}$ on \cens{} data.

\begin{figure}[htbp]
    \centering

    \begin{subfigure}[b]{0.3\textwidth}
        \includegraphics[width=\textwidth]{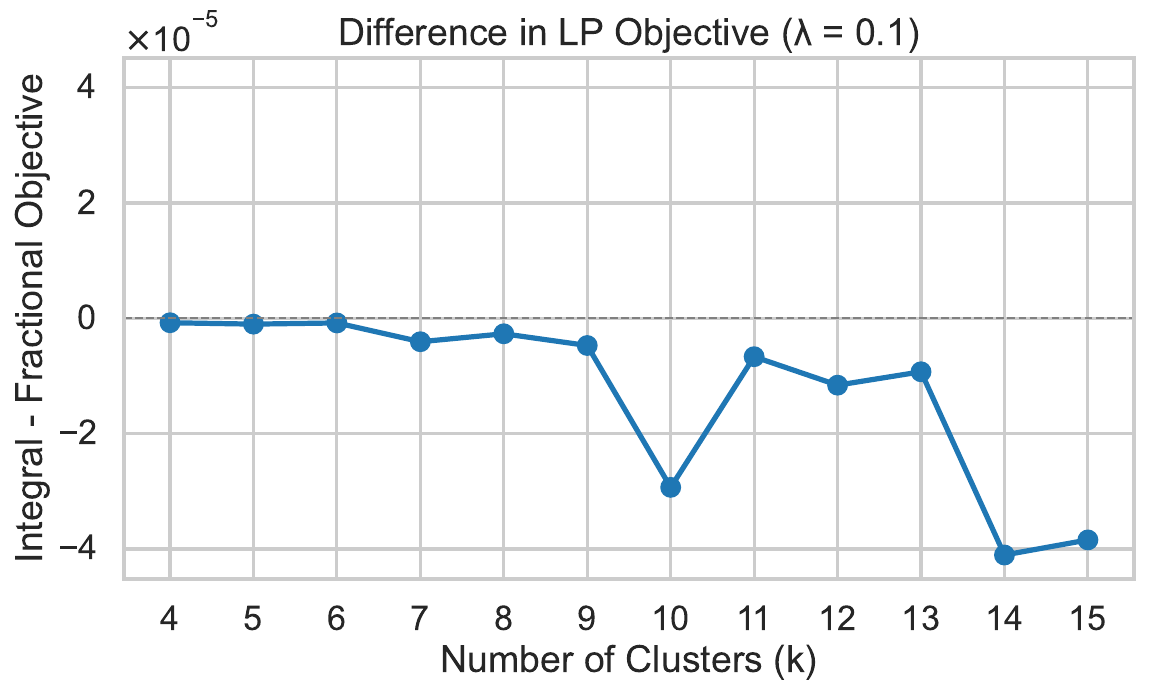}
        
    \end{subfigure}
    \hfill
    \begin{subfigure}[b]{0.3\textwidth}
        \includegraphics[width=\textwidth]{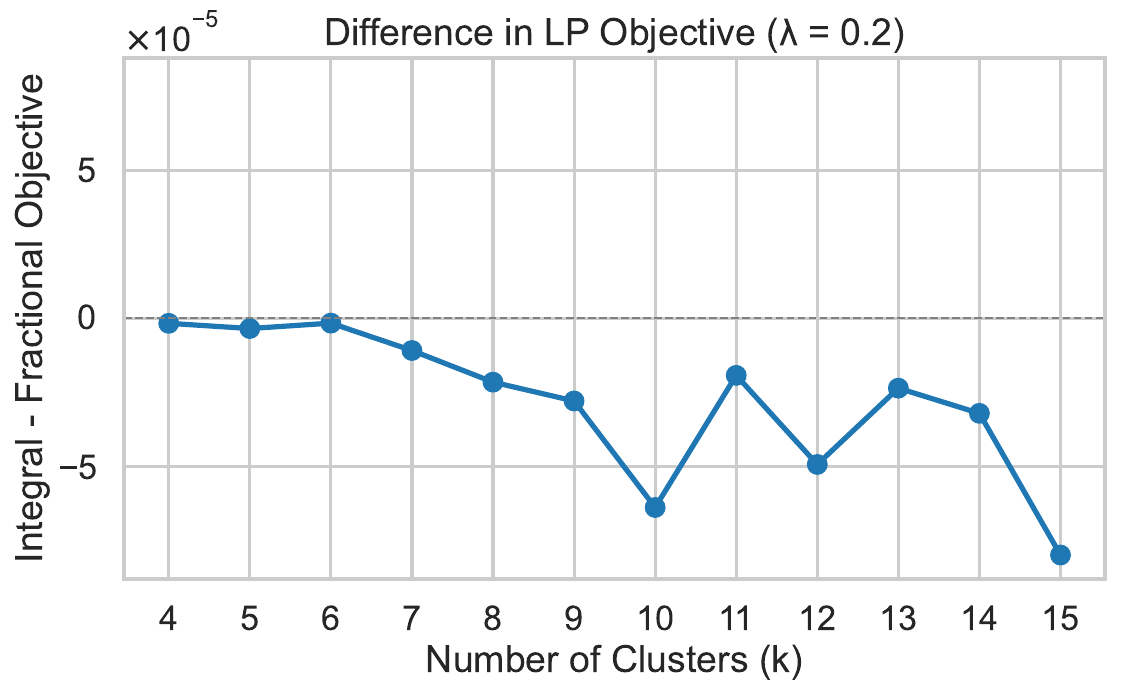}
        
    \end{subfigure}
    \hfill
    \begin{subfigure}[b]{0.3\textwidth}
        \includegraphics[width=\textwidth]{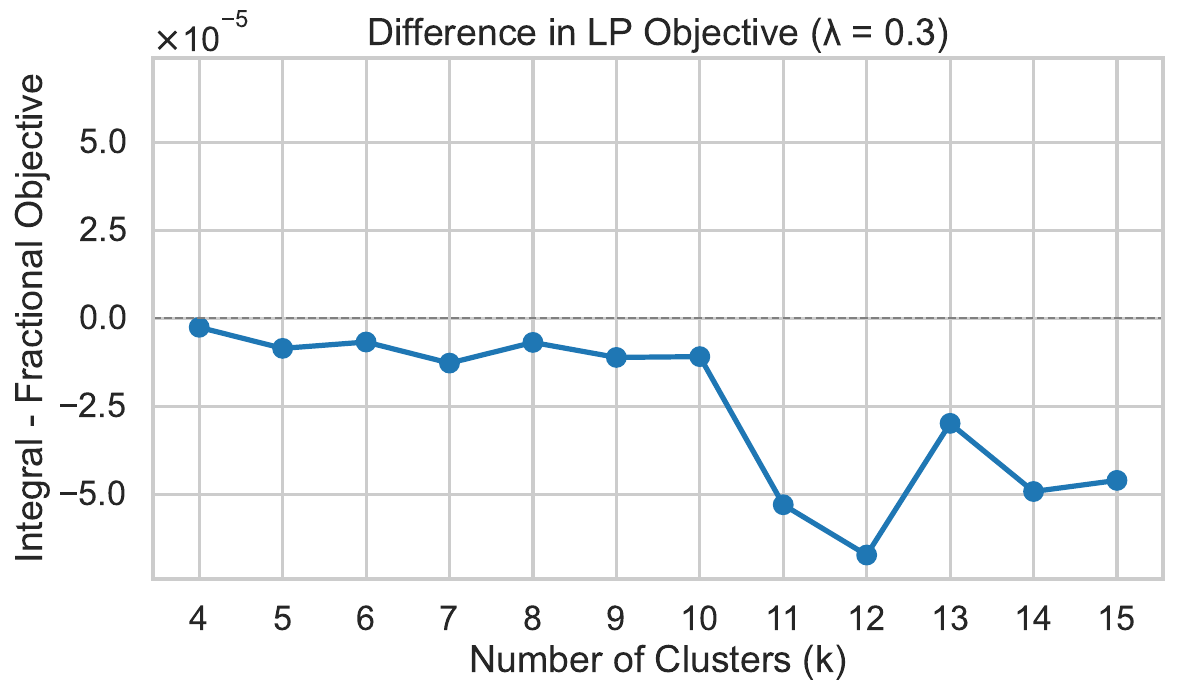}
        
    \end{subfigure}

    \begin{subfigure}[b]{0.3\textwidth}
        \includegraphics[width=\textwidth]{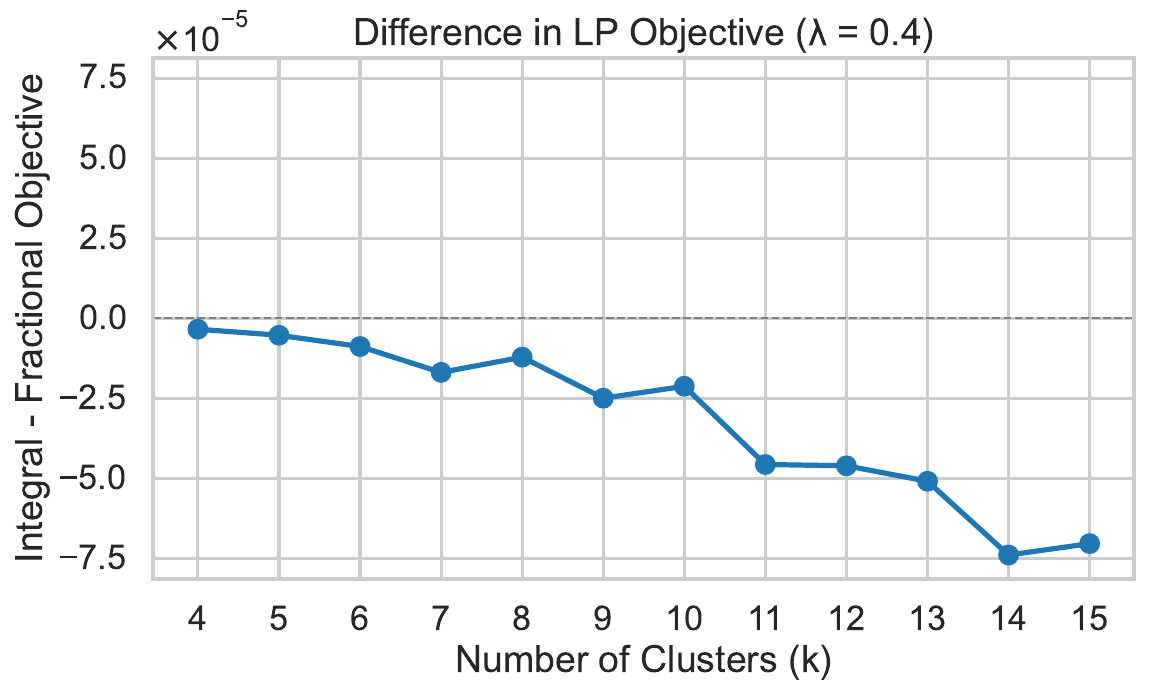}
        
    \end{subfigure}
    \hfill
    \begin{subfigure}[b]{0.3\textwidth}
        \includegraphics[width=\textwidth]{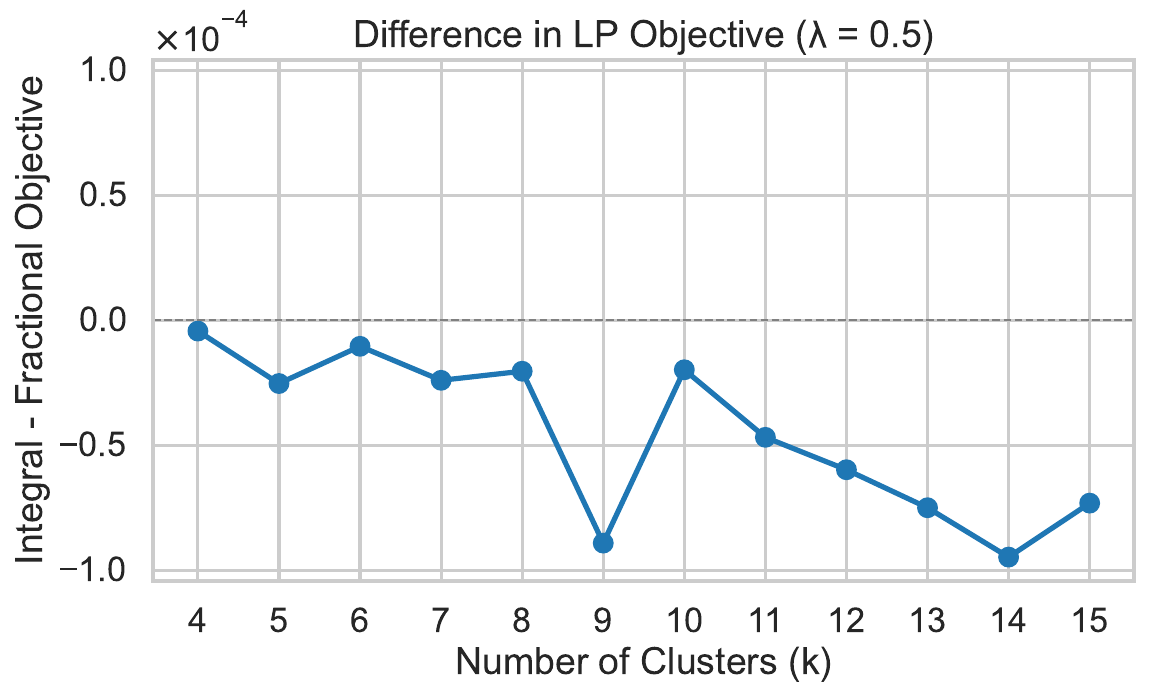}
        
    \end{subfigure}
  \hfill
    \begin{subfigure}[b]{0.3\textwidth}
        \includegraphics[width=\textwidth]{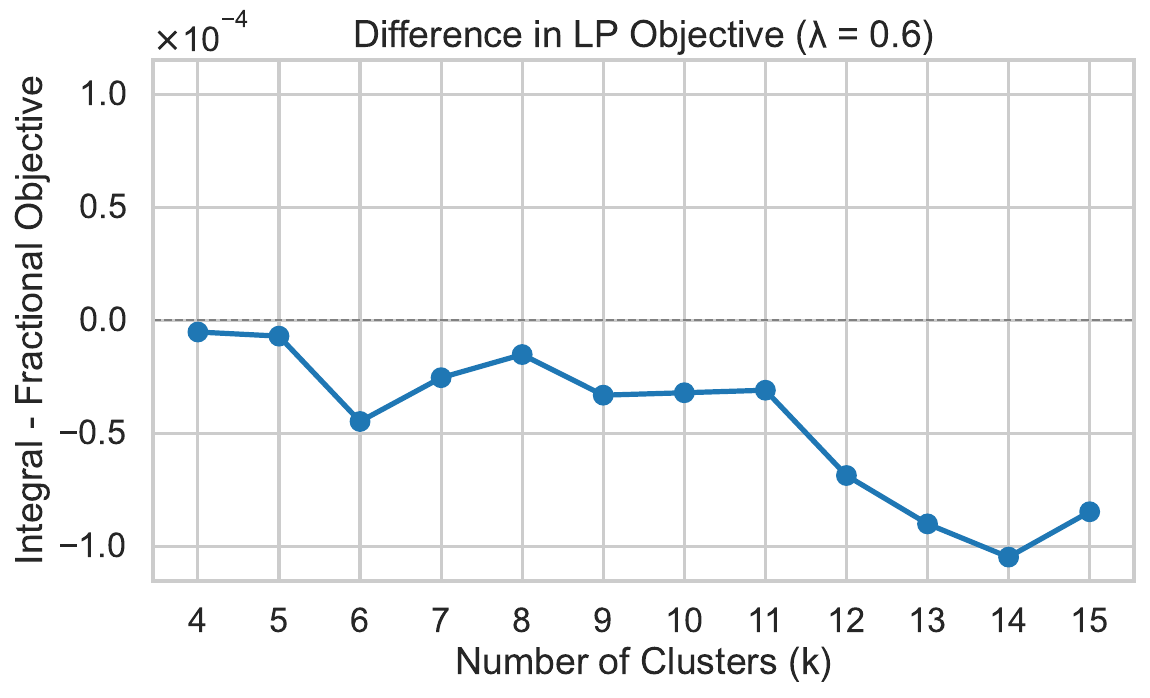}
        
    \end{subfigure}
 
    \begin{subfigure}[b]{0.3\textwidth}
        \includegraphics[width=\textwidth]{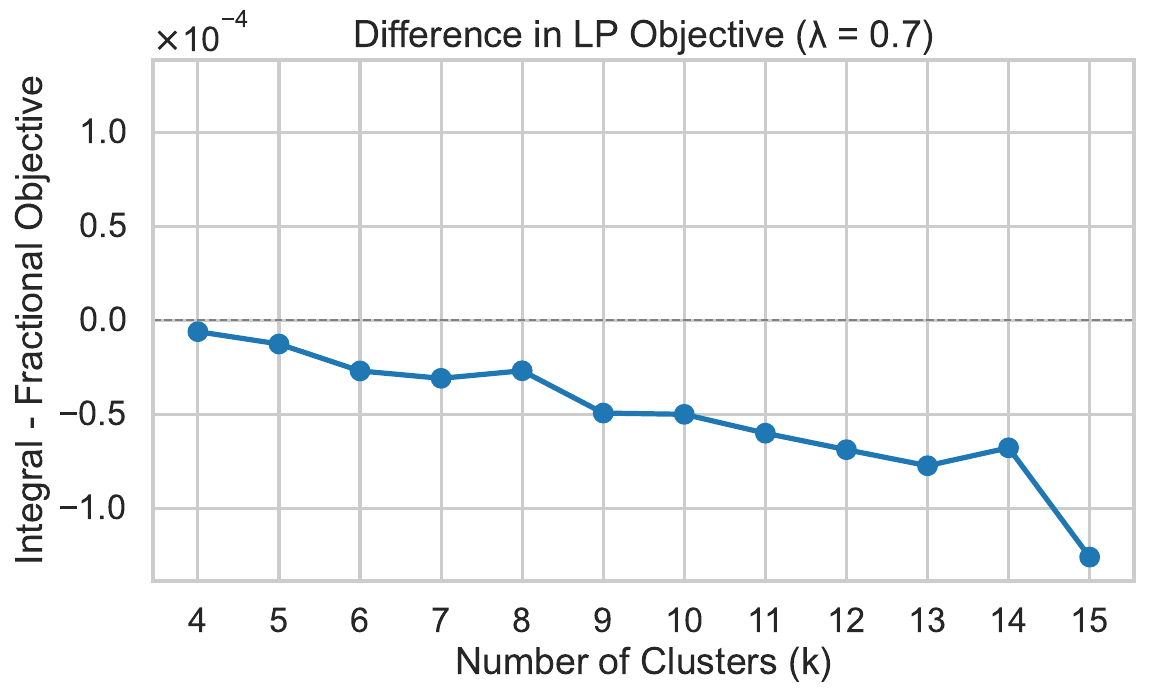}
        
    \end{subfigure}
    \hfill
    \begin{subfigure}[b]{0.3\textwidth}
        \includegraphics[width=\textwidth]{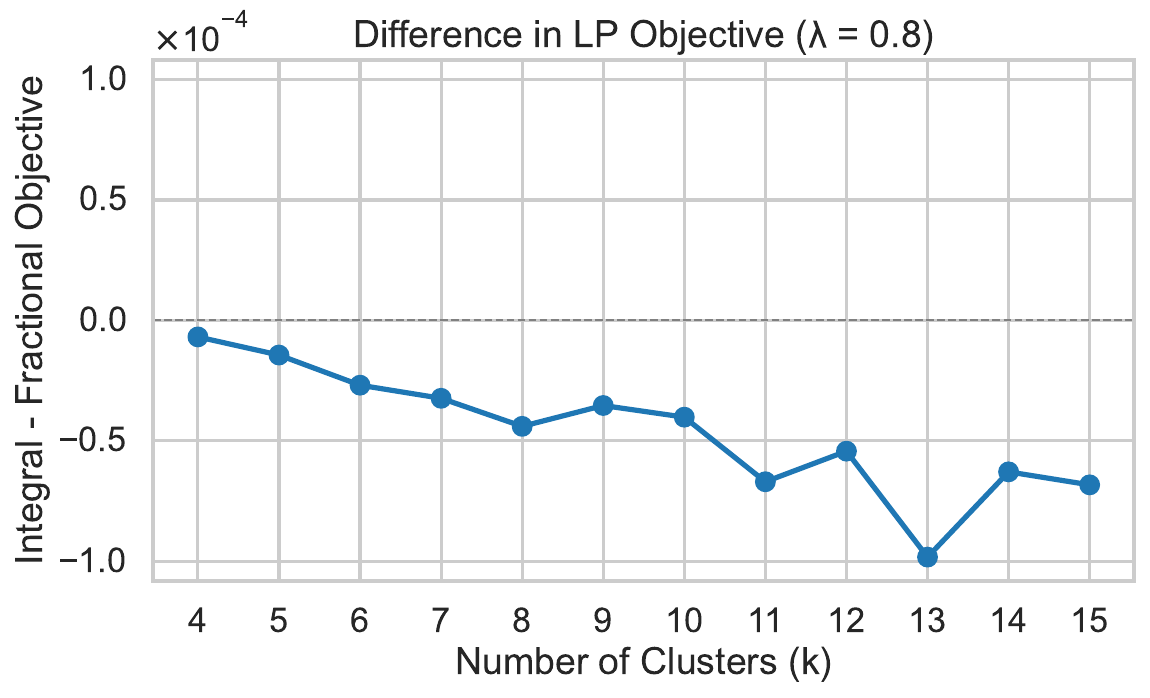}
        
    \end{subfigure}
    \hfill
    \begin{subfigure}[b]{0.3\textwidth}
        \includegraphics[width=\textwidth]{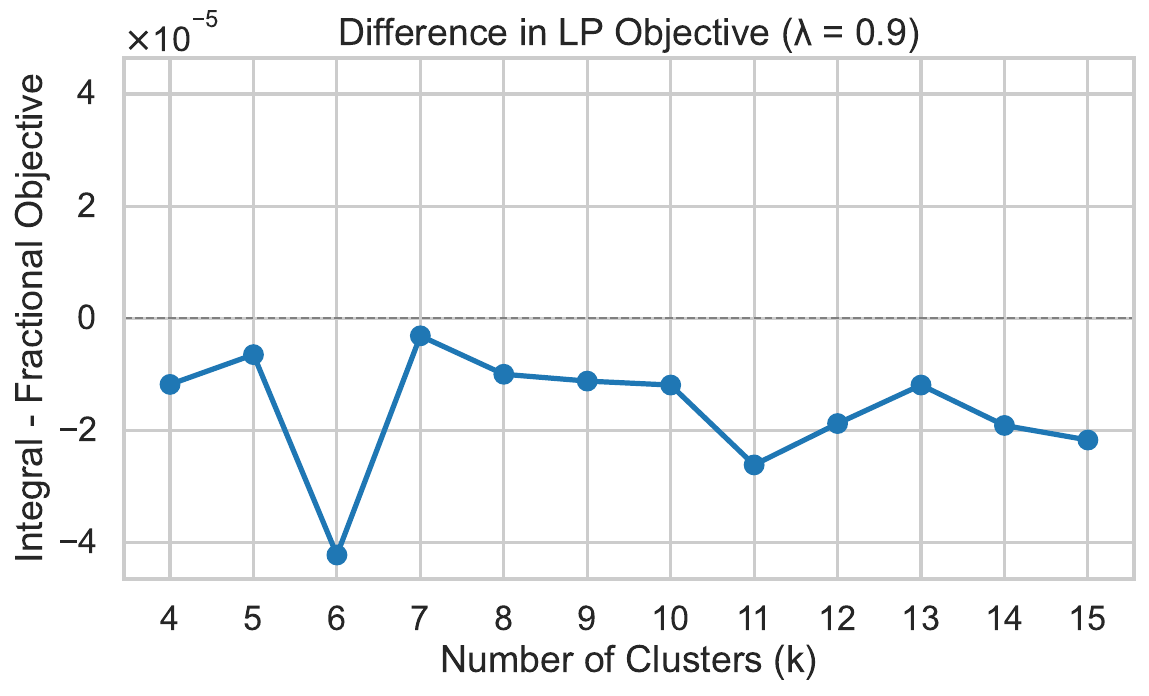}
        
    \end{subfigure}

    \caption{Integral Solution Utilitarian Value - Fractional Solution Utilitarian Value}
    \label{fig:eleven_subfigs}
\end{figure}
\FloatBarrier

\subsection{Run Time of $\algr$ and $\algu$}
In this section, we report the actual execution time of $\algr$ and $\algu$. We measure end-to-end runtime for each algorithm on sub-sampled versions of the \cens{} dataset, with sizes ranging from 10,000 to 20,000 in increments of 1,000. Each run solves the respective objective using $k=4$ clusters. Reported times are averaged over 5 independent runs, with standard deviation shown. All experiments were conducted on a MacBook Pro (2.3GHz 2-core Intel i5, 8GB RAM) using Python 3.10.

Figure~\ref{fig:runtime_plot} and Table~\ref{tab:runtime_table} display the recorded runtimes. Across all sample sizes, execution times range roughly from 15 to 55 seconds, with increasing sample size generally leading to longer runtimes. The shaded regions in the plot indicate the variability across runs.

\begin{figure}[h!]
    \centering
    \includegraphics[width=0.6\textwidth]{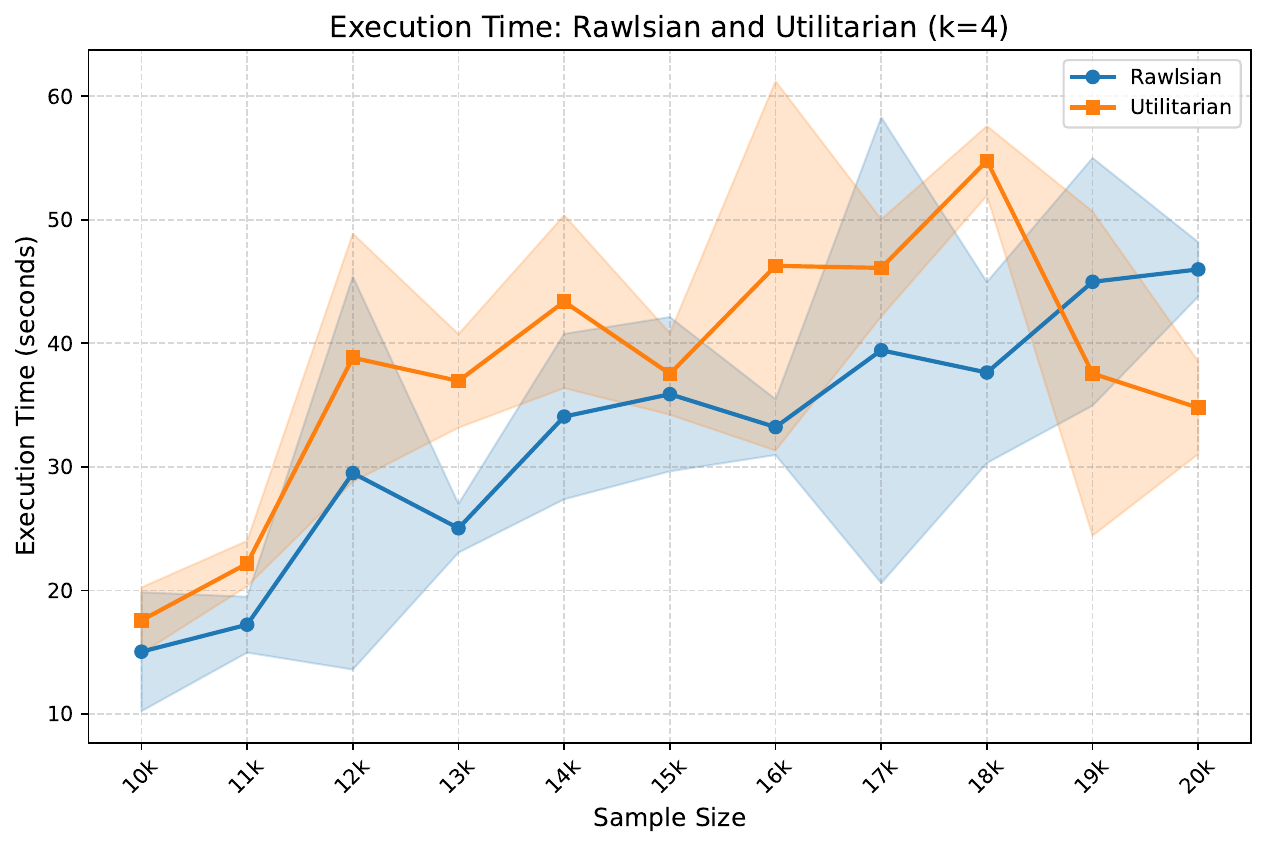}
    \caption{Runtime of $\algr$ and $\algu$ on \cens{} with $k=4$ for different dataset sizes.}
    \label{fig:runtime_plot}
\end{figure}

\begin{table}[h]
\centering
\begin{tabular}{|c|c|c|}
\hline
\textbf{Size} & \textbf{Rawlsian Time (s)} & \textbf{Utilitarian Time (s)} \\
\hline
10k & 15.05 $\pm$ 4.81 & 17.58 $\pm$ 2.67 \\
11k & 17.23 $\pm$ 2.26 & 22.18 $\pm$ 1.83 \\
12k & 29.50 $\pm$ 15.90 & 38.83 $\pm$ 10.04 \\
13k & 25.04 $\pm$ 1.97 & 36.95 $\pm$ 3.78 \\
14k & 34.07 $\pm$ 6.69 & 43.37 $\pm$ 6.99 \\
15k & 35.88 $\pm$ 6.24 & 37.53 $\pm$ 3.31 \\
16k & 33.22 $\pm$ 2.25 & 46.27 $\pm$ 14.94 \\
17k & 39.44 $\pm$ 18.85 & 46.10 $\pm$ 3.93 \\
18k & 37.63 $\pm$ 7.32 & 54.76 $\pm$ 2.82 \\
19k & 44.97 $\pm$ 10.02 & 37.55 $\pm$ 13.11 \\
20k & 45.98 $\pm$ 2.20 & 34.78 $\pm$ 3.75 \\
\hline
\end{tabular}
\caption{Runtime of $\algr$ and $\algu$ on \cens{} with $k=4$ for different dataset sizes.}
\label{tab:runtime_table}
\end{table}

\section{Limitations and Impacts}\label{appsec:limitation}

While our proposed welfare formulations provide flexible frameworks for balancing disutilities from distance and proportional representation, we emphasize that fairness is inherently a context-dependent and value-laden concept. The deployment of the proposed formulation and  algorithms in practice should involve careful deliberation with relevant stakeholders to determine appropriate fairness goals and trade-offs. In particular, the choice of parameter $\lambda$, which governs the trade-off between distance and  proportional representation must be guided by domain-specific considerations and the priorities of affected communities. Further, should the Rawlsian objective, which is mainly an egalitarian objective, be optimized or should the Utilitarian be optimized instead? Moreover, our formulation currently quantifies disutility primarily in terms of clustering distance and proportional group representation. In real-world settings, disutility may depend on additional factors such as downstream impacts of clustering outcomes as hinted in \citet{dickerson2024fair}. Additionally, even when it is clear and agreed upon that proportional group representation is important, the exact values of the proportion relaxation parameters $\alpha_h$ and $\beta_h$ should be carefully done. It is possible that for a given application only group under-representation would matter and as a result one may set $\alpha_h\ge 1-r_h$ for all colors $h \in \Colors$. 

We believe that our proposed framework and algorithms can lead to better welfare and societal impacts if such considerations are carefully taken into account. On the other hand, a superficial approach that does not incorporate the full considerations and does carefully tune parameters such as $\lambda$, $\alpha_h$, and $\beta_h$, may lead to clustering outcomes that are less preferable.



\end{document}